\documentclass{statsoc}
\usepackage[dvipsnames]{xcolor}
\usepackage{xr-hyper}

\usepackage{amsthm,amsmath,amssymb}

\usepackage[a4paper]{geometry}

\allowdisplaybreaks[3]
\usepackage{placeins}
\theoremstyle{plain} 
\newtheorem{theorem}{Theorem}
\newtheorem{lemma}{Lemma}
\newtheorem{proposition}{Proposition}
\newtheorem{corollary}{Corollary}

\theoremstyle{definition}
\newtheorem{definition}{Definition}

\newtheorem{remark}{Remark}

\usepackage[utf8]{inputenc}
\usepackage{xifthen}
\usepackage{graphicx}
\usepackage[hypertexnames=false]{hyperref}
\usepackage{multirow}
\usepackage{verbatim}
\usepackage{color}
\usepackage{makecell}
\usepackage{multicol}
\usepackage{tikz}

\usetikzlibrary {shapes}
\usetikzlibrary {arrows}
\usetikzlibrary {positioning}
\usetikzlibrary {calc}
\usetikzlibrary {fit}				
\usetikzlibrary {backgrounds }
\usepackage {pgfplots}
\usepackage {changepage}
\pgfplotsset{compat=1.10}

\newcommand{\beao}{\begin{equation*}}
\newcommand{\eeao}{\end{equation*}}

\newcommand{\beam}{\begin{equation}}
\newcommand{\eeam}{\end{equation}}

\newcommand{\barr}{\begin{array}}
\newcommand{\earr}{\end{array}}

\DeclareMathOperator{\RR}{\mathbb{R}}
\DeclareMathOperator{\R}{\mathbb{R}}
\DeclareMathOperator{\E}{\mathbb{E}}

\renewcommand{\P}{\mathbb{P}}

\newcommand{\T}{\mathcal{T}}

\newcommand{\G}{\mathcal{G}}

\newcommand{\calx}{\mathcal{X}}
\newcommand{\calt}{\mathcal{T}}
\newcommand{\calg}{\mathcal{G}}

\newcommand{\calr}{\mathcal{R}}

\newcommand{\halmos}{\quad\hfill\mbox{$\Box$}}
\newcommand{\bigR}{\overline{r}}
\newcommand{\smallR}{\underline{r}}
\newcommand{\Var}{{\rm Var}}

\newcommand{\al}{\alpha}
\newcommand{\eps}{\varepsilon}

\newcommand{\eqd}{\stackrel{\mathrm{d}}{=}}

\newcommand{\QTree}{\texttt{QTree}}
\renewcommand{\E}{\mathbb{E}}

\newcommand{\JB}[1]{\textcolor{black}{ #1}} 

\definecolor{fincolor}{RGB}{200,230,255}

\usepackage{comment}
\usepackage{color, colortbl}
\usepackage{threeparttable}
\usepackage{bm}
\usepackage{multirow}
\usepackage{url}
\usepackage[noend]{algpseudocode}
\usepackage{algorithm}

\definecolor{Gray}{gray}{0.8}
\definecolor{Gray2}{gray}{0.95}

\definecolor{sienna}{RGB}{160,82,45}
\definecolor{gray}{RGB}{128, 128, 128}

\usepackage{tikz}
\usetikzlibrary{arrows, automata}
\usetikzlibrary{positioning}
\usetikzlibrary {shapes,decorations.pathmorphing}
\usetikzlibrary{backgrounds}

\usetikzlibrary{calc,matrix,positioning}

\tikzset{ 
    table/.style={
        matrix of nodes,
        row sep=-\pgflinewidth,
        column sep=-\pgflinewidth,
        nodes={rectangle,text width=3em,align=center},
        text depth=1.25ex,
        text height=2.5ex,
        nodes in empty cells
    },
}

\usepackage{natbib}
\usepackage{float}

\makeatletter
\newcommand*{\addFileDependency}[1]{
\typeout{(#1)}
\@addtofilelist{#1}
\IfFileExists{#1}{}{\typeout{No file #1.}}
}\makeatother

\renewcommand{\JB}[1]{\textcolor{black}{ #1}} 

\title[Estimating a Directed Tree for Extremes]{Estimating a Directed Tree for Extremes}

\author{Ngoc Mai Tran\thanks{Supported by NSF Grant DMS-2113468 and NSF IFML 2019844 Award.}}
\address{Department of Mathematics, University of Texas at Austin, Speedway 2515 Stop C1200, Austin TX 78712, USA, email: ntran@math.utexas.edu}
\author[Ngoc M. Tran, Johannes Buck and Claudia Kl\"uppelberg]
{Johannes Buck\thanks{Supported by the Hanns Seidel Foundation} 
and 
Claudia Kl\"uppelberg\thanks{Corresponding author}}
\address{Department of Mathematics, Technical University of Munich,  85748 Garching, Boltzmannstr.~3, Germany, emails:  j.buck@tum.de, cklu@cit.tum.de}

\begin{document}
\setlength{\textfloatsep}{5pt}
\setlength{\abovecaptionskip}{2pt}

\begin{abstract}
We propose a new method to estimate a root-directed spanning tree from extreme data.
A prominent example is a river network, to be discovered from extreme flow measured at a set of stations. 
Our new algorithm utilizes qualitative aspects of a max-linear Bayesian network, which has been designed for modelling causality in extremes.
The algorithm estimates bivariate scores and returns a root-directed spanning tree. 
It performs extremely well on benchmark data and new data.
We prove that the new estimator is consistent under a max-linear Bayesian network model  with noise. 
We also assess its strengths and limitations in a small simulation study.
\end{abstract}


{\em Keywords: }
{Bayesian network},
{causal inference},
{directed acyclic graph},
{extreme value analysis},
{graphical model},
{max-linear model}.

\section{Introduction}\label{sec:intro}


Graphical models can represent multivariate distributions in an intuitive way and, hence, facilitate statistical analyses of high-dimensional data.  
Traditionally, such models are linear and distributions are Gaussian; see e.g. \cite{BG, drtonmaathuis, lauritzen1996graphical, Drtonetal}. 
In recent years, extensions to non-Gaussian linear models have been proposed with statistical methods focusing on second order properties of their distributions, which estimate the graph structure for observations in the center of the distribution.

Gaussian models and correlations are inappropriate for assessing high risks, where the extreme observations contain the relevant information. For such problems extreme value distributions provide natural models. 
Whereas Gaussian distributions arise as limit distributions of normalized sums, extreme value distributions arise as limit distributions of normalized maxima, thus, being natural candidates for modelling high risks.
Gaussian distributions are sum-stable (closed with respect to sums) and extreme value distributions are max-stable (closed with respect to maxima). 
This duality and more relations between sums and maxima are presented for dimension $d=1$ in \cite{EKM}. 

For dimension $d\ge 2$ max-stable distributions also arise as limits of componentwise normalized maxima of independent copies of a random vector $X$, and the dependence structure between components of the limit vector has been extensively studied and applied in multivariate risk problems.
Textbook treatments can be found in \cite{Beirlant2004, coles2001introduction, de2007extreme, Resnick1987, ResnickHeavy}. A very readable review paper is \cite{davison2015statistics}. For statistical applications, correlations or other bivariate measures of dependence in the center of the distribution are replaced by {\em extreme dependence measures} (\cite{colesetal, EV, lars, Sibuya1960}). 

This paper focuses on causal interpretation and is motivated by the fact that rare events like environmental or financial risks are often cascading through a network. For instance, pollutants can propagate through an unseen underground waterway, causing extreme measurements at multiple locations (\cite{Leigh2019}), or credit markets can fail due to some endogenous systemic risk propagation (\cite{creditprop}). 

Graphical models can allow for causal interpretation, 
however, it is not immediately obvious how to extend the past decades of work on causal inference (\cite{Bollen, drtonmaathuis, lauritzen1996graphical, Drtonetal, Pearl2009,spirtes2000causation}) for Gaussian and discrete distributions to an extreme value setting.

 \begin{figure}[t]
\centering
\begin{tikzpicture}[->,>=stealth',shorten >=1pt,auto,node distance=1.5cm,
                thick,main node/.style={circle,draw,font=\small\bfseries,minimum size=0.8cm,inner sep=0cm},true edge/.style={Green},stream edge/.style={densely dashed,YellowGreen},wrong edge/.style={ red,decorate,
  decoration={zigzag,amplitude=0.4pt,segment length=1.4mm,pre=lineto,pre length=0pt}},missing edge/.style={loosely dotted}]
                
  \node[main node] (1925) {1};
  \node[main node] (7) [right of=1925] {2};
  \node[main node] (1277) [right of=7] {3};
  \node[main node] (3) [right of=1277] {4};
  \node[main node] (6) [below of=1277] {5};
  \node[main node] (2) [right of=6] {6};
  
  \path
    (7) edge  node {} (1925)
    (1277) edge  node {} (7)
    (3) edge  node {} (1277)
    (6) edge  node {} (7)
    (2) edge  node {} (1277)
    ;
\end{tikzpicture}
\caption{\label{fig:example} Example of a root-directed spanning tree with root 1. It is a simple directed acyclic graph, where each node has exactly one child, except the root, which has none. Moreover, there is exactly one path from every node to the root.
}
\end{figure}

We approach this problem from two directions.

Firstly, we follow the general idea (see e.g. \cite{Pearl2009}) that causality is often provided through a recursive system on a directed acyclic graph.
Prominent examples in the literature are {\em linear} recursive systems, where each node represents a random variable defined as a weighted sum of its parent variables and an independent random variable. 
But instead of such classical linear causal graphical models or linear Bayesian networks, we use max-linear causal graphical models or max-linear Bayesian networks. 
Introduced in \cite{GK1}, they are defined via {\em max-linear} recursive structural equation models on a directed acyclic graph. 

In a {\em max-linear Bayesian network}, each node represents a positive random variable defined as a weighted maximum of its parent variables and an independent random variable, called {\em innovation}. Although motivated by extreme value theory, the multivariate distribution of a max-linear vector is not restricted to an extreme value distribution; the emphasis of the model is on its structure given by a directed acyclic graph.
We relax the strict max-linear Bayesian network model by allowing for an independent noise variable.
Nonparametric statistical inference aims at identifying the directed graphical structure regardless of the node distributions under weak conditions on the noise distributions.

Secondly, we propose a new algorithm, \QTree, motivated by qualitative aspects of a max-linear Bayesian network to estimate a root-directed spanning tree  (see Figure~\ref{fig:example}) as a simple directed acyclic graph. 
\QTree\  uses pairwise dependence, handles missing data, and has an automated parameter tuning procedure.  Here we use the fact that the non-noisy model has a left-sided atom for the distribution of a ratio of marginal random variables, when there is a directed edge between the nodes. We also show that, under a max-linear Bayesian network model  with noise and natural distributional assumptions, the \QTree\ algorithm returns asymptotically almost surely the correct tree.

Max-linear Bayesian networks have recently emerged as suitable directed graphical models for causality in extremes (\cite{Ngoc_etal,BK,nadinethesis, GK1}), however, existing methods for learning them aim to learn the model parameters and, thus, are highly sensitive to model misspecifications; see \cite{BK,nadinethesis, GKL}; \cite{GKO, KK, KL2017}. 

Motivated by extreme value theory, it is not surprising that max-linear Bayesian networks have been mostly investigated  for heavy-tailed innovations: \citet[Section~3.3]{einmahl2016} consider tail dependence functions for i.i.d. Fr\'echet innovations. \cite{GKO} investigate tail dependence for i.i.d. regularly varying innovations, and \cite{KK} the scaling properties of the same model. \cite{AMS} and \cite{segers2019one} investigate regularly varying Markov trees, and more recently, \cite{AS} investigate a new max-linear graphical model on trees of transitive turnaments.

Graphical models for extremes have also been proposed in \cite{engelke:hitz:18} based on a different concept.
They define a new extreme conditional dependence concept for multivariate Pareto distributions (which have Lebesgue densities) and use this concept to define extreme undirected graphical models similarly to the classical concept. The multivariate distribution determines the model and has in general to be specified for statistical inference. Here the multivariate H\"usler-Reiss distribution plays a prominent role; see \cite{AMS, AS, EV, ELV, Huetal, REZ}.


\subsection{The Extremal River Problem}

The relevance of extremal graphical models for multivariate distributions has been validated on several data sets, prominently on the Upper Danube river network. 
The goal is to recover a river network from \emph{only extreme flow} measured at a set $V$ of stations, \emph{without} any information on the stations' location. We refer to it as the \emph{Extremal River Problem}. Here, the true river network is known and serves as the `gold standard', allowing one to verify the performance of a proposed estimator. 
Success in solving the Extreme River Problem can translate to new solutions to the \emph{contaminant tracing} challenge in hydrology (\cite{Leigh2019,mcgrane2016impacts,Rodriguez2020,Verhoef2,VerHoef1,WOLF20128}). There, one needs an inexpensive method to trace pollutants or chemical constituents transported by a complex and unknown underground waterway that is prohibitive to model or survey with traditional fluid mechanics methods (\cite{anderson2015applied}). 
Recent advances point towards an imminent data explosion (\cite{bartos2018open,mao2019low}), where pollutants exceeding certain thresholds can be detected via a sensor network. Thus, contaminant tracing with sensors data is a version of the Extremal River Problem without the gold standard, where the network is truly unknown. 

A solution to the Extremal River Problem requires an algorithm to recover the true river network using test data given by {\em river discharges}, the volume of water flowing through a river channel, measured at any given point in cubic metres or cubic feet per second. 
The Extremal River Problem for the Upper Danube river network with measurements collected at $d=31$ stations has proven to be challenging and very stimulating for extreme value theory, with each paper taking a different technique. The data have been preprocessed in \cite{asadi2015} and are available in the \texttt{R} package \texttt{graphicalExtremes}  (\cite{RgraphicalExtremes}).

The preprocessed data have been analysed in a number of publications with focus on modelling extreme dependence: 
flow- and spatial dependence (\cite{asadi2015}) and undirected graphical models for extremes (\cite{engelke:hitz:18, ELV, Huetal, Huseretal, REZ}). 
In a first paper, \cite{engelke:hitz:18} returned a highly accurate but \emph{undirected} graph, followed by publications using new models and applying different methods for reconstructing the undirected graph.

Our focus is on causality in extremes, modelled by the edges of a root-directed tree: a large value at node $j$ causes a large value at node $i$, whenever there is an edge from $j$ to $i$.
For a river network the causal structure in the {\em extremes} is the same as for {\em average values}  and in this case the tree can be learned with methods for extremes and averages.
However, it may well happen that causality is stronger in the extremes than in average values.
Indeed, in \cite{tran.extremes}, Section~2.1 it is shown that for the Upper Danube data also a naive algorithm based on the pairwise correlation matrix as score matrix performs well, whereas for Lower Colorado data it returns a less precise tree. So this is an example, where the extremes contain more causal information than the average observations.

Causal dependence models for extremes have also been considered using expected quantile scores (\cite{mhalla2020causal}) and  causal dependence coefficients (\cite{gnecco2019causal}).
\cite{gnecco2019causal} correctly recovered the causal order of 12 nodes out of 31, but did not learn the entire river network, while \cite{mhalla2020causal} focused on flow-connections and did well at detecting nodes connected by a directed path; see Figure~7 in \cite{mhalla2020causal}. These two publications have slightly different notions of causality. We give more details and compare our method with theirs in Section~4.

\subsection{Main contributions and structure of the paper}

Below, we explain the novel aspects of our paper,  summarize the organisation of our paper, and define the standard metrics to assess the quality of an estimated graph.

Most prominently, we suggest a new algorithm \QTree\ to recover causality in extremes, where causality is modelled by the edges of a root-directed tree. 
This algorithm relies on qualitative aspects of a max-linear Bayesian network model and is as such a structural model, which does not require to specify a distribution family. Moreover, no normalization of the data to standard Fr\'echet or Gumbel distribution is needed. 
  
The \QTree\ Algorithm~\ref{alg:qtree} estimates a score matrix $W$, giving a score for each potential edge independently, and applies a standard algorithm to output a root-directed spanning tree of optimum score. 
Algorithm~\ref{alg:qtree} runs in time $O(n|V|^2)$, where $n$ is the number of observations and $|V|$ is the number of nodes. Moreover, it maximizes the information available from missing data, since at each step it only utilizes the data projected onto two coordinates. \QTree\ needs sufficiently many extreme observations and relies on the signal to have heavier tail than the noise.  

We improve this simple algorithm by an optimization procedure to find the best model parameters using a grid search in combination with a stabilizing subsampling procedure that is based on bootstrap aggregation.
This \QTree\ Algorithm~\ref{alg:qtree.automated}  is very flexible, has at most two tuning parameters, and proves to be very efficient. 

Besides the Upper Danube data set, which serves as benchmark set for comparison to other methods, we analyse three new data sets from the Lower Colorado river network in Texas. We also show by a small simulation study that \QTree\ is robust with respect to different dependence structures (given by edge-weights) and different node distributions entailed from different innovations distributions.

\QTree\ is implemented as a plug-and-play package in \texttt{Python} (\cite{ngoc_alg}) at \\
\centerline{\url{https://github.com/princengoc/qtree}} \\
which includes all data and codes to produce the results and figures in this paper.

Beyond hydrology, \QTree\ can be applied to cause and effect detection in every high risk problem 
assuming that the network is a root-directed tree.  \QTree\ can, however, also solve a slightly more general problem. 
Assume that the tree structure is only in the extremes, whereas ``average" data follow a different model. For instance, it can happen that only data from certain nodes follow a heavy-tailed distribution (able to model extreme events, while other nodes are negligible from an extreme value point of view; see e.g. \cite{de2007extreme} or \cite{Resnick1987, ResnickHeavy}).
Then it may be possible that causality in the extremes can be modelled by a tree on a subset of nodes.

Assuming that the data come from a noisy max-linear Bayesian network, under distributional assumptions with appropriate signal-to-noise ratio, we prove in Theorem~\ref{thm:main} that the tree output by \QTree\ is strongly consistent as the sample size tends to infinity.
This proof is based on a new variational argument to account for noise in the data.


Our paper is organized as follows. 
We introduce \QTree\ (Algorithm \ref{alg:qtree}) and \texttt{auto-tuned} \QTree\ (Algorithm \ref{alg:qtree.automated}) in Section \ref{sec:alg} and give some intuition supported by preliminary simulation results. In Section \ref{sec:data_description}, we present the data sets, discuss their specific challenges and describe the data preprocessing steps. In Section~\ref{sec:results}, we present the estimation results of \QTree\ and analyse the performance of the automated parameter selection. Here we also compare different algorithms in the literature with ours.
In Section~\ref{sec:sim}, we test the limits of \QTree\ by a small simulation study. 
Section \ref{sec:summary} concludes with a summary. 
The Supplementary Material includes the proof of the Consistency Theorem (Theorem~\ref{thm:main}) in its Section~\ref{sec:s3}.

\textbf{Notations.} 
Estimators are compared based on standard performance metrics in causal inference (\cite{zheng2018dags}): normalized structural Hamming distance (nSHD), false dicovery rate (FDR), false positive rate (FPR), and true positive rate (TPR). 
We recall their definitions here:
Let $\G$ be the true graph on a node set $V$ and $\hat{\G}$ an estimated graph. The \emph{structural Hamming distance} SHD$(\G,\hat{\G})$ between $\G$ and $\hat{\G}$ is the minimum number of edge additions, deletions and reversals to obtain $\G$ from $\hat{\G}$. 
Denote $E(\calg)$ and $E(\hat{\G})$ the set of edges in $\G$ and $\hat{\G}$, respectively. Note that $\vert E(\hat \G) \setminus E(\G) \vert$ is the number of edges in $\hat \G$ that are not in $\G$, while $\vert E(\hat \G) \cap E(\G)\vert$ is the number of correctly estimated edges. We then have
\begin{align} \label{eq:per_metrics}
 \text{nSHD}(\hat\G, \G) &:=\frac{\text{SHD}(\hat{\G},\G)}{|E(\hat \G)|+|E(\G)|}, & 
    \text{FDR}(\hat \G, \G) & :=\frac{\vert E(\hat \G) \setminus E(\G) \vert }{|E(\hat \G)|},\\
        \text{FPR}(\hat \G, \G) &:=\frac{\vert E(\hat \G) \setminus E(\G) \vert }{|V|\times(|V|-1)-|E( \G)|}, &
        \text{TPR}(\hat \G, \G)& :=\frac{\vert E(\hat \G) \cap E(\G) \vert }{|E(\G)|}\notag.    
\end{align}
All metrics lie in $[0,1]$ and the performance of an algorithm is better the smaller the first three metrics are and the larger TPR is. 
We shall use this throughout Section~\ref{sec:results}.

\section{The algorithm}\label{sec:alg}


\subsection{The data generation model}

Throughout we assume data $X\in\R^{V}$ with causal dependence structure modelled by a root-directed spanning tree $\mathcal{T}$ on $V$ nodes. In such a tree, each node $i \in V$ except the root $r$ has exactly one child, the root $r$ has none, and there is a path from every node $i \neq r$ to $r$. An example of such a tree is given in Figure~\ref{fig:example}. 
We solve the Extremal River Problem by estimating $\mathcal{T}$ from extreme river discharges $X_i$ at nodes $i \in V$. 
Here a {\em river discharge} is the volume of water flowing through a river channel, measured at any given point in cubic metres or cubic feet per second.

Extreme value models have a long tradition in hydrology as the ample references in the Introduction show. River networks are prominent examples for root-directed trees. As the water direction determines the flow, the root-directed tree is known and an ideal test case for a new extreme value model and a new algorithm.

Our starting point is the max-linear Bayesian network (\cite{GK1}), a model for risk propagation in a directed acyclic graph. 
When the graph is a tree $\mathcal{T}$, then the model is defined as
\begin{equation}\label{eqn:main}
X_i = \bigvee_{j: j \to i \in \mathcal{T}}c_{ij}X_j \vee Z_i, \quad c_{ij}, Z_i > 0, \quad i \in V. 
\end{equation}
The $Z_i$, called {\em innovations}, are independent with support $\R_{\ge 0}$ and have atom-free distributions.
Each edge $j \to i$ in $\mathcal{T}$ has a {\em weight} $c_{ij} > 0$, interpreted as some measure of the flow rate from $j$ to $i$, and an extreme discharge at $i$ is either the result of an unknown external input $Z_i$ (e.g. heavy rainfall), or it is the maximum of weighted discharges coming by recursion from an ancestral node of $i$. 

For numerical stability, we prefer to work with the logarithm of the data. 
To avoid new symbols, we keep the same notation, so the max-linear Bayesian tree becomes
\begin{equation}\label{eqn:main.max.plus}
X_i = \bigvee_{j: j \to i \in \T}(c_{ij} + X_j) \vee Z_i,\quad c_{ij},Z_i\in\R, \quad i \in V. 
\end{equation}

Our approach to the Extremal River Problem is to assume that the observations follow {\em approximately} a max-linear model.
This means that we assume that data is corrupted with independent noise in each coordinate such that the problem we want to solve in this paper is the following.

\begin{quote} \textbf{Extremal River Problem}.
Assume i.i.d. observations 
\begin{equation}\label{eqn:x}
\mathcal{X}= \{x^1+\eps^1,\dots,x^n+\eps^n\} \in \R^V,
\end{equation}
where for $k=1,\dots,n$, the components of $x^k$ are generated via \eqref{eqn:main.max.plus}
and the components of the noise vectors $\eps^k$ are independent  in $\R$, find $\T$.
\end{quote}

We stress that the root-directed tree assumption is \emph{different} from the usual tree in Bayesian networks, where each \emph{child} has at most one parent. Learning the single-parent tree can be done with the message passing algorithm, which recursively identifies the parent of a node through likelihood calculations (\cite{wainwright2008graphical}). This strategy does not work for the root-directed tree, since each child can have multiple parents. 

\subsection{Intuition of \QTree}\label{sec:intuition}

In general, learning Bayesian networks with more than one parent is NP-hard (\cite{chickering1996learning}). 
However, learning a max-linear root-directed tree from i.i.d. 
\emph{noise-free observations} is solvable in time $O(|V|^2n)$ with $O(|V|(\log(|V|))^2)$ observations (cf. Section~\ref{sec:s1} in the Supplementary Material). 
Here is the intuition. 

Fix an edge $j \to i$ and consider the noise-free model \eqref{eqn:main.max.plus}. 
If for an observation $x\in\R^V$ the components $i$ and $j$ satisfy
$x_i = c_{ij} + x_j$, then we say {\em $j$ drives $i$}.
If $j$ does not drive $i$, then $x_i > c_{ij} + x_j$. Over $n$ independent observations, if the value at $j$ drives the value at $i$ at least twice, then the distribution of $x_i - x_j$ has an \emph{atom} at its left endpoint. 
Repeating this argument shows that if $j$ drives $k$ and $k$ drives $i$, then $x_i - x_j = c_{ik} + c_{kj}$. That is, if the sample $\mathcal{X}$ is noise-free, the empirical distribution of
\begin{equation}\label{eqn:F.initial}
\calx_{ij} := \{x_i-x_j: x \in \mathcal{X}\}
\end{equation}
has for sufficiently many observations multiple values \emph{at} the minimum of its support if and only if $j\rightsquigarrow i$; i.e., if there is a path from $j$ to $i$. 
Thus, with enough observations, one can recover the directed path $j\rightsquigarrow i$, from which the graph $\mathcal{T}$ can be uniquely constructed as it is a root-directed tree; this is depicted in the histograms of the first row of Figure~\ref{fig:histogram}.

Under the presence of noise, max-linear models can no longer be recovered by means of an atom. However, \QTree\ exploits the above intuition. Consider an ordered pair of nodes $(j,i) \in V$.  If the noise at $i$ is small relative to the signal at $j$, one can expect a concentration of observations \emph{near} the minimum of $\calx_{ij}$ if and only if  $j \rightsquigarrow i$. This is the intuition of \QTree. Throughout, for simplicity, we use the notation $x\in\calx$ for possibly noisy observations.

\begin{figure}[t!]
\includegraphics[width=6.5cm]{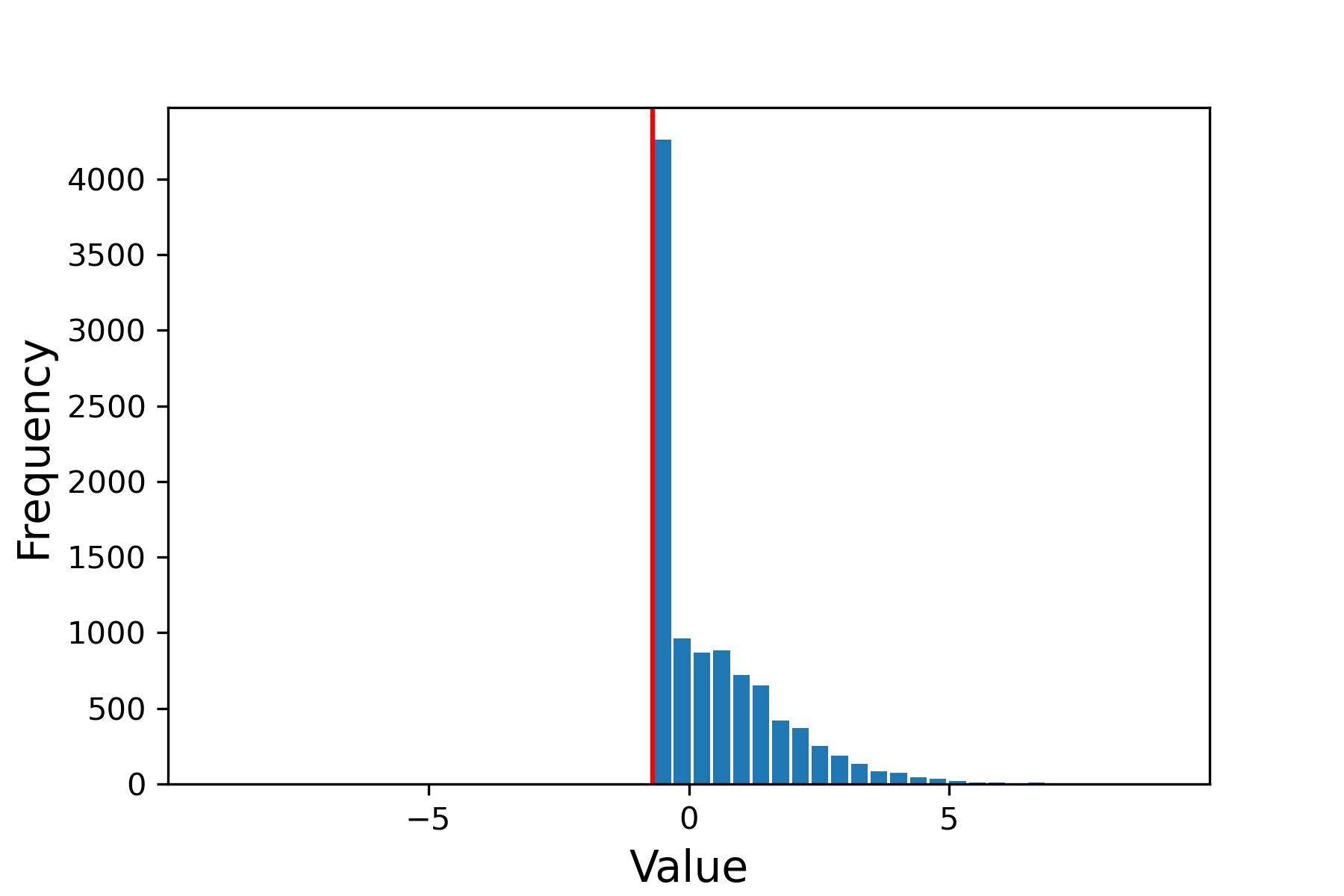} \quad\quad
\includegraphics[width=6.5cm]{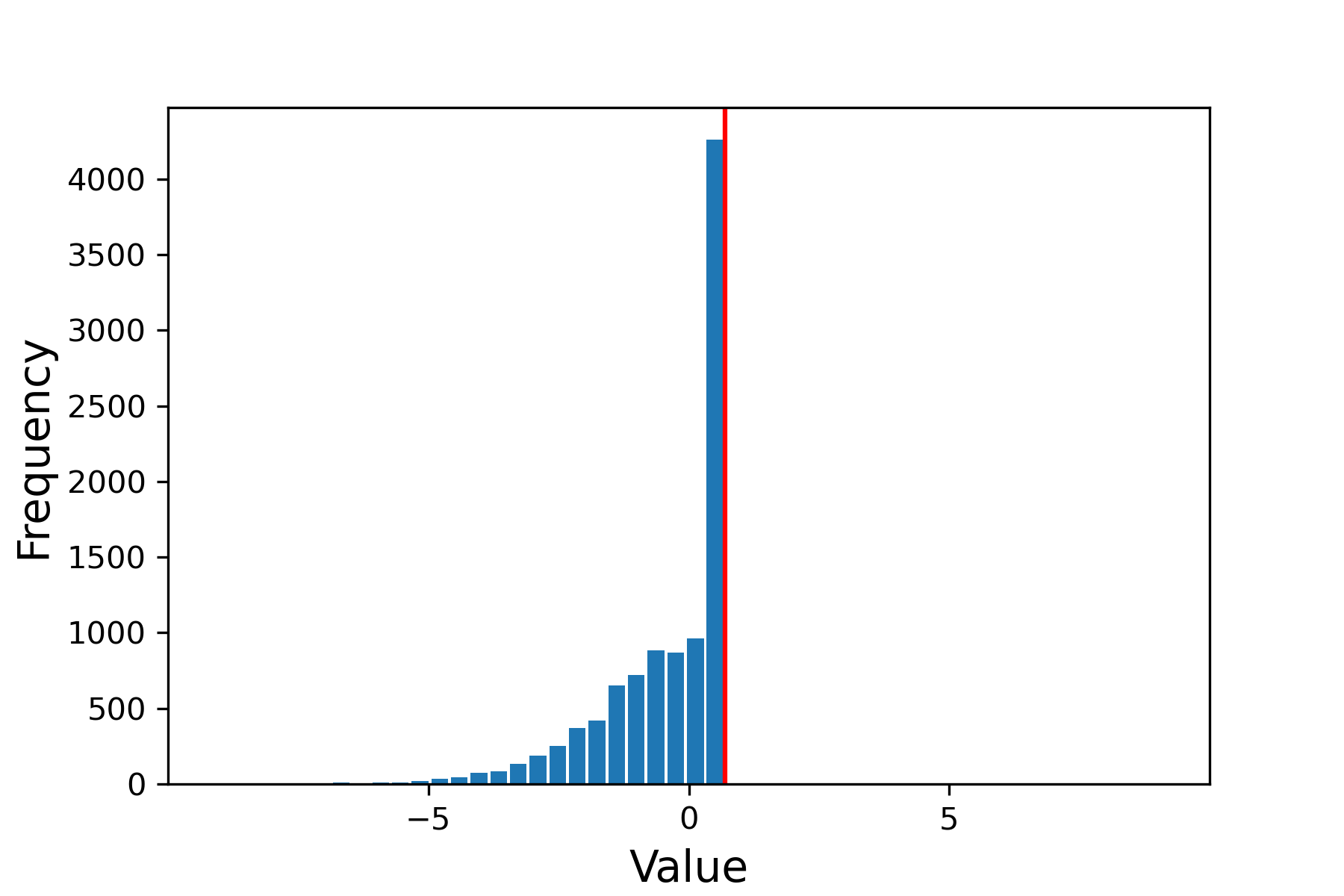}\\
\includegraphics[width=6.5cm] {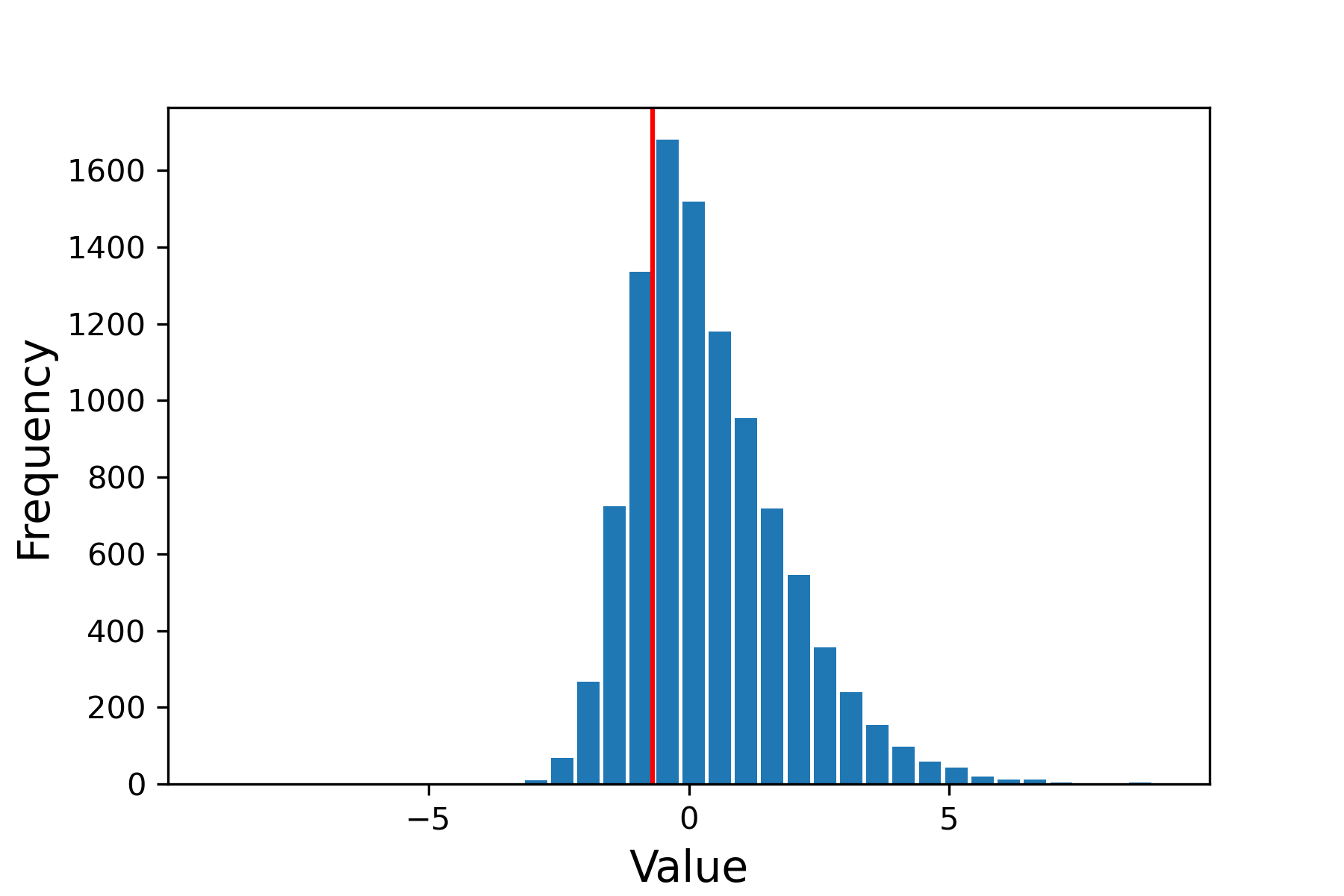} \quad\quad
\includegraphics[width=6.5cm] {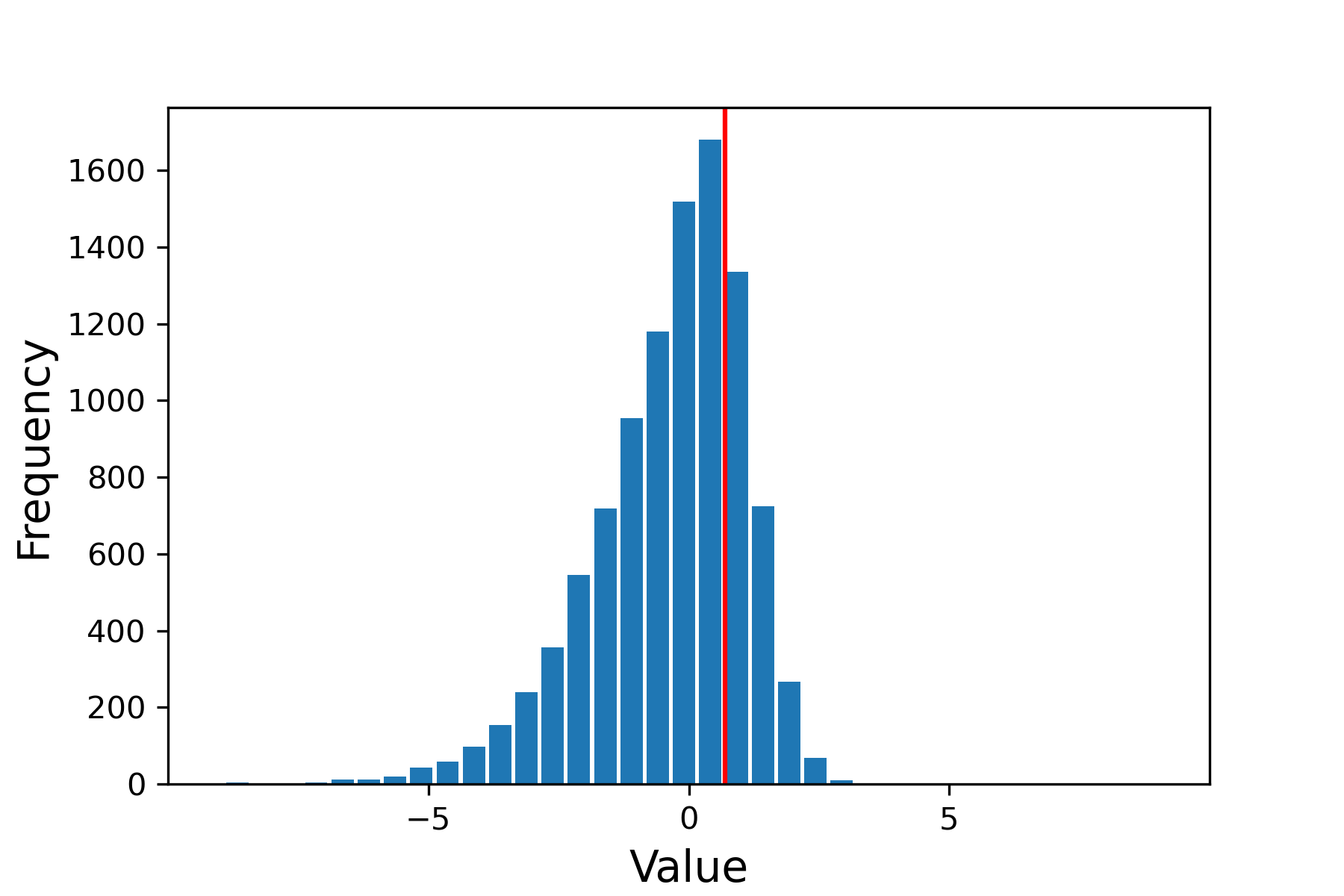}\\
\includegraphics[width=6.5cm] {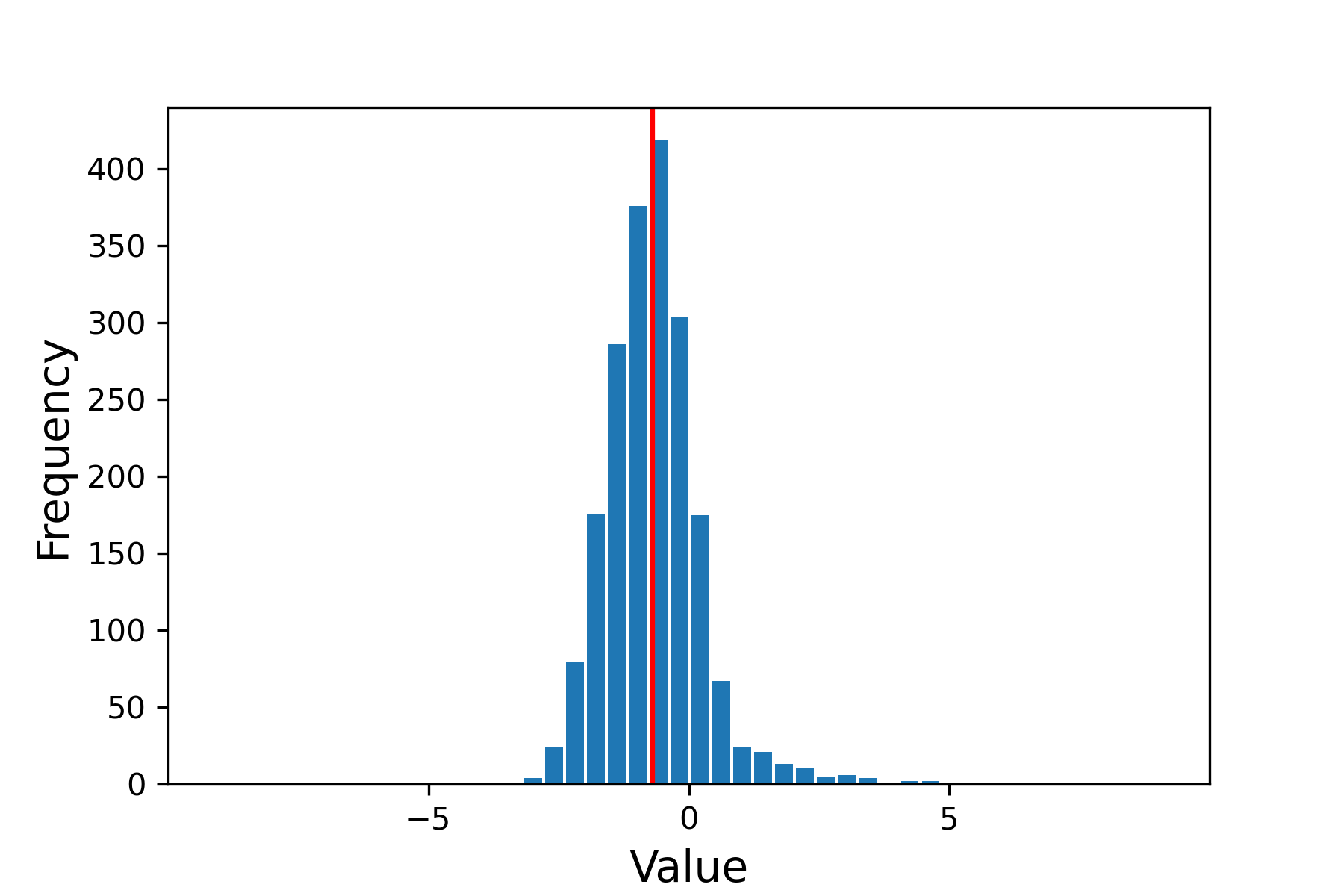} \quad\quad
\includegraphics[width=6.5cm] {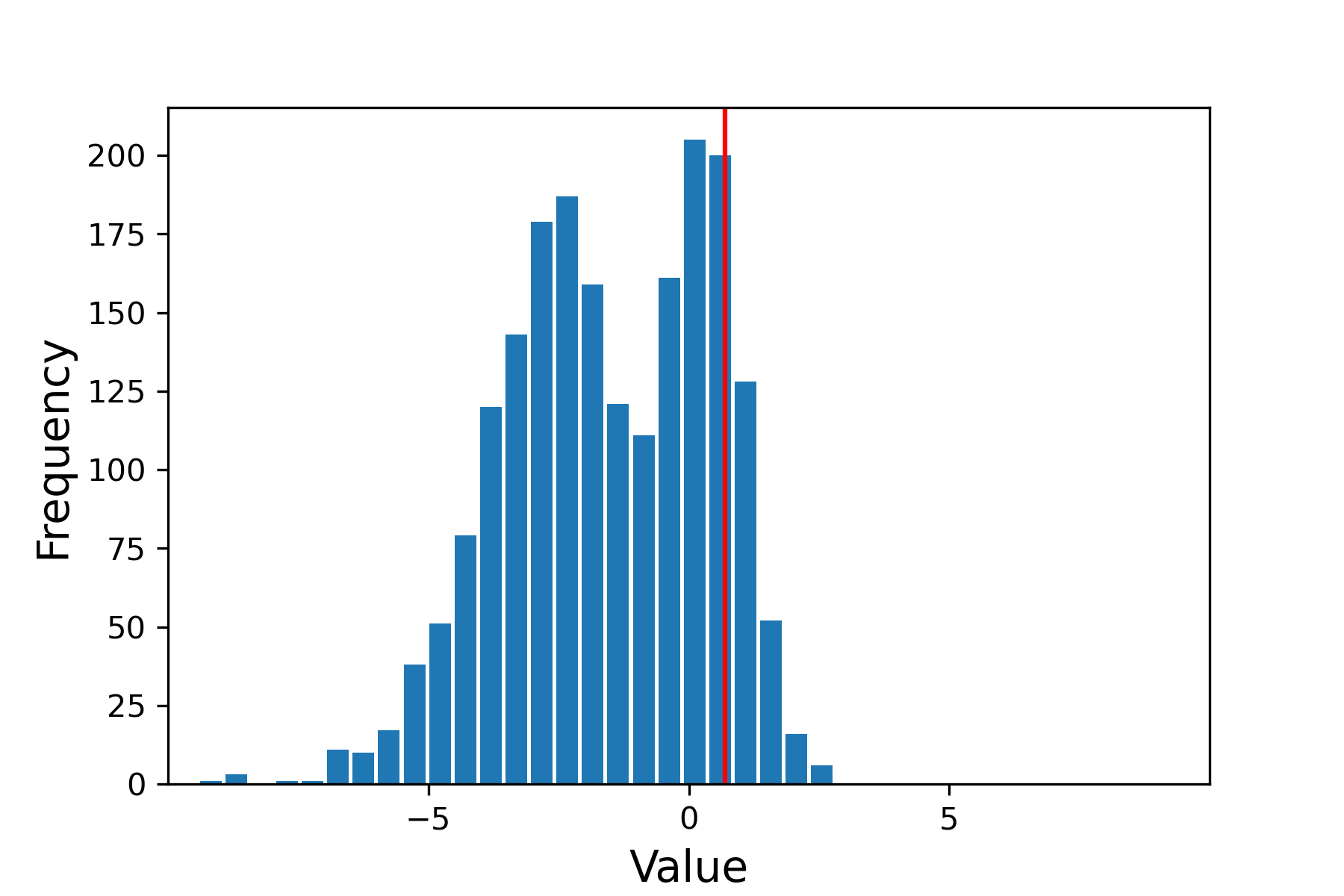}
\caption{For the simple graph $1\to 2$ with $c_{21}=\log(0.5)=-0.69$ and normal centered noise with standard deviation 0.5, {\em the first column} depicts  histograms with red vertical line giving the position of the atom $\log(0.5)$ of $\calx_{21}$ in the noise-free distribution. The upper figure shows the noise-free observations and the middle figure the noisy observations for $\calx_{21}$ as in \eqref{eqn:F.initial}; the lower figure shows the histogram of the noisy, truncated observations $\calx_{21}(0.8)$ as in \eqref{eqn:F.q}.
{\em The second column} shows the same histograms, however, for reversely directed edges, i.e. histograms of $\calx_{12}$ and $\calx_{12}(0.8)$.
In the first column, the lower figure shows a substantial increase of symmetry around $c_{21}$. 
This is because for $\calx_{21}(0.8)$, a large value $x_2$ has a high chance of being realised from a large value of $x_1$, increasing the symmetry around $c_{21}$.
In the lower figure of the second column, a large value $x_2$ is realised either from large $x_1$ or from a large innovation $Z_2$, giving the bimodal distribution.
}\label{fig:histogram}
\end{figure} 

While we have no control over the noise, one way to obtain `strong signals $x_j$' is to replace \eqref{eqn:F.initial} by the set
\begin{equation}\label{eqn:F.q}
\mathcal{X}_{ij}(\alpha) := \{x_i -  x_j: x \in \mathcal{X}, x_j > Q_{\mathcal{X}_j}(\alpha)\},
\end{equation}
where $Q_{\mathcal{X}_j}(\alpha)$ is the $\alpha$-quantile of the empirical distribution of $\mathcal{X}$ in the $j$-th coordinate. For $\alpha > 0$, this amounts to a transformation of $\calx_{ij}$ that amplifies its concentration near the minimum, at the cost of keeping only a fraction of the available observations (cf. Figure~\ref{fig:histogram}). We can then compute empirical scores for every pair $(j,i)$ of vertices based on the concentration of observations around a small quantile. Finally, since the data is supported on a root-directed tree, we can use 
the scores to estimate a root-directed tree.

\subsection{The \QTree\ Algorithm} \label{sec:the_alg}

The \QTree\ Algorithm~\ref{alg:qtree} computes independently for each potential edge $j \to i$ a score $w_{ij}$, seen as a measure of concentration of $\mathcal{X}_{ij}(\alpha)$ near its minimum, then outputs a minimum directed spanning tree of the graph $\calg$ with scores $W = (w_{ij})$. 
The idea is that at each node $i$, data would show the highest concentration at the true edge among all edges from some parent of $i$ to $i$. 
Theorem~\ref{thm:main} proves this for the Gumbel-Gaussian noise model; see \eqref{eqn:noise} and below. The default concentration measure for \QTree\ is the empirical {\em quantile-to-mean gap}
\begin{equation}\label{eqn:wij.tail.mean}
w_{ij}(\smallR) := 
\frac{1}{n_{ij}}\left(\E(\mathcal{X}_{ij}(\al))- Q_{\mathcal{X}_{ij}(\al)}(\smallR) \right)^2,
 \end{equation}
where $\E$ is the empirical mean, $Q$ the empirical quantile, $\smallR \in (0,1)$ is a small quantile level and $n_{ij} = |\mathcal{X}_{ij}(\alpha)|$ is the number of observations in the set $\mathcal{X}_{ij}(\alpha)$ defined in \eqref{eqn:F.q}.

The normalization factor $n_{ij}$ only matters when missing values are unevenly distributed across pairs, such as for the Lower Colorado network (cf.  Section~\ref{sec:lower_col}). Then, pairs with fewer observations get a relative penalty in the concentration estimate to account for larger variability in sample quantile estimates due to a small sample size. If no missing values are present, as is the case with the Upper Danube network, then $n_{ij} = n \cdot (1-\al)$ for all pairs $(j,i)$ and the algorithm would return the exact same tree $\hat{\T}$ as if the concentration measure was defined without dividing by $n_{ij}$. 

We note that there are other choices for a concentration measure, such as the empirical {\em lower quantile gap},
\begin{equation}\label{eqn:wij.tail}
w_{ij}(\smallR,\bigR) :=
\frac{1}{n_{ij}}\left(Q_{\mathcal{X}_{ij}(\alpha)}(\bigR)-Q_{\mathcal{X}_{ij}(\alpha)}(\smallR) \right)^2,
 \end{equation}
where $0 < \smallR < \bigR < 1$ is a fixed pair of quantile levels. If $\bigR$ is small, then $w_{ij}(\smallR,\bigR)$ is a local measure of concentration in the lower tail of $\mathcal{X}_{ij}(\al)$. Note that, if the number of observations is small, then $\bigR$ cannot be too small, so the two empirical concentration measures are in fact rather similar on a real data set. In practice, the lower quantile gap has one more parameter to tune, and thus we choose the quantile-to-mean gap as our default.



\begin{algorithm}\caption{\QTree\ for fixed parameters} \label{alg:qtree}
\flushleft 
\textbf{Parameters}: $\smallR \in (0,1)$, $\al \in [0,1)$. \\
\textbf{Input}: data $\mathcal{X} =\{x^1,\dots,x^n\}\subset \R^{V}$. \\
\textbf{Output}: a root-directed spanning tree $\hat{\T}$ on $V$.
\begin{algorithmic}[1]
\For{$j \to i$, $j,i \in V, j \neq i$}
  \State Compute $w_{ij}(\smallR)$ by \eqref{eqn:wij.tail.mean}.
\EndFor
	\State Compute $\hat{\T} := $ minimum root-directed spanning tree on the directed graph $(V,\G)$ with score matrix $W = (w_{ij}(\smallR)) \in \R^{V \times V}$ with Chu–Liu/Edmonds' algorithm 
	with variable root. 
\State \textbf{Return } $\hat{\T}$
\end{algorithmic}
\end{algorithm}

\begin{remark} \label{remark1}
Given a score matrix $W$ (equivalently a bidirected graph) and a unique root (the initial node), Chu–Liu/Edmonds' algorithm (see \cite{gabow1986efficient}, and \cite{GLS}, Sections~7.2 and~8.4 for more background)
finds a minimum directed spanning tree; i.e., a network of minimum score with $\sum_{j:j\to i\in\hat\T} w_{ij}(\smallR)$ as small as possible.
As we want a minimum {\em root-directed} spanning tree,
we simply reverse edge directions.
Moreover, we run the algorithm for every possible node as root, and take a tree with minimum score.
Finally, provided that all scores are different, the algorithm finds a unique minimum root-directed spanning tree. 
\end{remark}

\begin{remark}
For a given root, the \QTree~Algorithm~\ref{alg:qtree} has complexity $O(|V|^2n)$. For a proof we refer to Lemma~\ref{lem:qtree.complexity} of the Supplementary Material.
\end{remark}

\subsubsection{Theoretical properties of \QTree}

We prove consistency of the \QTree\ Algorithm~1 under natural conditions on the distribution of the innovations. 
We focus on the structural tree model of a max-linear Bayesian network as defined in \eqref{eqn:main} taking i.i.d. innovations with Frech\'et distribution function $P(Z_i\le x)=e^{-x^{-\alpha}}, x>0,$ for $\alpha>0$. Then, using the solution  $X$ of \eqref{eqn:main} given in Theorem~2.2 of \cite{GK1}, by max-stability (e.g. \cite{EKM}, Section~3.2),  $X$ is multivariate Fr\'echet distributed with marginals as in Proposition~A.2 of \cite{GKO}:
$$P(X_i\le x)= \exp\{-(x_i\mu_i)^{-\alpha}\},\quad x>0,$$
for $\mu_i=\big(\sum_{j:j\rightsquigarrow i} {c_{ji}}^{\alpha}\big)^{-1/\alpha}.$
Taking logarithms of the $X_i$, leading to model \eqref{eqn:main.max.plus},  is equivalent to taking logarithms of the innovations $Z_i$ with
$P(\log(Z_i)\le x)=\exp\{-e^{-x/ \beta}\}, x \in \RR$, for $\beta:=1/\alpha>0$.
This results in a Gumbel model with
$$P(X_i\le x)=\exp\{-e^{-( x-\mu_i)/\beta}\},\quad x\in\R.$$
Therefore, for log-data, the $X_i$ are Gumbel$(\beta,\mu_i)$ distributed with scale $\beta:=1/\alpha$ and location parameter $\mu_i$. 
 
Instead of taking the logarithmic analog of a {\em Generalised} Fr\'echet model as often done in the literature, we prefer instead to add a small independent noise to the max-linear Bayesian tree model \eqref{eqn:main.max.plus} with Gumbel$(\beta,0)$ innovations.
This motivates the {\em noise model}
\begin{equation}\label{eqn:noise}
X_i = \Big(\bigvee_{j: j \to i \in \mathcal{T}} (c_{ij}+X_j) \vee Z_i\Big) +\eps_i, \quad c_{ij}, Z_i,\eps_i\in\R,\quad i \in V. 
\end{equation}
with the following innovation-noise distributions:

\begin{quote}
\textbf{Gumbel-Gaussian noise model.} For $i \in V$, the innovations $Z_i$ are i.i.d. Gumbel$(\beta,0)$, the noise variables $\eps_i$ are i.i.d. with symmetric, light-tailed density $f_\eps$ satisfying
\beam\label{eqn:light.tail}
f_{\eps}(x) \sim  e^{-K x^p}\mbox{ as } x \to \infty,
\eeam
for some $p > 1,K > 0$ and such that the derivative of $f_{\eps}$ exists in the tail region.
Throughout, for two functions $a,b$, positive in their right tails, we write $a(x)\sim b(x)$ as $x\to\infty$ for $\lim_{x\to\infty} a(x)/b(x) = c,$ where $c>0$ is some arbitrary constant.
\end{quote}

\begin{remark}
The density $f_\eps$ in \eqref{eqn:light.tail} belongs to a special class of light-tailed densities whose convolution tail can be derived asymptotically (\cite{BKR}). The family includes the Gaussian ($p = 2$), and though it is strictly more general than the Gaussian, we follow \cite{BKR}, and call our noise model Gumbel-Gaussian for ease of reference. Condition \eqref{eqn:light.tail} guarantees that the upper tail of $\eps_i-\eps_j$ is \emph{lighter} than that of $Z_i-Z_j$ (cf. Lemma~\ref{lem:tail} in the Supplementary Material).
\end{remark}

Theorem~1 below, proved in Section~\ref{sec:s3} of the Supplementary Material, says that under the Gumbel-Gaussian noise model, both quantile-to-mean and lower quantile gap produce together with Chu-Liu/Edmonds' algorithm strongly consistent estimators for the true root-directed spanning tree $\mathcal T$ for appropriate choice of parameters. Simulation results (cf. Figure~\ref{fig:simulation}) indicate that the error scales as $O(1/n)$ for \emph{any} fixed graph size $|V|=d$. In particular, for a large graph with $d = 100$, \QTree\ only needs $n = 200$ observations to bring the metrics nSHD to less than 5\% and TPR to more than 95\%; see definitions in \eqref{eq:per_metrics}.

We are now ready to state our main theorem. Observe that, while $\alpha$ as in \eqref{eqn:F.q} is an important tuning parameter, 
we prove the theorem for $\alpha=0$; i.e., by taking the full set of observations.

\begin{theorem}[Consistency Theorem]\label{thm:main}
Assume the Gumbel-Gaussian noise model \eqref{eqn:noise} with distributions specified there.\\
(a) There exists an $r^\ast > 0$ such that for any pair $0 < \smallR < \bigR <  r^\ast$, the \QTree\ Algorithm~\ref{alg:qtree} with score matrix $W=(w_{ij})$ defined by the lower quantile gap $w_{ij}(\smallR,\bigR)$ in \eqref{eqn:wij.tail} returns a strongly consistent estimator for the tree $\T$ as the sample size $n\to\infty$.\\
(b) There exists an $r^\ast > 0$ such that for any $0 < \smallR < r^\ast$, the \QTree\ Algorithm~\ref{alg:qtree} with score matrix $W=(w_{ij})$ defined by the quantile-to-mean gap $w_{ij}(\smallR)$ in \eqref{eqn:wij.tail.mean} returns a strongly consistent estimator for the tree $\T$ as the sample size $n\to\infty$.
\end{theorem}

\begin{remark}\label{rem:fail}
To understand why a condition like \eqref{eqn:light.tail} is necessary, suppose that $V=\{1,2\}$ ($d=2$) and that the true graph is $1 \to 2$. Let $F_{21}$ be the distribution function of $(\eps_2 - \eps_1) + (Z_2 - Z_1) \vee c_{21}$.
The lower tail of $F_{21}$ essentially is the lower tail of $(\eps_2 - \eps_1)$, while the upper tail is essentially the upper tail of the convolution $(\eps_2 - \eps_1) + (Z_2 - Z_1)$, which is dominated by the signal $(Z_2 - Z_1)$ if it has the heavier tail, and otherwise it is dominated by the noise $(\eps_2 - \eps_1)$.
Since $(\eps_2 - \eps_1)$ has symmetric distribution, if the noise term dominates the distribution, 
$w_{12} \approx w_{21}$ and it would be impossible to distinguish the edge $1 \to 2$ from the edge $2 \to 1$. If the signal dominates, the asymmetry between the lower and upper tails of $F_{12}$ lends us the crucial inequality to distinct between the two graphs as illustrated in Figure~\ref{fig:histogram}.

The argument extends to $d>2$ for a graph with only one directed path.
For a realistic matrix with real-valued entries, Chu–Liu/Edmonds' algorithm outputs an approximately correct root-directed tree. Intuitively, reversing every edge direction gives the same score but is generally not a root-directed tree.
\halmos

\begin{figure}[t]
\includegraphics[width=0.48\textwidth] {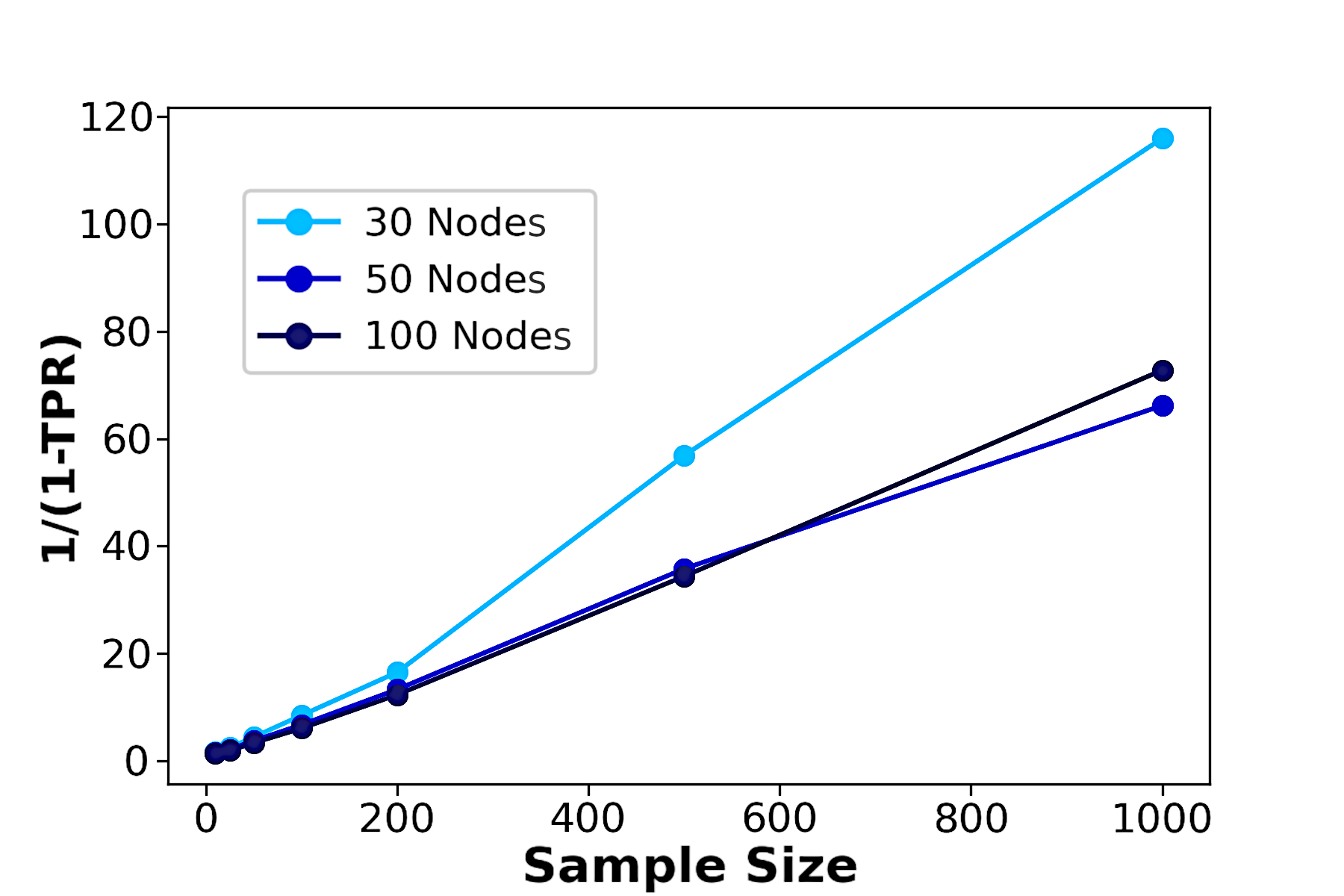} \,
\includegraphics[width=0.48\textwidth] {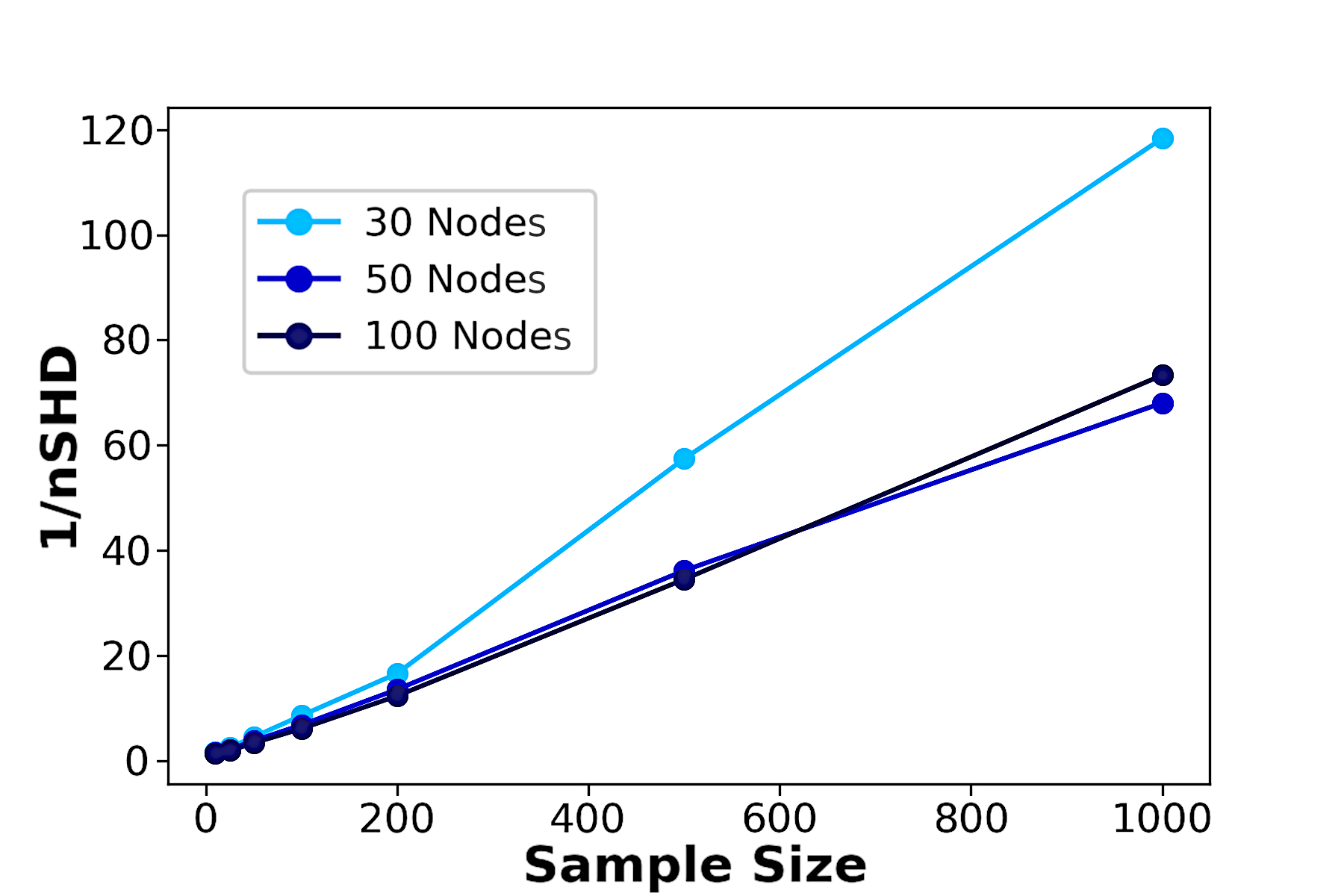}
\caption{$1/(\mbox{mean errors})$ vs. number of observations $n$ for different graph sizes $d = 30, 50, 100$. We simulated 100 root-directed spanning trees as described in Section~\ref{sec:sim} below, where we use the Gumbel-Gaussian noise model.
Then we applied \QTree\ Algorithm~1 with quantile-to-mean gap \eqref{eqn:wij.tail.mean} with $\underline{r}=0.05$ and $\al=0$ to estimate the true (simulated) tree, and computed the average error measured by 1-TPR (left) and nSHD (right) given in \eqref{eq:per_metrics}.}\label{fig:simulation}
\end{figure}
\end{remark}

\subsection{Parameter tuning by bootstrap aggegation}\label{sec:parameter.tuning}

\begin{algorithm}[t]\caption{\texttt{Auto-tuned} \QTree} \label{alg:qtree.automated}
\flushleft 
\textbf{Parameters}: subsampling fraction $f \in [0,1]$, number of subsamples $m \in \mathbb{N}$, a set of parameters $\Theta = \{(\smallR,\al)\} \subset [0,1)^2$ to search over. \\
\textbf{Input}: data $\mathcal{X} = \{x^1,\dots,x^n\} \subset \R^V$. \\
\textbf{Output}: the optimal parameter $(\smallR^\ast,\al^\ast) \in \Theta$ and the corresponding root-directed spanning tree $\hat{\T}_{\max}$ on $V$.
\begin{algorithmic}[1]
\For{$(\smallR,\al) \in \Theta$}
  \For{$\ell = 1,\dots,m$}
    \State Sample without replacement a random subset $\mathcal{X}^\ell$ of $n \cdot f$ observations from $\mathcal{X}$. 
    \State Let $\T^\ell(\smallR,\al)$ be the output of \QTree\ Algorithm~\ref{alg:qtree} fitted on $\mathcal{X}^\ell$. 
    \EndFor
    \State Let $\mathsf{T}(\smallR,\al) = \{\T^\ell(\smallR,\al)): \ell = 1,\dots,m\}$
    \State Compute $S(\mathsf{T}(\smallR,\al))$ by \eqref{eqn:W.T}.
    \State Compute $E(\mathsf{T}(\smallR,\al))$ as the maximum root-directed spanning tree of $S(\mathsf{T}(\smallR,\al))$ per Lemma \ref{lem:E.T}. 
    \State Compute $\Var(\mathsf{T}(\smallR,\al))$ by \eqref{eqn:var.T} 
\EndFor
  \State Define $(\smallR^\ast,\al^\ast) := \arg\min\{\Var(\mathsf{T}(\smallR,\al)): (\smallR,\al) \in \Theta\}$. 
\State \textbf{Return } the optimal pair $(\smallR^\ast,\al^\ast)$ and $\hat\T_{\max}:=E(\mathsf{T}(\smallR^\ast,\al^\ast))$.
\end{algorithmic}
\end{algorithm}

The \QTree\ Algorithm~\ref{alg:qtree} has two parameters: the quantile level $\smallR \in (0,1)$ and the cut-off level $\al \in [0,1)$.
If data came from a noise-free max-linear Bayesian tree, then we should select $\smallR$ as small as possible and $\alpha=0$, and fit \QTree\ on all of the available data $\calx$.   
However, due to the presence of noise, setting $\smallR$ and $\alpha$ too small would make the estimator volatile to large values of the noise variables. 

In this section, we propose in a first step a subsampling procedure to stabilize \QTree\  Algorithm~\ref{alg:qtree} and, in a second step, automatically choose $\smallR$ and $\al$ in \QTree. This results in Algorithm~\ref{alg:qtree.automated}, which we also refer to as \texttt{auto-tuned} \QTree. 

The basic idea is to run an algorithm on multiple subsets of the data, and then average the resulting estimators. This subsampling approach is also called bootstrap aggregation or  bagging; see \citet[Section~8.2.1]{James2013} and \citet{Subsampling} for a variety of subsampling procedures. Since \QTree\ outputs a root-directed tree as its estimator, which is a combinatorial object, one cannot simply take the average of their adjacency matrices, as that would not produce a tree. Instead, we see the set of output trees as a distribution over trees. Then, we solve a second problem, namely, to find the centroid tree E(T) of this distribution, defined as that tree which minimizes the expected Hamming distance to a typical tree (cf. Definition \ref{def:1}). 
Lemma \ref{lem:E.T} below proves that the centroid can be computed with another application of Chu–Liu/Edmonds' algorithm. This ensures that the estimator produced by \texttt{auto-tuned} \QTree\ can be computed quickly (cf. Lemma~\ref{lem:2}). 

Our key indicator for model performance is variability in the estimated tree, that is, whether the tree $\hat{\T}$ and its reachability graph $\hat{\calr}$ output by \QTree\ would change significantly if we fit it to different subsamples of the data. Here, we denote the {\em reachability graph}  $\hat{ \mathcal R}$ of $\hat{\T}$ as the graph that results from drawing an edge between a pair $(j,i)$ whenever there is path from $j$ to $i$ in $\hat{\T}$. We propose the following definition of variability for a distribution of root-directed spanning trees.

\begin{definition}\label{def:1}
Let $V$ be a set of nodes and $\mathsf{T} = \{\T^1,\dots,\T^m\}$ a collection of root-directed spanning trees on $V$, and let $\mathsf{R} = \{\calr^1,\dots,\calr^m\}$ be their corresponding reachability graphs. The {\em centroid} of $\mathsf{T}$, denoted $E(\mathsf{T})$, is the root-directed spanning tree on $V$ that minimizes the sum of normalized structural Hamming distances  defined in \eqref{eq:per_metrics}  as follows:
\begin{equation}\label{eqn:E.T.argmin}
E(\mathsf{T}) := \arg\min\limits_{\mathcal T \in \Psi}\sum_{i=1}^m {\rm nSHD}(\T,\T^i),
\end{equation}
where $\Psi$ is the space of root-directed spanning trees on $V$.

Let $E(\mathsf{R})$ denote the reachability graph of $E(\mathsf{T})$. 
Let $e_{\mathsf{T}}$ be the number of edges of $E(\mathsf{T})$, and $e_{\mathsf{R}}$ be the number of edges of $E(\mathsf{R})$, respectively. 
We define the {\em  variability} of $\mathsf{T}$, denoted $\Var(\mathsf{T})$, as
\begin{equation}\label{eqn:var.T}
\Var(\mathsf{T}) := \frac{1}{e_{\mathsf{T}}}\frac{1}{m}\sum_{i=1}^{m}{\rm nSHD}(\T^i,E(\mathsf{T})) + \frac{1}{e_{\mathsf{R}}}\frac{1}{m}\sum_{i=1}^{m}{\rm nSHD}(\calr^i,E(\mathsf{R})).
\end{equation}
\end{definition}
\medskip

Involving the Hamming distance of the reachability graphs in \eqref{eqn:var.T} penalizes the situation where $\T^i$ and $E(\mathsf{T})$ differ in a few edges low down in the tree, for example, if they have different roots. Such a difference would lead to a small structural Hamming distance between the two trees, but a large structural Hamming distance between their reachability graphs, and in particular, very different river networks. 

The following lemma says that $E(\mathsf{T})$ is a maximum root-directed spanning tree of a particular graph with score matrix $S(\mathsf{T})$ that measures the stability among the trees in $\mathsf{T}$. 
In particular, $E(\mathsf{T})$ can be computed using Chu–Liu/Edmonds' algorithm (choosing the root realizing the minimum score), and thus $\Var(\mathsf{T})$ can be computed in polynomial time. The proof can be found in Section~\ref{sec:s1}.

\begin{lemma}\label{lem:E.T}
Let $V$ be a set of nodes and $\mathsf{T} = \{\T^1,\dots,\T^m\}$ a collection of root-directed spanning trees on $V$. 
Define the {\em stability score matrix} $S:=S(\mathsf{T}) \in \R_{\geq 0}^{V \times V}$ by
\begin{equation}\label{eqn:W.T}
s_{ij} := S(\mathsf{T})_{ij} := \#\{\T \in \mathsf{T}: j \to i \in \T\}. 
\end{equation}
Suppose that the maximum root-directed spanning tree $\T_{\max}$ of the graph on $V$ with score matrix $S(\mathsf{T})$ is unique. Then $E(\mathsf{T}) = \T_{\max}$.
\end{lemma}

\begin{remark}
    The \texttt{auto-tuned} \QTree~Algorithm~\ref{alg:qtree.automated} has complexity $O(|V|^2nm|\Theta|)$. For a proof we refer to Lemma~\ref{lem:2}.
\end{remark}

 \begin{figure}
\centering
\includegraphics[height=7cm] {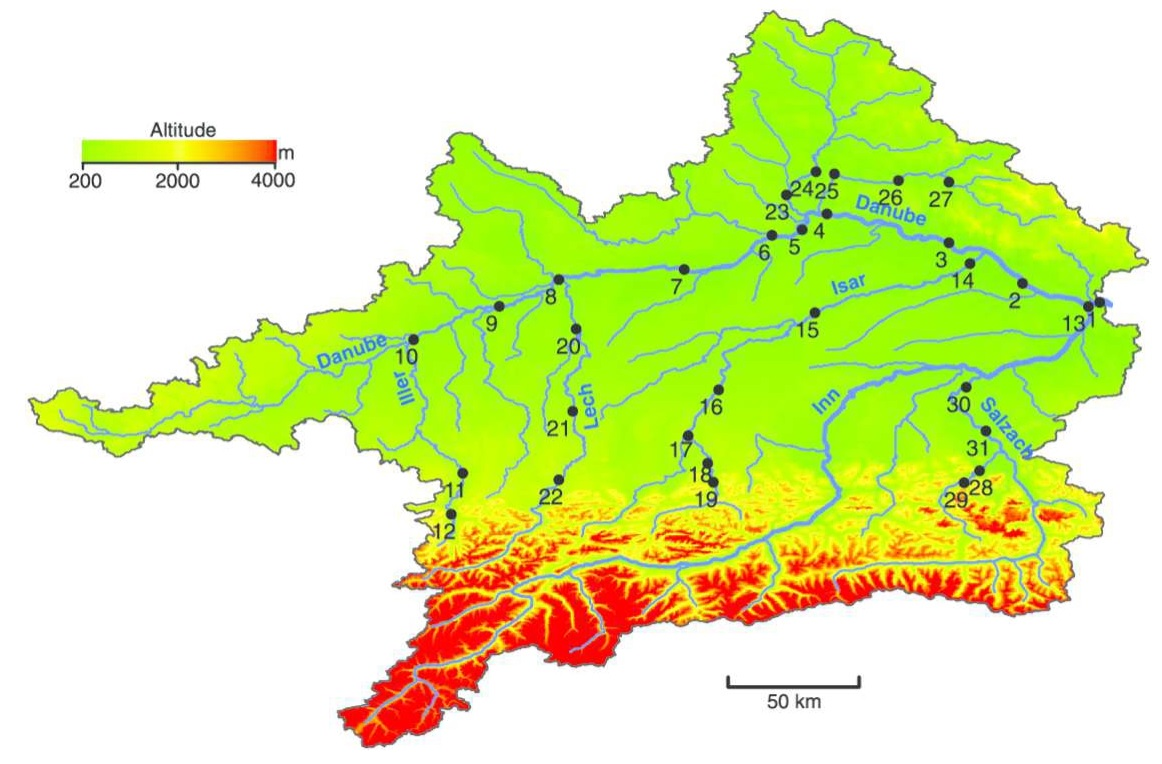}
\caption{\label{fig:danube-true} Topographic map of the Upper Danube Basin, showing the sites of   31 gauging stations along the Danube and its tributaries.} 
\end{figure}

\section{Data description}\label{sec:data_description}

We focus on river discharge data in two river networks, the Upper Danube network with data from Bavaria, Germany, and the Lower Colorado network in Texas, USA.
Large flood events are classical examples for high risk analysis. The Danube data as well as the data of all three sectors of the Colorado are available in the \texttt{Python} package \texttt{QTree}  (\cite{ngoc_alg}). The Danube data are available in the \texttt{R} package \texttt{graphicalExtremes} (\cite{RgraphicalExtremes}).

In general, river discharges across a set of stations is recorded multiple times per hour and some preprocessing is needed to turn the raw data into independent data. This was detailed in \cite{asadi2015} 
for the Danube data. 
We follow their procedure (described in Section~\ref{sec:Danube}) with slight modifications for the Colorado data (described in Section~\ref{sec:lower_col}).

\subsection{The Upper Danube network}\label{sec:Danube}

The Danube network data consist of measurements collected at $d=31$ gauging stations over 50 years from 1960 to 2009 by the Bavarian Environmental Agency (\url{http:www.gkd.bayern.de}); see Figure~\ref{fig:danube-true}.
Preprocessing the data, \cite{asadi2015} first take daily mean values in each time series. 
Their idea is to find non-overlapping time windows of $p$ days, centered around the observation of maximal rank across all time series. For the Danube, the authors choose $p=9$ days ($\pm 4$ days around the observation of maximal rank). For each time series, they then take the maximum within the given time window, delete the data of this window, and proceed until no window of $p$ consecutive days remains.
In order to reduce temporal non-stationarity, in particular, the effect of snow melt, only the months June, July and August are considered.
This results in $n=428$ observations from a $31$-dimensional random vector whose
$i$-th entry corresponds to the maximum water discharge at the $i$-th station, observed within a 9-day window, where at least one station witnessed a large discharge value; these observations are assumed to be independent. 

\begin{figure}[t]
\centering
$\begin{array}{cc}
 \includegraphics[width=6cm] {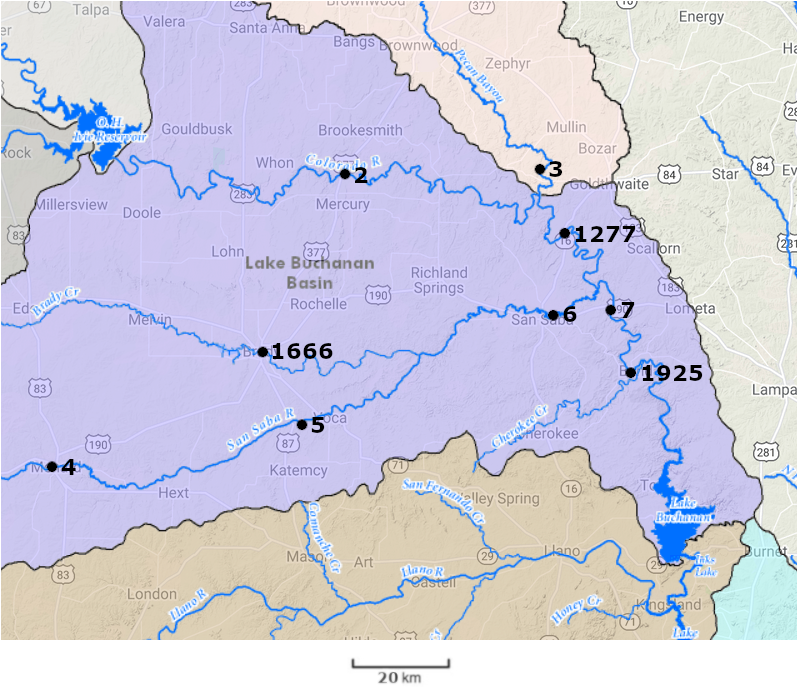}  &
  \includegraphics[width=8cm] {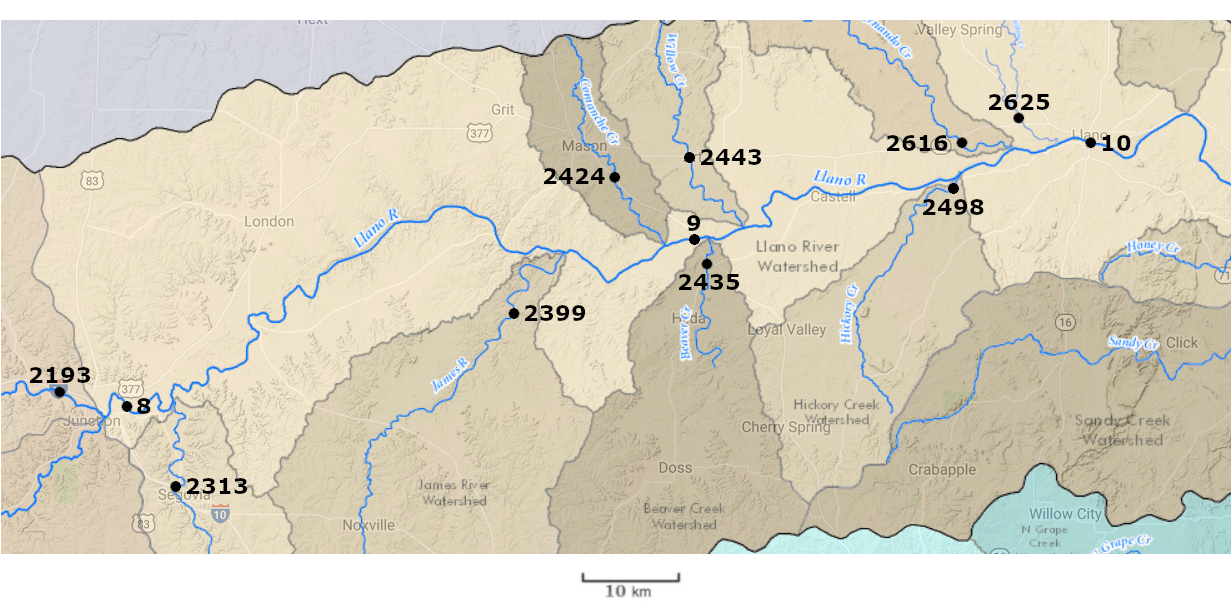} 
 \end{array}
 $
  \includegraphics[height=7cm] {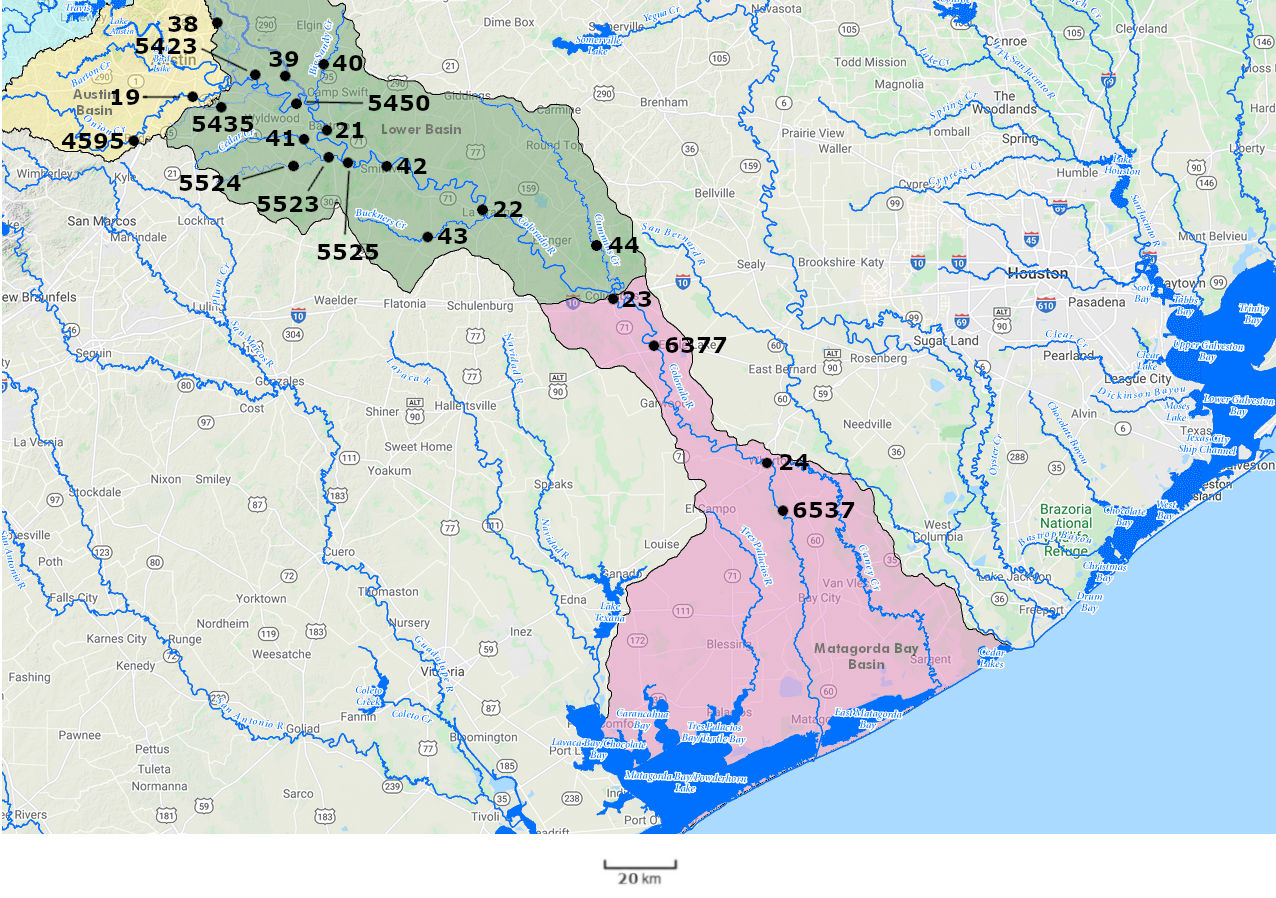}
\caption{\label{fig:upper-col}  Topographic maps of the Top, Middle and Bottom sectors (arranged clockwise) of the Lower Colorado network, showing sites of the gauging stations along the Colorado River and its tributaries. We treat them as three unrelated data sets. }
\end{figure} 

\subsection{The Lower Colorado network in Texas} \label{sec:lower_col}

This section describes the new data set of the Lower Colorado river network in Texas collected by the Lower Colorado River Authority (LCRA, \url{https://www.lcra.org/}) and details their preprocessing.

The Lower Colorado is one of the major rivers in Texas. 
Flowing through major population centers such as Austin, the state capital of Texas, flood and drought mitigation in the Lower Colorado Basin is of prominent interest. 
A particularly challenging feature of the Lower Colorado network is prolonged drought (discharge of $0$) followed by flash flooding which can damage sensors, resulting in loss of data over multiple days. 
This makes the Colorado data much more challenging than the Danube data. 

\begin{figure}
\centering
\includegraphics[height=5cm] {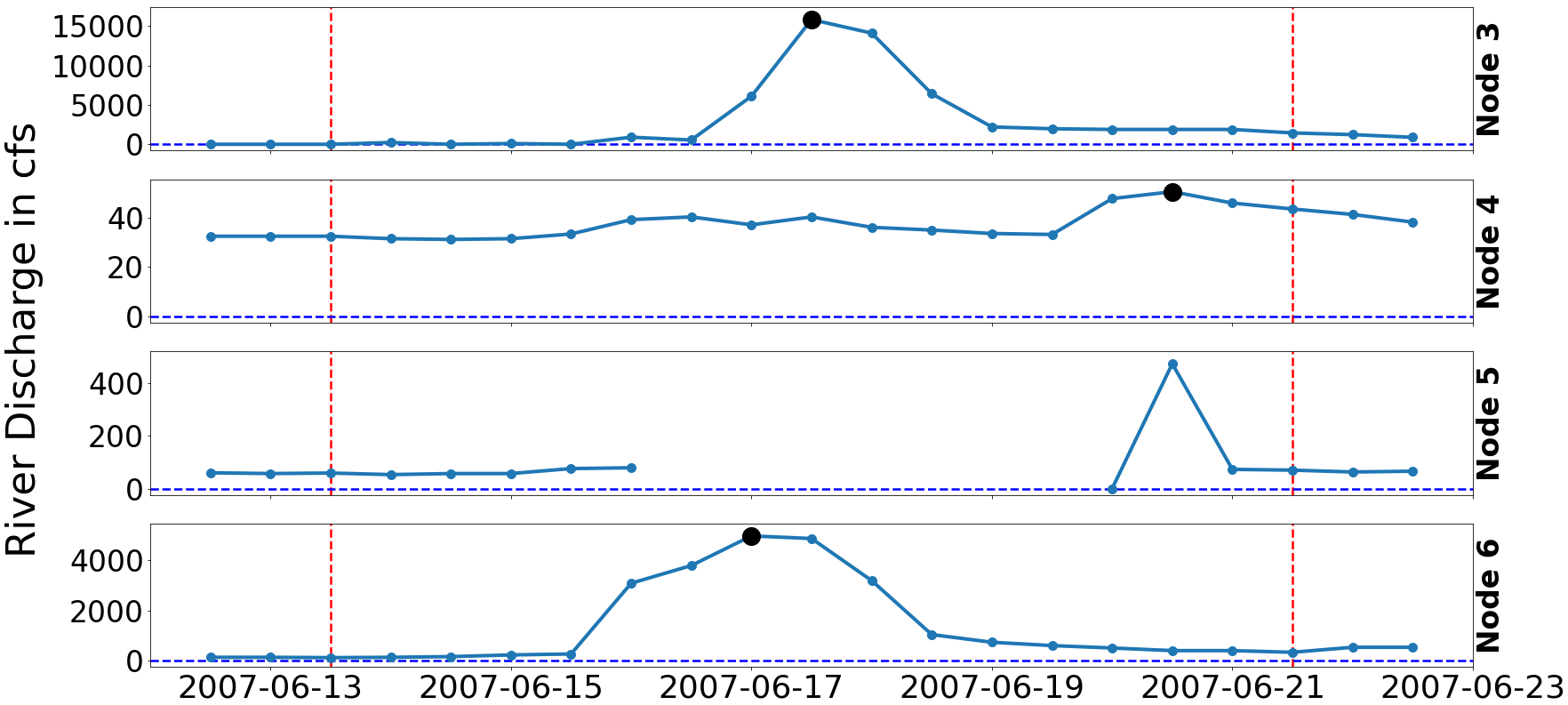}
\caption{\label{fig:declustering}
Typical discharge at various gauging stations around one flood event in the Lower Colorado network. The vertical lines mark a time window of $p=17\times 12h$ or $p =8.5$ days. Dots denote the 12-hour maximum water discharges. At stations 3,4 and 6, the bold dot denotes the peak discharge during this window. For station 5, the sensor did not function during this entire time period, so the peak for station 5 is recorded as missing. Station 3 has a median discharge of only 15 cfs and is hence mostly drained over the measurement period,
but the water flow regularly aggregates to over 16,000 cfs within a very short period of time.} 
\end{figure} 

The river discharges at the Lower Colorado network, measured in cubic feet per second (cfs), are collected multiple times per day at a total of 104 stations around the Colorado River and its tributaries in Texas from the 1st of December 1991 to the 14th of April 2020 (10,363 days). 
 We do not take into account 5 nodes of the Blanco River and San Bernand River, which are not flow-connected to the Colorado River, and also 21 nodes with zero observations.
Moreover, we exclude the nodes 5476, 5634, 5635, 6397, and 6533 as they are located close to hydropower plants. This gives a total of 73 nodes. 

Another problem occurs, because in the Lower Colorado Basin, multiple dams cut off the river into disjoint sectors (\cite{Dams}). Thus, we split the river network such that in each section, we get the largest set of nodes where (i) no node is within 10km of a major dam, (ii) all nodes are connected, and (iii) for each pair $(j, i)$ of this subset, there are at least 1000  pairwise daily observations, which is 9.6\% of the total amount of 10,363 days available. Criterion (iii) ensures that among the given pairs of nodes, any possible causal relation can be discovered, and not be affected by the lack of concurrent data. 
This results in 42 nodes divided into three sectors, which we call the  \emph{Top, Middle and Bottom} sectors of the Lower Colorado (cf. Figure~\ref{fig:upper-col}) with 9, 12 and 21 nodes, respectively. 
From here on we treat these three sectors as three separated, unrelated data sets. 

In contrast to the Danube network with snow melt and seasonal periodicity, we can and do take all Colorado data of a year into account. 
We observe further that, by the special weather conditions, flood events can last as little as a few hours. Therefore, in contrast to \cite{asadi2015}, who take daily time slots, we take 12-hour time slots. As a first step, we take maxima of each 12-hour time slot to retain the knowledge about large possible peaks and such periods where no data are collected. 
In a second step we then take non-overlapping time windows of $p=8.5$ days ($\pm 8$\, 12-hour slots around
the observation of maximal rank); see Figure~\ref{fig:declustering}. 

We take the most conservative approach to missing data, namely, if node $i$ has any missing data during the considered time window, then its maximum discharge over this window is labeled as missing (cf. Figure~\ref{fig:declustering}). 
This is because a sensor can break before the river reaches peak discharge and for practical reasons can only be replaced after the flood event is over (\cite{LCRA2}), and thus the sensor potentially did not measure the largest water discharge that occurred at node~$i$.  This results in the Top sector having 9 nodes, 975 observations, $18\%$ missing data; the Middle sector has 12 nodes, 972 observations, $27\%$ missing data. 
The Bottom sector is most challenging, for it has the most nodes (21 nodes), 961 observations, the highest amount of missing data (37\%), and many nodes around the city of Austin with only a few miles apart from each other. 
The close proximity of nodes induces strong spatial dependence even among nodes that are not flow-connected, making it potentially more challenging to recover the true network. 
Moroever, in the Bottom sector there are many nodes with a very small number of observations due to the many missing observations, and we create a new data set by excluding all nodes with less than 150 observations and refer to them as {\em Bottom150}.
A summary of the available data for each data set is given in Table~\ref{table:dn}.

\begin{table}
\caption{Number of nodes $d$, number of observations $n$ and percentage of missing data used for the algorithmic reconstruction of the river network.\label{table:dn}}
\centering
\fbox{%
\begin{tabular}{c|ccccc}
 & Danube & Top & Middle & Bottom & Bottom150   \\
\hline
$d$ & 31 & 9 & 12 & 21 & 16   \\
\hline
$n$ & 428 & 975 & 972 & 961 & 961   \\
\hline
$\%$ & 0\% & 18\% & 27\% & 37\% & 22\%    \\
\end{tabular}}
\end{table}

\section{Results}\label{sec:results}

\subsection{Results of \texttt{auto-tuned} \QTree\ for all river networks}\label{sec:resultsQTree}

\enlargethispage{2\baselineskip}

For each of the four river networks Danube, Top, Middle and Bottom sectors of the Colorado, we ran \texttt{auto-tuned} \QTree~(Algorithm~\ref{alg:qtree.automated})  with fixed $\smallR = 0.05$,  subsampling rate $f = 0.75$, and number of repetitions $m = 1000$ to choose the tuning parameter $\al$ automatically from $\{0.7, 0.725, 0.75,\allowbreak,\dots, 0.9\}$. 
The optimal parameters $\al^*$ selected  by \texttt{auto-tuned} \QTree\ for these networks are shown in Table~\ref{table:param}. 

\begin{table}
\caption{Optimal parameters $\al^*$ selected by \texttt{auto-tuned} \QTree\ with $\underline{r}=0.05$ using grid search with bootstrap aggregation.\label{table:param}}
\centering
\fbox{%
\begin{tabular}{c|ccccc}
 & Danube & Top & Middle & Bottom & Bottom150      \\
\hline
$\al^*$ & 0.775 & 0.825 & 0.75 & 0.85 & 0.725   \\
\end{tabular}}
\end{table}


Figures~\ref{fig:dir-danube}, \ref{fig:dir-col-mid}, and \ref{fig:dir-col-low-part}(top) show the estimated trees of the Danube, Top, Middle and Bottom sectors of the Colorado, respectively.
We do two estimated-vs-true comparisons: one for the tree, and one for its reachability graph. 
The four performance metrics we use are those defined in  \eqref{eq:per_metrics}. 
Table~\ref{table:qtreescores} gives all metrics over all data sets. We recall that the performance of an algorithm is better the smaller the first three metrics are and the larger TPR is.
\QTree\ performs very well across all data sets except for the Bottom sector of the Colorado. 
For the reachability graph, the statistics are even better, 
indicating that a wrongly estimated edge directs rather from an ancestor (which is not a parent) to a child ({\em flow-connection} is preserved), than a {\em spurious edge }(an edge which contradicts flow-connection). 
The number of {\em missing edges} are determined by the fact that a tree has exactly $d-1$ edges.

Figure~\ref{fig:dir-col-low-part}(top) visualizes the estimation of the Bottom sector of the Colorado. 
This data set is the most challenging due to large portions of missing data and the clustering of nodes around the city of Austin. Nevertheless, even for this data set the estimated tree has only two spurious edges, between $6537$ and $24$ and between $42$ and $5525$. Both of these node pairs are physically close. All the remaining wrongly estimated edges are flow-connected. 

We note that the majority of errors made by \QTree\ involves nodes with less than 150 observations (which are the nodes 5525, 5450, 5423, 5435 and 5524). This is not at all surprising. The model was fitted to only 75\% of the data, and the optimally chosen $\al^*$ is $0.85$, which means that for each edge involving one of the above nodes, the number of observations available to \QTree\ is at most $150 \times 0.75 \times 0.15 = 14$. 
To check the hypothesis that this number is too small for \QTree\ to perform reliably, we excluded all nodes with less than 150 observations and refitted \QTree\ on the remaining 16 nodes (Bottom150), resulting in an optimal $\al^*=0.725$.
The result depicted in Figure~\ref{fig:dir-col-low-part}(bottom) shows significant improvements. This manifests another desirable feature of \QTree, namely, that it relies on local (pairwise) estimation, and thus changes to the node set in one part of the tree do not affect the estimated network elsewhere.

We present details of the parameter selection procedure of \texttt{auto-tuned} \QTree\ for the Danube in Figure~\ref{fig:parameter_sel_danube}. The respective figures for the Colorado data sets can be found in Figures~\ref{fig:parameter_sel_colup}-\ref{fig:parameter_sel_lowcol150} of the Supplementary Material.
As expected from a statistical estimation procedure, the statistical choice of the parameter selection by \texttt{auto-tuned} \QTree\ does not always output the best result on every data set. 
However, it fails only by very few edges to the graph estimated with the best choice of parameters. For example, for the Danube the parameter $\al=0.75$ (instead of the estimated optimal $\al^*=0.775$) would have lead to a better result (cf. Figure~\ref{fig:dir-danube-0.75}). 
Also for the Top sector of the Colorado, $\al=0.9$ would have given perfect recovery of the true network; this is clear as all four metrics become optimal (see Figure~\ref{fig:parameter_sel_colup}).

\enlargethispage{2\baselineskip}

In summary, on all four data sets considered, \texttt{auto-tuned} \QTree\ performed well for nodes with a sufficient number of observations as is obvious from Figure~\ref{fig:dir-col-low-part}(top and bottom). 
The estimated optimal parameter $\al^*$ is either the best one (i.e., the corresponding estimated tree is best across all $\alpha$ as for the Middle and Bottom sectors of the Colorado), or such that it is within one to two wrong edges of the best one (Danube and Top sector of the Colorado).
The method can handle data with missing observations and close spatial proximity between nodes. 

\begin{table}
\caption{Metrics nSHD, FPR, FDR and TPR for \texttt{auto-tuned}  \QTree. Numbers display the metrics for the pairs $(\mathcal T, \hat{ \mathcal T})$ and numbers in brackets for the pairs $(\mathcal R, \hat{ \mathcal R})$ of their respective reachability graphs.
\label{table:qtreescores}
}
\centering
\fbox{%
\begin{tabular}{c|c|cccc}
 & \multirow{2}{*}{Danube} & \multicolumn{4}{c|}{Colorado}   \\
 &   & Top & Middle & Bottom & Bottom150     \\
\hline
nSHD & 0.18(0.09) & 0.13(0.02) & 0.09(0.06) & 0.45(0.15) & 0.10(0.12)   \\
FPR & 0.01(0.02) & 0.04(0.00) & 0.02(0.04) & 0.05(0.03) & 0.02(0.02)   \\
FDR & 0.20(0.05) & 0.13(0.00) & 0.09(0.10) & 0.50(0.02) & 0.13(0.02)   \\
TPR & 0.80(0.87) & 0.88(0.95) & 0.90(1.00) & 0.50(0.74) & 0.87(0.78)  \\
\end{tabular}}
\end{table}

\enlargethispage{2\baselineskip}

\begin{figure}
  \centering
\begin{tikzpicture}[->,>=stealth',shorten >=1pt,auto,node distance=1.2cm,
                thick,main node/.style={circle,draw,font=\bfseries,minimum size=0.7cm,inner sep=0cm},true edge/.style={Green},stream edge/.style={densely dashed,YellowGreen},wrong edge/.style={ red,decorate,
  decoration={zigzag,amplitude=0.4pt,segment length=1.4mm,pre=lineto,pre length=0pt}},missing edge/.style={loosely dotted}]

  \node[main node] (1)[fill={rgb,1:red,0.3; green,0.6; blue,1}] {1};
  \node[main node] (2)[fill={rgb,1:red,0.3; green,0.6; blue,1}] [right of=1] {2};
  \node[main node] (3)[fill={rgb,1:red,0.3; green,0.6; blue,1}] [right of=2] {3};
  \node[main node] (4)[fill={rgb,1:red,0.3; green,0.6; blue,1}] [right of=3] {4};
  \node[main node] (5)[fill={rgb,1:red,0.3; green,0.6; blue,1}] [right of=4] {5};
  \node[main node] (6)[fill={rgb,1:red,0.3; green,0.6; blue,1}] [right of=5] {6};
  \node[main node] (7)[fill={rgb,1:red,0.3; green,0.6; blue,1}] [right of=6] {7};
  \node[main node] (8)[fill={rgb,1:red,0.3; green,0.6; blue,1}] [right of=7] {8};
  \node[main node] (9)[fill={rgb,1:red,0.3; green,0.6; blue,1}] [right of=8] {9};
  \node[main node] (10)[fill={rgb,1:red,0.3; green,0.6; blue,1}] [right of=9] {10};
  \node[main node] (11)[fill={rgb,1:red,0.3; green,0.6; blue,1}] [right of=10] {11};
  \node[main node] (12)[fill={rgb,1:red,0.3; green,0.6; blue,1}] [right of=11] {12};
  \node[main node] (13)[fill={rgb,1:red,0.3; green,0.6; blue,1}] [below of=2] {13};
  \node[main node] (14)[fill={rgb,1:red,0.3; green,0.6; blue,1}] [right of=13] {14};
  \node[main node] (15)[fill={rgb,1:red,0.3; green,0.6; blue,1}] [right of=14] {15};
  \node[main node] (23)[fill={rgb,1:red,0.3; green,0.6; blue,1}] [right of=15] {23};
  \node[main node] (24)[fill={rgb,1:red,0.3; green,0.6; blue,1}] [right of=23] {24};
  \node[main node] (20)[fill={rgb,1:red,0.3; green,0.6; blue,1}] [below of=8] {20};
  \node[main node] (21)[fill={rgb,1:red,0.3; green,0.6; blue,1}] [right of=20] {21};
  \node[main node] (22)[fill={rgb,1:red,0.3; green,0.6; blue,1}] [right of=21] {22};
  \node[main node] (30)[fill={rgb,1:red,0.3; green,0.6; blue,1}] [below of=14] {30};
  \node[main node] (31)[fill={rgb,1:red,0.3; green,0.6; blue,1}] [right of=30] {31};
  \node[main node] (25)[fill={rgb,1:red,0.3; green,0.6; blue,1}] [right of=31] {25};
  \node[main node] (26)[fill={rgb,1:red,0.3; green,0.6; blue,1}] [right of=25] {26};
  \node[main node] (27)[fill={rgb,1:red,0.3; green,0.6; blue,1}] [right of=26] {27};
  \node[main node] (16)[fill={rgb,1:red,0.3; green,0.6; blue,1}] [below of=25] {16};
  \node[main node] (17)[fill={rgb,1:red,0.3; green,0.6; blue,1}] [right of=16] {17};
  \node[main node] (18)[fill={rgb,1:red,0.3; green,0.6; blue,1}] [right of=17] {18};
  \node[main node] (19)[fill={rgb,1:red,0.3; green,0.6; blue,1}] [right of=18] {19};
  \node[main node] (28)[fill={rgb,1:red,0.3; green,0.6; blue,1}] [below of=16] {28};
  \node[main node] (29)[fill={rgb,1:red,0.3; green,0.6; blue,1}] [right of=28] {29};
  
  \path
    (2) edge [true edge] node {} (1)
    (3) edge [true edge] node {} (2)
    (4) edge [true edge] node {} (3)
    (5) edge [true edge] node {} (4)
    (6) edge [true edge] node {} (5)
    (7) edge [true edge] node {} (6)
    (8) edge [missing edge] node {} (7)
    (9) edge [true edge] node {} (8)
    (10) edge [true edge] node {} (9)
    (11) edge [true edge] node {} (10)
    (12) edge [true edge] node {} (11)
    (13) edge [true edge] node {} (1)
    (14) edge [missing edge] node {} (2)
    (15) edge [missing edge] node {} (14)
    (23) edge [true edge] node {} (4)
    (24) edge [true edge] node {} (23)
    (20) edge [missing edge] node {} (7)
    (21) edge [true edge] node {} (20)
    (22) edge [true edge] node {} (21)
    (30) edge [true edge] node {} (13)
    (31) edge [true edge] node {} (30)
    (25) edge [missing edge] node {} (4)
    (26) edge [true edge] node {} (25)
    (27) edge [missing edge] node {} (26)
    (16) edge [true edge] node {} (15)
    (17) edge [true edge] node {} (16)
    (18) edge [true edge] node {} (17)
    (19) edge [true edge] node {} (18)
    (28) edge [true edge] node {} (31)
    (29) edge [true edge] node {} (28)
    (8) edge [stream edge, bend right] node {} (6)
    (20) edge [stream edge] node {} (6)
    (27) edge [stream edge, bend right] node {} (25)
    (15) edge [stream edge] node {} (2)
    (25) edge [wrong edge] node {} (14)
    (14) edge [wrong edge, bend left] node {} (15)
    (14) edge [wrong edge, bend left] node {} (15)
    (10.8,-2.43) edge [true edge] (11.3,-2.43)
    (10.8,-3.13) edge [stream edge] (11.3,-3.13)
    (10.8,-3.83) edge [wrong edge] (11.3,-3.83)
    (10.8,-4.53) edge [missing edge] (11.3,-4.53)
    ;
    \draw[draw=black] (9.6,-5) rectangle ++(5,3);
    \node at (11.7,-2.4) {$e=j \hspace{0.7cm} i: e \in \hat{\mathcal{T}}\cap \mathcal{T}$};
     \node at (12.1,-3.1) 
    {$e=j \hspace{0.7cm} i: e \in \hat{\mathcal{T}}\cap \mathcal{R} \setminus \mathcal{T}$ };
    \node at (11.7,-3.8) 
    {$e=j \hspace{0.7cm} i: e \in \hat{\mathcal{T}}\setminus\mathcal{R}$};
    \node at (11.7,-4.5) 
    {$e=j \hspace{0.7cm} i:  e \in  \mathcal{T}\setminus \hat{\mathcal{T}}$};

\end{tikzpicture}
\caption{\label{fig:dir-danube} Danube river network, estimated by \texttt{auto-tuned} \QTree\ vs. true. 
Solid (green) edges are correct. 
Dashed (green) edges are not in the tree but in the reachability graph, that is, the causal direction or flow-connection is correct. 
Squiggly (red) edges are spurious (neither in the tree nor in the reachability graph).
Dotted (black) edges are in the true tree, but not in the estimated tree.
\QTree\ outputs a tree with only six wrongly estimated edges, four of them flow-connected and one path skipping a single node (the edge $8\to 6$ skips node $7$). Two edges are spurious. 
}
\end{figure}


\tikzset{
  myfillrgb/.code args={#1}{
    \pgfmathparse{1-0.008*#1)}\edef\rcolor{\pgfmathresult}%
    \pgfmathparse{1-0.006*#1)}\edef\gcolor{\pgfmathresult}%
    \pgfmathparse{1}\edef\bcolor{\pgfmathresult}%
    \definecolor{mycolour}{rgb}{\rcolor, \gcolor, \bcolor}%
    \pgfkeysalso{/tikz/fill=mycolour}
  }
}

\newcommand{\mycolorbar}[8]
{
        \foreach \x [count=\c] in {#3}{ \xdef\numcolo{\c}}
        \pgfmathsetmacro{\pieceheight}{#1/(\numcolo-1)}
        \xdef\lowcolo{}
        \foreach \x [count=\c] in {#3}
        {   \ifthenelse{\c = 1}
            {}
            {   \fill[bottom color=\lowcolo,top color=\x] (#7,{#8+(\c-2)*\pieceheight}) rectangle (#7+#2,{#8+(\c-1)*\pieceheight});
            }
            \xdef\lowcolo{\x}
        }
        \draw (#7,#8) rectangle (#7+#2,#8+#1);
        \pgfmathsetmacro{\secondlabel}{#4+#6}
        \pgfmathsetmacro{\lastlabel}{#5+0.01}
        \pgfkeys{/pgf/number format/.cd,fixed,precision=2}
        \foreach \x in {#4,\secondlabel,...,\lastlabel}
        { \draw (#7+#2,{#8+(\x-#4)/(#5-#4)*#1})[-] -- ++ (0.05,0) node[right] {\pgfmathprintnumber{\x}};
        }
}

 \begin{figure}
\centering
\begin{minipage}{.48\textwidth}
\begin{tikzpicture}[->,>=stealth',shorten >=1pt,auto,node distance=1.5cm,
                thick,main node/.style={circle,draw,font=\small\bfseries,minimum size=0.8cm,inner sep=0cm},true edge/.style={Green},stream edge/.style={densely dashed,YellowGreen},wrong edge/.style={ red,decorate,
  decoration={zigzag,amplitude=0.4pt,segment length=1.4mm,pre=lineto,pre length=0pt}},missing edge/.style={loosely dotted}]

  \node[main node,myfillrgb={60}] (1925) {1925};
  \node[main node,myfillrgb={93}] (7) [right of=1925] {7};
  \node[main node,myfillrgb={55}] (1277) [right of=7] {1277};
  \node[main node,myfillrgb={91}] (3) [right of=1277] {3};
  \node[main node,myfillrgb={92}] (6) [below of=1277] {6};
  \node[main node,myfillrgb={91}] (2) [right of=6] {2};
  \node[main node,myfillrgb={95}] (5) [below of=2] {5};
  \node[main node,myfillrgb={94}] (4) [right of=5] {4};
  \node[main node,myfillrgb={42}] (1666) [below of=5] {1666};
  
  \path
    (7) edge [true edge] node {} (1925)
    (1277) edge [true edge] node {} (7)
    (3) edge [true edge] node {} (1277)
    (6) edge [true edge] node {} (7)
    (5) edge [true edge] node {} (6)
    (4) edge [true edge] node {} (5)
    (1666) edge [true edge] node {} (6)
    (2) edge [missing edge] node {} (1277)
    (2) edge [stream edge] node {} (7)
    ;

\end{tikzpicture}
\centering
\end{minipage}
\begin{minipage}{.48\textwidth}
\centering
\begin{tikzpicture}[->,>=stealth',shorten >=1pt,auto,node distance=1.5cm,
                thick,main node/.style={circle,draw,font=\small\bfseries,minimum size=0.8cm,inner sep=0cm},true edge/.style={Green},stream edge/.style={densely dashed,YellowGreen},wrong edge/.style={ red,decorate,
  decoration={zigzag,amplitude=0.4pt,segment length=1.4mm,pre=lineto,pre length=0pt}},missing edge/.style={loosely dotted}]

  \node[main node,myfillrgb={95}] (10) at (0,1) {10};
  \node[main node,myfillrgb={95}] (9)  at (1.7,1) {9};
  \node[main node,myfillrgb={93}] (8) [right of=9] {8};
  \node[main node,myfillrgb={35}] (2193) [right of=8] {2193};
  \node[main node,myfillrgb={70}] (2435)  at (1.7,-0.3) {2435};
  \node[main node,myfillrgb={65}] (2313) [right of=2435] {2313};
  \node[main node,myfillrgb={68}] (2443) at (1.7,-1.6) {2443};
  \node[main node,myfillrgb={70}] (2399) [right of=2443] {2399};
  \node[main node,myfillrgb={61}] (2498) at (1.7,-2.9) {2498};
  \node[main node,myfillrgb={65}] (2424) [right of=2498] {2424};
  \node[main node,myfillrgb={70}] (2616) at (1.7,-4.2) {2616};
  \node[main node,myfillrgb={54}] (2625) at (1.7,-5.5) {2625};
  
  \path
    (9) edge [true edge] node {} (10)
    (8) edge [true edge] node {} (9)
    (2193) edge [true edge] node {} (8)
    (2435) edge [true edge] node {} (10)
    (2443) edge [true edge] node {} (10)
    (2498) edge [true edge] node {} (10)
    (2616) edge [true edge] node {} (10)
    (2313) edge [true edge] node {} (9)
    (2399) edge [true edge] node {} (9)
    (2424) edge [true edge] node {} (9)
    (2625) edge [missing edge] node {} (10)
    ;
      \node[main node] (2616)[fill={rgb,1:red,0.3; green,0.6; blue,1}] at (1.7,-4.2) {2616};
      \node[main node,myfillrgb={61}] (2498) at (1.7,-2.9) {2498};
    \mycolorbar{7}{0.4}{{rgb,1:red,1; green,1; blue,1}, {rgb,1:red,0.2; green,0.4; blue,1}}{0}{1000}{200}{5.7}{-5.7}
    
    \begin{scope}[on background layer]
        \draw  (2625) edge [wrong edge] node {} (8);
    \end{scope}

\end{tikzpicture}
\end{minipage}
\caption{\label{fig:dir-col-mid} Top (left) and Middle (right) sectors of the Colorado network, estimated by \texttt{auto-tuned} \QTree\ vs. true. Node colors represent the amount of available data  after taking care of missing data. 
Edges are as described in Figure~\ref{fig:dir-danube}. Both estimated networks only contain one single wrongly estimated edge. Top:  one edge wrong but flow-connected; Middle: one edge spurious.}
\end{figure}

\begin{figure}
\centering
\begin{minipage}[b]{\textwidth}
\begin{tikzpicture}[->,>=stealth',shorten >=1pt,auto,node distance=1.11cm,
                thick,main node/.style={circle,draw,font=\scriptsize\bfseries,minimum size=0.75cm,inner sep=0cm},true edge/.style={Green},stream edge/.style={densely dashed,YellowGreen},wrong edge/.style={ red,decorate,
  decoration={zigzag,amplitude=0.4pt,segment length=1.4mm,pre=lineto,pre length=0pt}},missing edge/.style={loosely dotted}]
                
  \node[main node,myfillrgb={67}] (6537) {6537};
  \node[main node,myfillrgb={93}] (24) [right of=6537] {24};
  \node[main node,myfillrgb={52}] (6377) [right of=24] {6377};
  \node[main node,myfillrgb={94}] (23) [right of=6377] {23};
  \node[main node,myfillrgb={94}] (22) [right of=23] {22};
  \node[main node,myfillrgb={82}] (42) [right of=22] {42};
  \node[main node,myfillrgb={15}] (5525) [right of=42] {5525};
  \node[main node,myfillrgb={94}] (21) [right of=5525] {21};
  \node[main node,myfillrgb={14}] (5450) [right of=21] {5450};
  \node[main node,myfillrgb={15}] (5423) [right of=5450] {5423};
  \node[main node,myfillrgb={70}] (38) [right of=5423] {38};
  \node[main node,myfillrgb={82}] (43) [below of=22] {43};
  \node[main node,myfillrgb={53}] (5523) [below of=21] {5523};
  \node[main node,myfillrgb={70}] (39) [right of=5523] {39};
  \node[main node,myfillrgb={14}] (5435) [right of=39] {5435};
  \node[main node,myfillrgb={92}] (19) [right of=5435] {19};
  \node[main node,myfillrgb={61}] (4595) [right of=19] {4595};
  \node[main node,myfillrgb={82}] (44) [below of=43] {44};
  \node[main node,myfillrgb={72}] (40) [below of=39] {40};
  \node[main node,myfillrgb={35}] (41) [below of=40] {41};
  \node[main node,myfillrgb={14}] (5524) [below of=41] {5524};
  
  \path
    (6377) edge [true edge] node {} (24)
    (23) edge [true edge] node {} (6377)
    (22) edge [true edge] node {} (23)
    (43) edge [true edge] node {} (23)
    (44) edge [true edge] node {} (23)
    (5450) edge [true edge] node {} (21)
    (39) edge [true edge] node {} (21)
    (41) edge [true edge] node {} (5523)
    (5524) edge [true edge] node {} (5523)
    (4595) edge [true edge] node {} (19)
    (24) edge [missing edge] node {} (6537)
    (42) edge [missing edge] node {} (22)
    (5525) edge [missing edge] node {} (42)   
    (21) edge [missing edge] node {} (5525)
    (40) edge [missing edge] node {} (21)
    (5523) edge [missing edge] node {} (5525)
    (5435) edge [missing edge] node {} (5450)
    (19) edge [missing edge] node {} (5423)
    (5423) edge [missing edge] node {} (5450)
    (19) edge [missing edge] node {} (5423)
    (38) edge [missing edge] node {} (5423)
    (6537) edge [wrong edge, bend right] node {} (24)
    (5525) edge [stream edge, bend right] node {} (22)
    (21) edge [stream edge, bend right] node {} (42)
    (5423) edge [stream edge, bend right] node {} (21)
    (38) edge [stream edge, bend right] node {} (21)
    (19) edge [stream edge] node {} (21)
    (5523) edge [stream edge] node {} (42)
    (42) edge [wrong edge, bend right] node {} (5525)
    ;
  
    \mycolorbar{5.4}{0.35}{{rgb,1:red,1; green,1; blue,1}, {rgb,1:red,0.2; green,0.4; blue,1}}{0}{1000}{200}{13}{-5.1}

    \begin{scope}[on background layer]
        \draw   (5435) edge [stream edge, bend left] node {} (23);
        \draw     (40) edge [stream edge] node {} (42);
    \end{scope}
    
\end{tikzpicture}
\end{minipage}
\begin{minipage}[b]{\textwidth}
\begin{tikzpicture}[->,>=stealth',shorten >=1pt,auto,node distance=1.37cm,
                thick,main node/.style={circle,draw,font=\small\bfseries,minimum size=0.85cm,inner sep=0cm},true edge/.style={Green},stream edge/.style={densely dashed,YellowGreen},wrong edge/.style={ red,decorate,
  decoration={zigzag,amplitude=0.4pt,segment length=1.4mm,pre=lineto,pre length=0pt}},missing edge/.style={loosely dotted}]
                
  \node[main node,myfillrgb={67}] (6537) {6537};
  \node[main node,myfillrgb={93}] (24) [right of=6537] {24};
  \node[main node,myfillrgb={52}] (6377) [right of=24] {6377};
  \node[main node,myfillrgb={94}] (23) [right of=6377] {23};
  \node[main node,myfillrgb={94}] (22) [right of=23] {22};
  \node[main node,myfillrgb={82}] (42) [right of=22] {42};
  \node[main node,myfillrgb={94}] (21) [right of=42] {21};
  \node[main node,myfillrgb={70}] (39) [right of=21] {39};
  \node[main node,myfillrgb={82}] (43) [below of=22] {43};
  \node[main node,myfillrgb={82}] (44) [below of=43] {44};
  \node[main node,myfillrgb={53}] (5523) [below of=21] {5523};
  \node[main node,myfillrgb={72}] (40) [right of=5523] {40};
  \node[main node,myfillrgb={70}] (38) [below of=40] {38};
  \node[main node,myfillrgb={92}] (19) [below of=38] {19};
  \node[main node,myfillrgb={35}] (41) [below of=19] {41};
  \node[main node,myfillrgb={61}] (4595) [right of=19] {4595};

  \path
    (6377) edge [true edge] node {} (24)
    (23) edge [true edge] node {} (6377)
    (22) edge [true edge] node {} (23)
    (42) edge [true edge] node {} (22)
    (21) edge [true edge] node {} (42)
    (39) edge [true edge] node {} (21)
    (40) edge [true edge] node {} (21)
    (38) edge [true edge] node {} (21)
    (19) edge [true edge] node {} (21)
    (41) edge [true edge] node {} (5523)
    (4595) edge [true edge] node {} (19)
    (43) edge [true edge] node {} (23)
    (44) edge [true edge] node {} (23)
    (5523) edge [missing edge] node {} (42)
    (24) edge [missing edge] node {} (6537)
    (5523) edge [stream edge] node {} (22)
    (6537) edge [wrong edge, bend right] node {} (24)
    ;

    \mycolorbar{6.5}{0.4}{{rgb,1:red,1; green,1; blue,1}, {rgb,1:red,0.2; green,0.4; blue,1}}{0}{1000}{200}{13}{-6.3}

\end{tikzpicture}
\end{minipage}
\caption{\label{fig:dir-col-low-part} Bottom sector of the Colorado network (Bottom and Bottom150), estimated by \texttt{auto-tuned} \QTree\ vs. true. 
Top Figure: Bottom, based on all 21 nodes, \QTree\ outputs a tree with ten wrongly estimated edges, eight of them flow-connected, two spurious edges pointing in the wrong direction. Bottom Figure: Bottom150, based on 16 nodes, after removing nodes with less than 150 observations.
There are only two wrongly estimated edges, one flow-connected, one spurious edge pointing in the wrong direction. Compared to the Bottom sector, this is a significant improvement. 
Node colors represent the amount of available data after taking care of missing data.
Edges are as described in Figure~\ref{fig:dir-danube}.}
\end{figure}
\begin{figure}[t]
  \centering
\begin{tikzpicture}[->,>=stealth',shorten >=1pt,auto,node distance=1.3cm,
                thick,main node/.style={circle,draw,font=\bfseries,minimum size=0.7cm,inner sep=0cm},true edge/.style={Green},stream edge/.style={densely dashed,YellowGreen},wrong edge/.style={ red,decorate,
  decoration={zigzag,amplitude=0.4pt,segment length=1.4mm,pre=lineto,pre length=0pt}},missing edge/.style={loosely dotted}]

  \node[main node] (1)[fill={rgb,1:red,0.3; green,0.6; blue,1}] {1};
  \node[main node] (2)[fill={rgb,1:red,0.3; green,0.6; blue,1}] [right of=1] {2};
  \node[main node] (3)[fill={rgb,1:red,0.3; green,0.6; blue,1}] [right of=2] {3};
  \node[main node] (4)[fill={rgb,1:red,0.3; green,0.6; blue,1}] [right of=3] {4};
  \node[main node] (5)[fill={rgb,1:red,0.3; green,0.6; blue,1}] [right of=4] {5};
  \node[main node] (6)[fill={rgb,1:red,0.3; green,0.6; blue,1}] [right of=5] {6};
  \node[main node] (7)[fill={rgb,1:red,0.3; green,0.6; blue,1}] [right of=6] {7};
  \node[main node] (8)[fill={rgb,1:red,0.3; green,0.6; blue,1}] [right of=7] {8};
  \node[main node] (9)[fill={rgb,1:red,0.3; green,0.6; blue,1}] [right of=8] {9};
  \node[main node] (10)[fill={rgb,1:red,0.3; green,0.6; blue,1}] [right of=9] {10};
  \node[main node] (11)[fill={rgb,1:red,0.3; green,0.6; blue,1}] [right of=10] {11};
  \node[main node] (12)[fill={rgb,1:red,0.3; green,0.6; blue,1}] [right of=11] {12};
  \node[main node] (13)[fill={rgb,1:red,0.3; green,0.6; blue,1}] [below of=2] {13};
  \node[main node] (14)[fill={rgb,1:red,0.3; green,0.6; blue,1}] [right of=13] {14};
  \node[main node] (15)[fill={rgb,1:red,0.3; green,0.6; blue,1}] [right of=14] {15};
  \node[main node] (23)[fill={rgb,1:red,0.3; green,0.6; blue,1}] [right of=15] {23};
  \node[main node] (24)[fill={rgb,1:red,0.3; green,0.6; blue,1}] [right of=23] {24};
  \node[main node] (20)[fill={rgb,1:red,0.3; green,0.6; blue,1}] [below of=8] {20};
  \node[main node] (21)[fill={rgb,1:red,0.3; green,0.6; blue,1}] [right of=20] {21};
  \node[main node] (22)[fill={rgb,1:red,0.3; green,0.6; blue,1}] [right of=21] {22};
  \node[main node] (30)[fill={rgb,1:red,0.3; green,0.6; blue,1}] [below of=14] {30};
  \node[main node] (31)[fill={rgb,1:red,0.3; green,0.6; blue,1}] [right of=30] {31};
  \node[main node] (25)[fill={rgb,1:red,0.3; green,0.6; blue,1}] [right of=31] {25};
  \node[main node] (26)[fill={rgb,1:red,0.3; green,0.6; blue,1}] [right of=25] {26};
  \node[main node] (27)[fill={rgb,1:red,0.3; green,0.6; blue,1}] [right of=26] {27};
  \node[main node] (16)[fill={rgb,1:red,0.3; green,0.6; blue,1}] [below of=25] {16};
  \node[main node] (17)[fill={rgb,1:red,0.3; green,0.6; blue,1}] [right of=16] {17};
  \node[main node] (18)[fill={rgb,1:red,0.3; green,0.6; blue,1}] [right of=17] {18};
  \node[main node] (19)[fill={rgb,1:red,0.3; green,0.6; blue,1}] [right of=18] {19};
  \node[main node] (28)[fill={rgb,1:red,0.3; green,0.6; blue,1}] [below of=16] {28};
  \node[main node] (29)[fill={rgb,1:red,0.3; green,0.6; blue,1}] [right of=28] {29};

  \path
    (2) edge [true edge] node {} (1)
    (3) edge [true edge] node {} (2)
    (4) edge [true edge] node {} (3)
    (5) edge [true edge] node {} (4)
    (6) edge [true edge] node {} (5)
    (7) edge [true edge] node {} (6)
    (8) edge [missing edge] node {} (7)
    (9) edge [true edge] node {} (8)
    (10) edge [true edge] node {} (9)
    (11) edge [true edge] node {} (10)
    (12) edge [true edge] node {} (11)
    (13) edge [true edge] node {} (1)
    (14) edge [true edge] node {} (2)
    (15) edge [true edge] node {} (14)
    (23) edge [true edge] node {} (4)
    (24) edge [true edge] node {} (23)
    (20) edge [missing edge] node {} (7)
    (21) edge [true edge] node {} (20)
    (22) edge [true edge] node {} (21)
    (30) edge [true edge] node {} (13)
    (31) edge [true edge] node {} (30)
    (25) edge [missing edge] node {} (4)
    (26) edge [true edge] node {} (25)
    (27) edge [missing edge] node {} (26)
    (16) edge [true edge] node {} (15)
    (17) edge [true edge] node {} (16)
    (18) edge [true edge] node {} (17)
    (19) edge [true edge] node {} (18)
    (28) edge [true edge] node {} (31)
    (29) edge [true edge] node {} (28)
    (27) edge [stream edge, bend right] node {} (25)
    (20) edge [stream edge] node {} (6)
    (8) edge [stream edge, bend right] node {} (6)
    (25) edge [wrong edge] node {} (14)
    ;
       \matrix (m) at (12.5,-4) [%
      matrix of nodes, row sep=0.1cm, 
    ] {%
      & \QTree     \\
      nSHD & 0.13(0.06)    \\
       FPR & 0.01(0.01)    \\
      FDR & 0.13(0.02)    \\ 
     TPR & 0.87(0.90)    \\\\
   };
\end{tikzpicture}
\caption{\label{fig:dir-danube-0.75} Danube river network, estimated by \QTree\ vs. true for $\al=0.75$.  Compared to Figure~\ref{fig:dir-danube}, the edges $15\to 14$ and $14\to 2$ are here correctly estimated. Also the performance measures at the right bottom of the figure compare favourably to those in the first column of Table~\ref{table:qtreescores}. }
\end{figure}
\begin{figure}
    \centering
  \includegraphics[width=\textwidth] {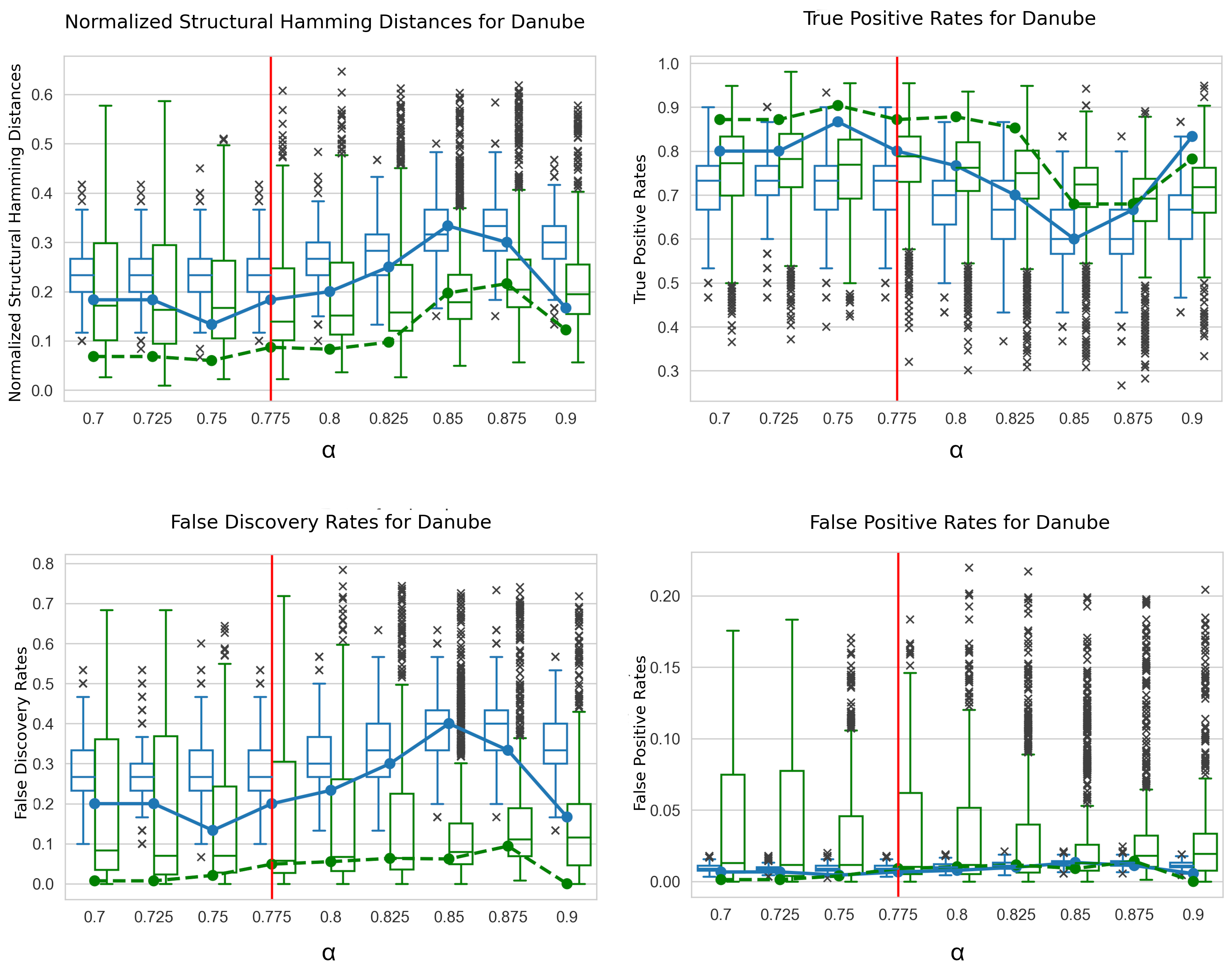}
    \caption{Metrics nSHD, TPR, FDR and FPR for the output $\hat{\mathcal T}_\al$ of the steps of  \texttt{auto-tuned} \QTree\ for varying parameter $\al$ for the Danube network. We subsample $75\%$ of the data 1000 times (Step 3 of Algorithm 2). 
    For each $\al$ we fit \QTree\ on these subsamples to obtain 1000 estimated trees
$\mathsf{T}_\al:=\{ \hat\T_\al^1,\ldots,\hat\T_\al^{1000} \}$ (Step 4 of Algorithm 2).
    The metrics of $(\hat{\mathcal T}_\al^\ell, \mathcal T)$ and $(\hat{\calr}_\al^\ell, \calr)$ are represented in boxplots, one for the tree (blue) and one for the reachability graph (green).
    The blue and green dots present the four metrics for the centroid $E(\mathsf{T}_\al)$ (Step 7 of Algorithm 2) and its reachability graph.
    The lines are interpolations for better visibility, solid blue for the tree and dashed green for the reachability graph.  The chosen $\al^\ast$ is the parameter with the least variablity in $\mathsf{T}_\al$ (Step 9 of Algorithm 2), indicated by a red vertical line.
}\label{fig:parameter_sel_danube}
\end{figure}

\clearpage

\subsection{Comparison to other scores in the literature} \label{sec:comparison}

We compare \QTree\ and \texttt{auto-tuned} \QTree\ with existing algorithms for extremal causal estimation in the literature. 
They are all based on pairwise extreme dependence measures, which we define below.
We then consider their empirical versions as scores and also include the quantile-to-mean gap for comparison:
\begin{enumerate}
\item The empirical {\em quantile-to-mean gap} as in \eqref{eqn:wij.tail.mean} for fixed $(\smallR,\alpha)=(0.05,0.9)$.
We fix $\smallR=0.05$ as we have used this throughout, and $\alpha=0.9$ as an arbitrary parameter.
    \item The {\em causal tail coefficient} $\Gamma_{ij}=\lim_{u\to 1}\E[F_i(X_i)\mid F_j(X_j)>u]$ (\cite{gnecco2019causal}, eq.~(3)). Observe that in \cite{gnecco2019causal}, the algorithm \textsf{EASE} outputs from the estimated score matrix $\Gamma$ an estimated causal order of the set of nodes, not a directed graph. 
    \item The {\em causal score} $S_{ij}^{\text{ext}}$ is based on expected quantile scores (\cite{mhalla2020causal}, eq.~(15)).\footnote{We want to thank Linda Mhalla for helping us to set up the \textsf{CausEv} implementation.}
    The goal of the authors is to discover causality in a directed graph modelled by flow-connection. The scores $S_{ij}^{\text{ext}}$ satisfy $S^{\text{ext}}_{ji}+S^{\text{ext}}_{ij}=1$
and an extreme observation at node $j$ causes an extreme observation at node $i$, whenever $S_{ij}^{\text{ext}}>0.5$. 
The authors propose a bootstrap method to generate 95\% confidence bounds to guarantee that the score is larger than 0.5. All these scores are interpreted as directed edges between nodes.
The algorithm \textsf{CausEV} outputs flow-connections induced by all scores larger than 0.5. 
Thus, their causal edges rather resemble the edges in the reachability graph of the tree.
The paper also treats the example of the Danube data in its Section 5 (see also its Figure 7). 
     \item The {\em tail dependence coefficient} $\chi_{ij}=\lim_{u\to 1} P(F_i(X_i)>u\mid F_j(X_j)>u)$, also called {\em extremal correlation}. It goes back to \cite{Sibuya1960}; see also \cite{colesetal}. 
      We add this dependence measure in our comparison as it is {\em the} classic one, having been used for more than 60 years in multivariate extreme value statistics.
Theoretical properties of the tail dependence coefficient in a max-linear Bayesian network have been investigated in \cite{GKO}. 
It has values in $[0,1]$ and a large value of $\chi_{ij}$ indicates strong extreme dependence. Its empirical estimator takes $u$ large, but finite, and the estimator is based on 10\% of the data.
The paper \cite{EV} uses $\chi$ and related measures on p. 14: the extremal correlation, the extremal variogram, and the combined extremal variogram for estimating undirected trees with Prim's algorithm (\cite{prim}). While the matrix $\chi$ is symmetric, causal inference is possible with Chu–Liu/Edmonds' algorithm; see Remark~\ref{rem:fail}.
\end{enumerate}

The scores discussed in (b)-(d) have not been used to estimate a directed tree, but as they are all pairwise scores, their respective estimated matrices can serve as input for Chu–Liu/Edmonds' algorithm. This gives a fair comparison, as we use then for all scores the knowledge that the true graph is a root-directed spanning tree. 
We keep all tuning parameters used for $\Gamma_{ij}$ and $S^{\text{ext}}_{ij}$ the same as in the respective papers. 

\begin{figure}
    \centering
    \includegraphics[width=0.8\textwidth] {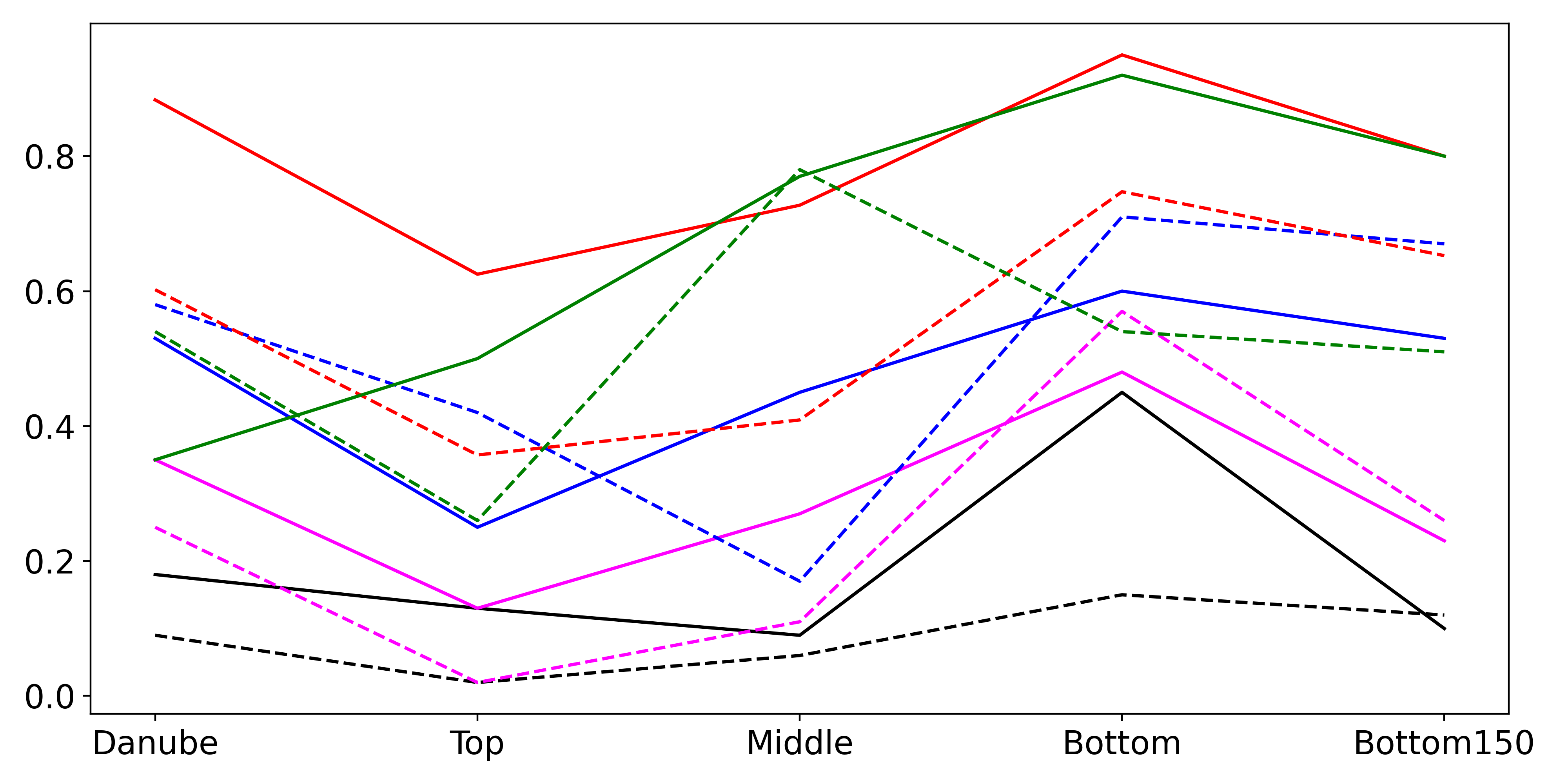}
\caption{Performance metric nSHD for all five data sets on the horizontal axis.
Solid lines display the metrics for the pairs $(\mathcal T, \hat{ \mathcal T})$ and dashed lines for the pairs $(\mathcal R, \hat{ \mathcal R})$ of their respective reachability graphs.
Black lines are used for \texttt{auto-tuned} \QTree, magenta lines for \QTree\ (both based on the score in (a)),
blue  lines for $\Gamma$ as in (b), 
red lines for $S^{\text{ext}}$  as in (c), 
and green lines for $\chi$ as in (d). 
}
\label{fig:qtree_comparison}
\end{figure}


Our comparison is two-fold. 
\begin{enumerate}
    \item[(i)]
We compare the new score of a quantile-to-mean gap with other scores from the literature.
\item[(ii)]
We compare \QTree\ with \texttt{auto-tuned} \QTree\ assessing the benefit of the stabilizing subsampling procedure.
\end{enumerate}

Tables~\ref{table:comparisongnecco}, \ref{table:comparisonmhalla} and \ref{table:comparisonotto} present the metrics nSHD, FPR, FDR and TPR based on the three scores (b), (c), and (d) for each data set previously considered. 
Comparing the metrics with those of \QTree\ in Table~\ref{table:qtreescoresalg1}, we find for the Danube network a comparably weak performance for all three alternative scores.
Among the three scores, $\chi$ performs best, followed by $\Gamma$. 
We visualize our findings from Table~\ref{table:qtreescores} and Tables~\ref{table:qtreescoresalg1}---\ref{table:comparisonotto} for nSHD in Figure~\ref{fig:qtree_comparison}. 

Solid lines present nSHD for the estimated tree vs. true  tree.
 We find that \texttt{auto-tuned} \QTree\ (black line) is uniformly best over all data sets, followed by \QTree. We conclude that the stabilizing subsampling procedure of \texttt{auto-tuned} \QTree\ improves the estimation. 
 Even for the Top sector of the Colorado, where  $\alpha=0.9$ is optimal, the nSHD is still positive for \QTree\, whereas for \texttt{auto-tuned} \QTree, nSHD is equal to zero and the estimated tree is equal to the true one. For all Colorado sectors, $\Gamma$ follows next, $\chi$ is surprisingly successful for the Danube, but not for any of the Colorado sectors. $S^{\text{ext}}$ does not perform well in any tree recovery, but this was also not the goal of \cite{mhalla2020causal}.

Dashed lines present nSHD for the reachability graphs of the estimated trees vs. true. 
Here \texttt{auto-tuned} \QTree\ and \QTree\ give the same answers as for the trees above. The blue dashed line representing $\Gamma$ is only moderately worse than the magenta line for \QTree\ for the Middle and Bottom sectors of the Colorado. The score $\chi$ is worst for the Middle Colorado; for the Danube and the other sectors of the Colorado it performs better than $\Gamma$ and $S^{\text{ext}}$. 

\begin{table}
\caption{Metrics nSHD, FPR, FDR and TPR for \QTree~Algorithm~\ref{alg:qtree} with $\alpha=0.9$. Numbers display the metrics for the pairs $(\mathcal T, \hat{ \mathcal T})$ and numbers in brackets for the pairs $(\mathcal R, \hat{ \mathcal R})$ of their respective reachability graphs.
\label{table:qtreescoresalg1}
}
\centering
\fbox{%
\begin{tabular}{c|c|cccc}
 & \multirow{2}{*}{Danube} & \multicolumn{4}{c|}{Colorado}   \\
 &   & Top & Middle & Bottom & Bottom150     \\
\hline
nSHD & 0.35(0.25) & 0.13(0.02) & 0.27(0.11) & 0.48(0.57) & 0.23(0.26)   \\
FPR  & 0.01(0.02) & 0.02(0.00) & 0.02(0.03) & 0.03(0.15) & 0.02(0.01)   \\
FDR  & 0.40(0.14) & 0.13(0.00) & 0.27(0.19) & 0.60(0.58) & 0.27(0.02)   \\
TPR  & 0.60(0.64) & 0.88(0.95) & 0.72(1.00) & 0.40(0.23) & 0.73(0.59)  \\
\end{tabular}}
\end{table}

\begin{table}
\caption{Metrics nSHD,  FPR, FDR and TPR for the maximum root-directed spanning tree estimated by Chu–Liu/Edmonds' algorithm with score matrix $\Gamma$ as in \cite{gnecco2019causal}, eq.~(8).
\label{table:comparisongnecco}
}
\centering
\fbox{%
\begin{tabular}{c|c|cccc}
 & \multirow{2}{*}{Danube} & \multicolumn{4}{c|}{Colorado}   \\
 &  & Top & Middle & Bottom & Bottom150  \\
\hline
nSHD & 0.53(0.58) & 0.25(0.42) & 0.45(0.17) & 0.60(0.71) & 0.53(0.67)   \\
FPR & 0.02(0.18) & 0.05(0.12) & 0.04(0.03) & 0.04(0.17) & 0.04(0.27)   \\
FDR & 0.73(0.71) & 0.38(0.40) & 0.45(0.21) & 0.70(0.78) & 0.67(0.83)   \\
TPR & 0.27(0.37) & 0.63(0.43) & 0.55(0.88) & 0.30(0.10) & 0.33(0.12)   \\
\end{tabular}}
\end{table}

\begin{table}
\caption{Metrics nSHD, FPR, FDR and TPR  for the maximum root-directed spanning tree estimated by Chu–Liu/Edmonds' algorithm with score matrix $S^{\text{ext}}$ as in \cite{mhalla2020causal}, eq.~(15). 
\label{table:comparisonmhalla}
}
\centering
\fbox{%
\begin{tabular}{c|c|cccc}
 & \multirow{2}{*}{Danube} & \multicolumn{4}{c|}{Colorado}   \\
 &  & Top & Middle & Bottom & Bottom150    \\
\hline
nSHD & 0.88(0.60) & 0.63(0.36) & 0.73(0.41) & 0.95(0.75) & 0.80(0.65)   \\
FPR & 0.03(0.10) & 0.08(0.18) & 0.07(0.12) & 0.05(0.12) & 0.05(0.17)  \\
FDR & 0.90(0.60) & 0.63(0.43) & 0.73(0.52) & 0.95(0.68) & 0.80(0.68)  \\
TPR & 0.10(0.33) & 0.38(0.57) & 0.27(0.76) & 0.05(0.12) & 0.20(0.17) \\
\end{tabular}}
\end{table}

\begin{table}
\caption{Metrics nSHD, FPR, FDR and TPR  for the maximum root-directed spanning tree estimated by Chu–Liu/Edmonds' algorithm with score matrix $\chi$ as in \cite{Sibuya1960} or \cite{colesetal}. 
\label{table:comparisonotto}
}
\centering
\fbox{%
\begin{tabular}{c|c|cccc}
 & \multirow{2}{*}{Danube} & \multicolumn{4}{c|}{Colorado}   \\
 &  & Top & Middle & Bottom & Bottom150    \\
\hline
nSHD & 0.35(0.54) & 0.50(0.26) & 0.77(0.78) & 0.92(0.54) & 0.80(0.51)   \\
FPR  & 0.02(0.21) & 0.06(0.25) & 0.07(0.33) & 0.05(0.27) & 0.06(0.30)   \\
FDR  & 0.47(0.69) & 0.50(0.45) & 0.82(0.93) & 0.95(0.69) & 0.87(0.69)   \\
TPR  & 0.53(0.48) & 0.50(0.76) & 0.18(0.18) & 0.05(0.26) & 0.13(0.28)  \\
\end{tabular}}
\end{table}

We conclude that for the Danube as well as for the various sectors of the Colorado, \texttt{auto-tuned} \QTree\ outperforms uniformly all algorithms without the stabilizing subsampling procedure; see  Figure~\ref{fig:qtree_comparison}. 
Moreover, the quantile-to-mean gap score outperforms the other scores when applying Chu-Liu/Edmonds' algorithm; therefore, we conclude that the quantile-to-mean gap \eqref{eqn:wij.tail.mean} is superior to the other scores on all data sets considered.

\section{A small simulation study}\label{sec:sim}

Theorem~1 ensures strong consistency of the output trees of \QTree, when the sample size $n$ tends to infinity.
In this section we show the quality of \QTree\ through the two performance metrics nSHD and TPR as defined in \eqref{eq:per_metrics} for varying number of observations $n$ and graph sizes $d$ by a small simulation study.

We generate data $\calx$ from a max-linear Bayesian tree as defined in equation \eqref{eqn:main.max.plus} with $|V|=d$ nodes.
For each node $i\in V$, we calculate the sample standard deviation $\hat \sigma_{X_i}$ of $(X^1_i, \ldots, X^n_i)$ and take the sample median $\hat \sigma$ over all nodes in $V$.
We then generate i.i.d. normally distributed noise variables $\eps^t_i$ with mean zero and standard deviation $k\cdot \hat \sigma$ for $i\in V$ and $t=1,\dots,n$. For the {\em noise-to-signal ratio} $k$, we choose $k=30\%$.

We generate root-directed spanning trees as follows.
We first generate a random undirected spanning tree of size $d$ using the graph generators module \texttt{networkX} (Release 2.8.8) in \texttt{Python} (\cite{NetworkX}). 
We then choose the root uniformly at random, which uniquely determines the root-directed spanning tree. 

For the distributions of the innovations $Z_1, \ldots, Z_d$ and the choice of independent edge weights $c_{ij}$, we consider the following three settings:

\begin{enumerate}
    \item[(1)]
    Innovations $Z_1,\dots,Z_d$ are independent Gumbel$(1,0)$ distributed and for every edge, we draw an edge weight $c_{ij}$ from the interval $[\log (0.1),\log(1)]$ uniformly.  We refer to this as the \emph{standard Gumbel setting}.
        \item[(2)]
    Innovations $Z_1,\dots,Z_d$ are independent Gumbel$(1,0)$ distributed and for every edge, we draw an edge weight $c_{ij}$ from the interval $[\log (0.1),\log(0.3)]$ uniformly. We refer to this as the \emph{weak dependence setting}.
    \item[(3)]
    50\% of the innovations $Z_1,\ldots,Z_d$ are Gumbel$(1,0)$ and 50\% are standard normally distributed. For every edge, we draw an edge weight $c_{ij}$ from the interval $[\log (0.1),\log(1)]$ uniformly. We refer to this as the \emph{mixed distribution setting}.
\end{enumerate}

For the score, we take the quantile-to-mean gap as in \eqref{eqn:wij.tail.mean} (normalization by $n_{ij}$ is not needed in this simulation setting) given by  
$$w_{ij}(\smallR) :=   \left(\E(\mathcal{X}_{ij}(\al))- Q_{\chi_{ij(\al)}}(\smallR)\right)^2
$$
and apply the \QTree~Algorithm~\ref{alg:qtree} with parameters $\underline{r}=0.05$ and $\alpha=0$;
we remark that a sensitivity analysis has shown that altering $\underline{r}$ influences the results only insignificantly. 

We use graph sizes $d = 10, 30, 50, 100$ and $100$ repetitions. 
For each repetition, we calculate nSHD and TPR,
and then take the mean over all 100 repetitions. 

Since \QTree\ performs so well not only on simple data like those from the Danube network, but also on all sectors of the Colorado network, we guess that it is fairly robust towards the strength of dependence, given by the $c_{ij}$ and even different node distributions.
The weak dependence setting (2) should manifest whether \QTree\ is also able to recover the underlying network if the dependence given by the weights $c_{ij}$ is substantially smaller.
We want to quantify robustness towards node distributions with the mixed distribution setting (3). 

Figures~\ref{fig:SHD} and~\ref{fig:TPR} depict the mean nSHD and TPR for the standard Gumbel setting (1) and all four graph sizes. 
Both metrics quickly tend to zero, respectively one, as the sample size $n$ increases. Moreover, comparing the four subfigures for a fixed sample size $n$, the metrics perform only slightly worse for increasing graph size $d$.  

Figures~\ref{fig:SHD-smalldep} and~\ref{fig:TPR-smalldep} in Section~\ref{sec:simSupp} of the Supplementary Material  depict the mean nSHD and TPR for the weak dependence setting (2) and all four graph sizes. 
Again, both metrics quickly tend to zero, respectively one. 
In comparison to the previous setting (1), the performance is slightly worse. 

Figures~\ref{fig:SHD-mixed} and~\ref{fig:TPR-mixed} depict the mean nSHD and TPR for the mixed distribution setting (3) and all four graph sizes. 
Despite the different distributions of the innovations, both metrics quickly tend to zero, respectively one. The performance compared to settings (1) and (2) is expectedly worse, however, less than perhaps could be expected.

The decrease in performance for increasing graph size $d$ is presented in Figure~\ref{fig:min10}, where we plot the minimum number $n$ of data needed to reach a mean nSHD of $10\%$. 
The first observation is that larger networks need a larger sample size to reach a lower bound of $10\%$. 
Since larger networks have more opportunities for a wrongly estimated causal influence, this is in line with what we expect. 
Moreover, the standard Gumbel setting (1) reaches the goal much faster than both other settings. Although the weights $c_{ij}$ of setting (2) are in general much smaller than the weights of setting (1), only for a very large graph, substantially more data are needed for the nSHD to fall below 10\%.
This implies that the smaller dependence impacts the estimation for a small graph only moderately, but for a larger graph more data are required.
The mixed distribution setting (3), however, requires for increasing graph size substantially more data.
This is also in line with our expectation as this scenario makes the discrimination between signal and noise more difficult. 

To summarize the results of our simulation study, \QTree\ is sensitive to weaker dependence, but much more sensitive to the tail behavior of innovations/noise distributions.


\begin{figure}[h]
 \centering
 \includegraphics[width=12cm] {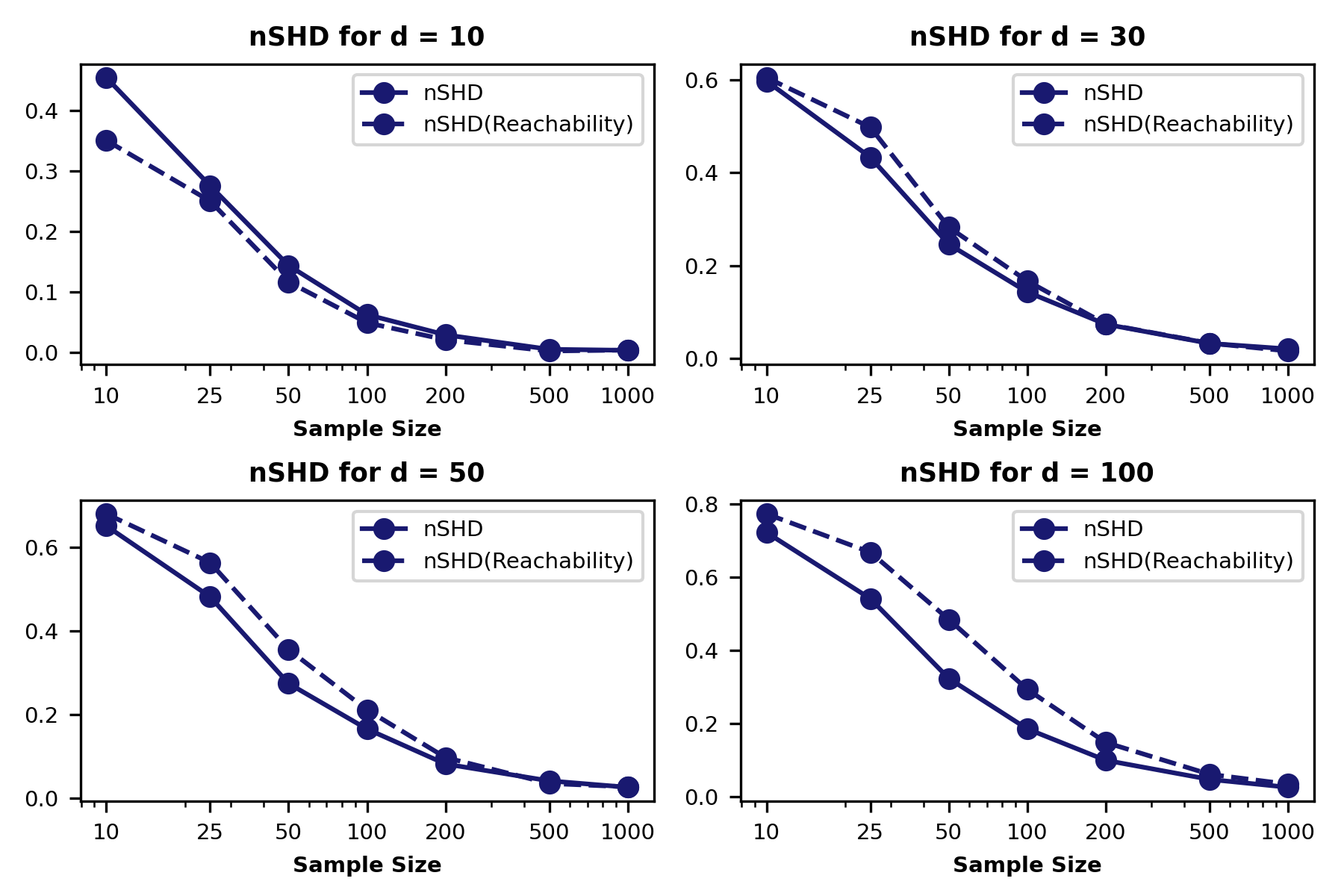}
        \caption[]
        {\small Mean nSHD for the standard Gumbel setting (1) and graph sizes $d=10$ (top left), $d=30$ (top right), $d=50$ (bottom left), $d=100$ (bottom right) and noise-to-signal ratio $k=30\%$. 
        Dots on solid lines display the metrics for the pairs $(\mathcal T, \hat{ \mathcal T})$ and on dashed lines for the pairs $(\mathcal R, \hat{ \mathcal R})$ of their respective reachability graphs.
        For all graph sizes, the nSHD quickly converges to 0 as $n$ increases. Increasing the graph size decreases the metric only moderately.}\label{fig:SHD}
    \end{figure}
    
          \begin{figure}[t]
 \centering
 \includegraphics[width=12cm] {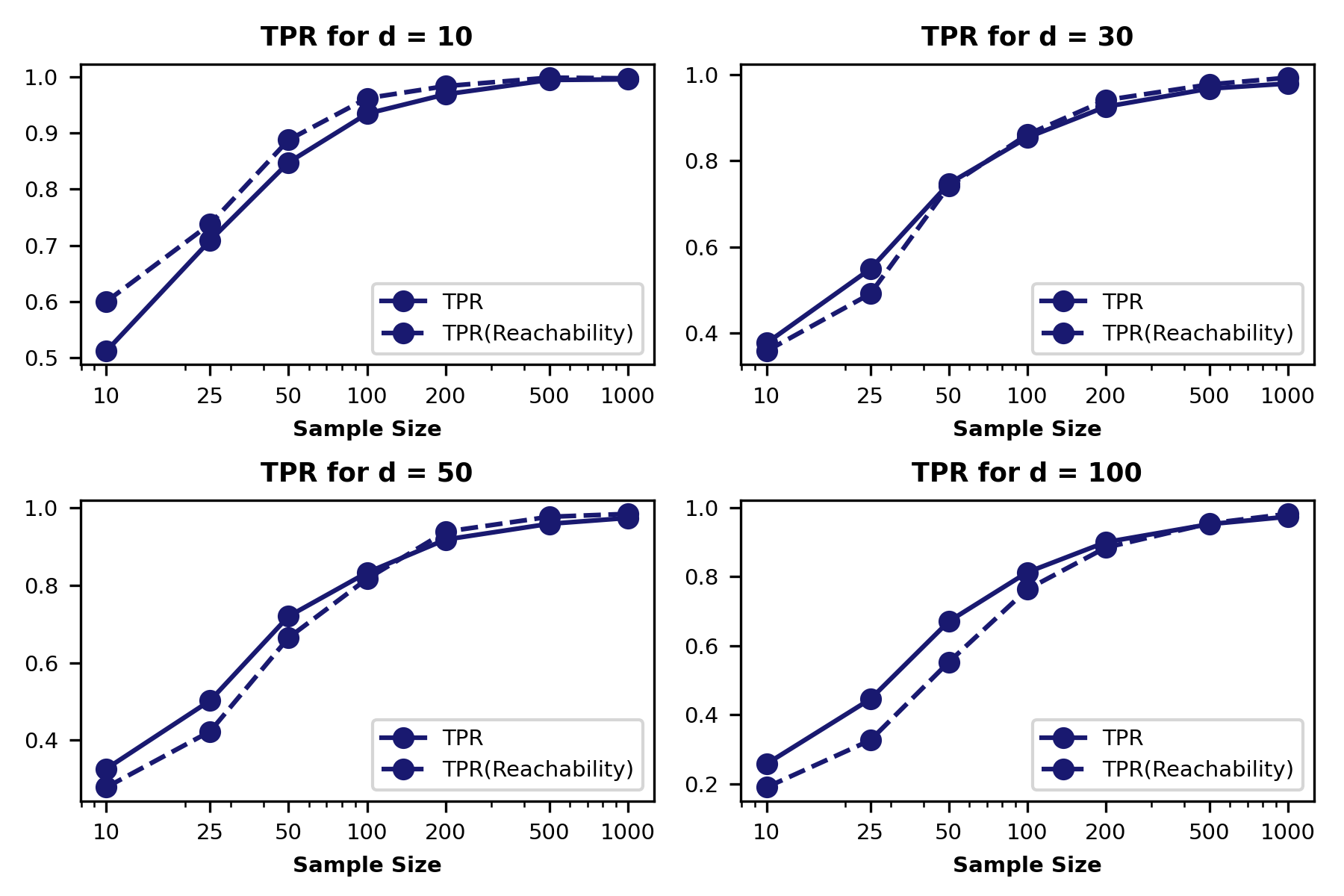}
        \caption[]
        {\small Mean TPR for the standard Gumbel setting (1) and graph sizes $d=10$ (top left), $d=30$ (top right), $d=50$ (bottom left), $d=100$ (bottom right) and noise-to-signal ratio $k=30\%$. 
        Dots on solid lines display the metrics for the pairs $(\mathcal T, \hat{ \mathcal T})$ and on dashed lines for the pairs $(\mathcal R, \hat{ \mathcal R})$ of their respective reachability graphs.
        For all graph sizes, the TPR quickly converges to 1 as $n$ increases. Increasing the graph size decreases the metric only moderately.} \label{fig:TPR}
    \end{figure}

           \begin{figure}
    \centering
 \includegraphics[width=0.6\textwidth] {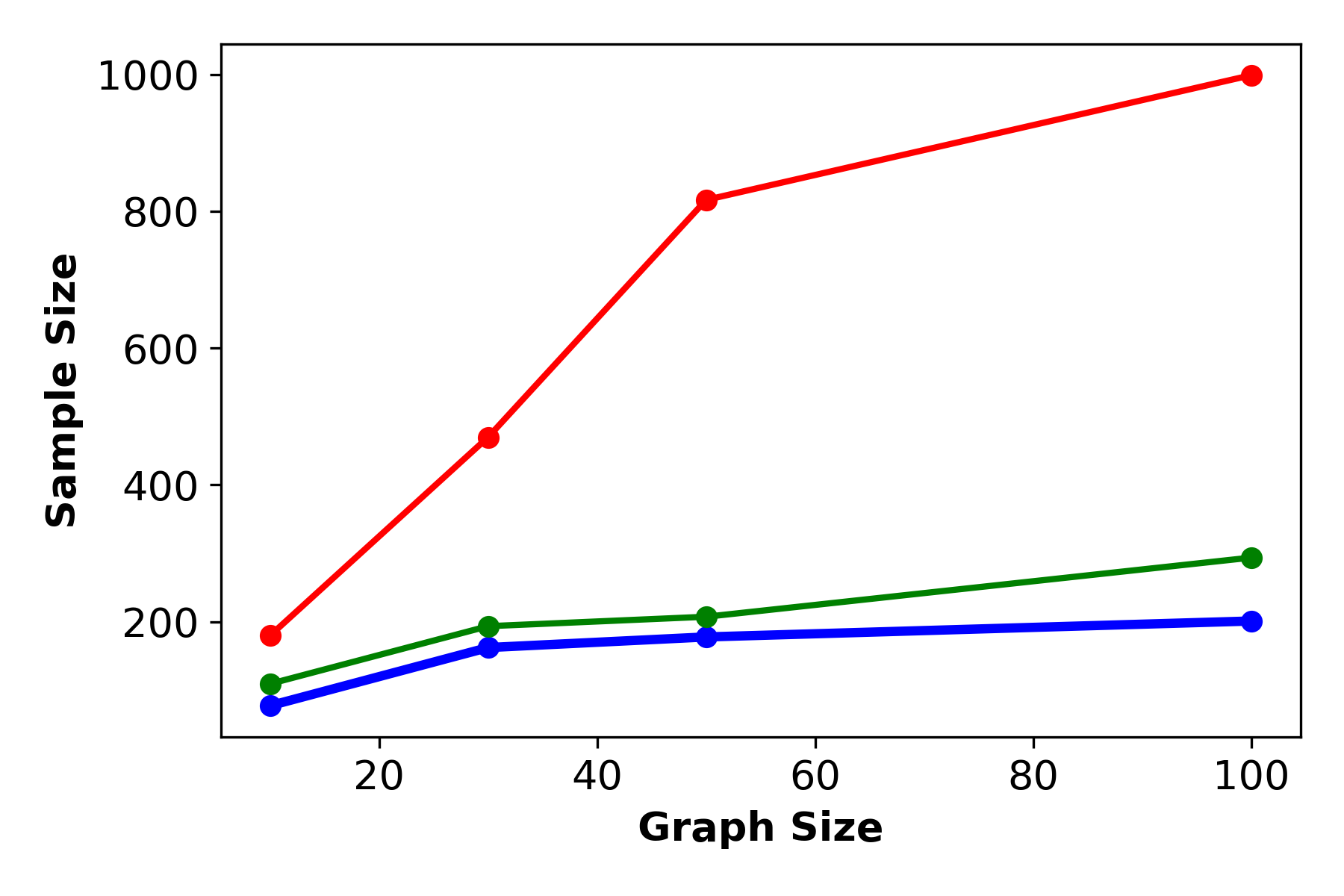}
    \caption{Minimum number of observations needed for the mean nSHD
     to fall below $10\%$: the blue line is for the standard Gumbel setting (1), the green line for the weak dependence setting (2), and the red line for the mixed distribution setting (3). For the standard Gumbel setting (1) and its weak dependence version (2), increasing the graph size $d$, the number of observations needs only a moderate increase to reach the same quality in performance. 
    The mixed distribution setting (2) requires for increasing graph size substantially more data to reach the same quality of performance.}\label{fig:min10} 
\end{figure}

\section{Summary}\label{sec:summary}

In this paper, we proposed \texttt{auto-tuned} \QTree, a new algorithmic solution to the Extremal River Problem---a benchmark problem for causal inference in extremes--- combining the benefits of a new score matrix as input to Chu-Liu/Edmonds' algorithm with a stabilizing subsampling procedure. We also presented three new data sets of the Lower Colorado network for the Extremal River Problem, which are more challenging than the by now classic Upper Danube network data due to a large fraction of missing data and close spatial proximity between nodes. Across all four data sets, \texttt{auto-tuned} \QTree\ performed very well. Our plug-and-play \texttt{Python} implementation in \cite{ngoc_alg} can fit \QTree\ on ten thousand observations in the range of 10 to 30 nodes on a personal laptop within half an hour. 
We proved that for a max-linear Bayesian network with Gumbel-Gaussian distributions for innovations and noise, the tree outputs of \QTree\ are strongly consistent as the number of observations tends to infinity.
Open research directions include (i) generalizations to learning directed acyclic graphs, (ii) better subsampling procedures with theoretical guarantees, and (iii) have the algorithm output a distribution over possible root-directed trees instead of a single best tree.

\section*{Acknowledgements} 
We thank the Lower Colorado River Authority for providing the original data. 
We are also grateful to Sebastian Engelke for providing Figure~\ref{fig:danube-true} of the Upper Danube Basin as well as the declustered data, now available at \cite{RgraphicalExtremes}. We also thank two unknown referees, whose comments, criticism and suggestions improved our paper considerably.
Ngoc Tran 
gratefully acknowledges support by the Hausdorff Center for Mathematics and the University of Bonn, Germany. 
\FloatBarrier

\setlength{\textfloatsep}{5pt}
\setlength{\abovecaptionskip}{2pt}

\appendix

\renewcommand{\thesection}{S\arabic{section}}
\renewcommand{\theequation}{S\arabic{equation}}
\renewcommand{\thefigure}{S\arabic{figure}}
\renewcommand{\thelemma}{S\arabic{lemma}}
\renewcommand{\thetheorem}{S\arabic{theorem}}
\renewcommand{\thecorollary}{S\arabic{corollary}}
\renewcommand{\theproposition}{S\arabic{proposition}}
\setcounter{equation}{0}
\setcounter{figure}{0}
\setcounter{theorem}{0}
\setcounter{proposition}{0}
\setcounter{lemma}{0}
\setcounter{corollary}{0}
\makeatletter

\noindent
{\LARGE\textbf{Supplementary Material}}
\vspace*{0.6cm}

In Section~\ref{sec:Sx} we prove Lemma~\ref{lem:E.T} of the Paper, Section~\ref{sec:s1} gives proofs of the complexity of Algorithms~\ref{alg:qtree} and \ref{alg:qtree.automated}.
In Section~\ref{sec:s3} we prove the Consistency Theorem~\ref{thm:main} of the Paper for the two scores, the lower quantile gap and the quantile-to-mean gap. Section~\ref{sec:s4} provides supplemental figures for Section~\ref{sec:comparison}, and Section~\ref{sec:simSupp} for Section~\ref{sec:sim}.

\section{Proof of Lemma~\ref{lem:E.T} of the Paper}\label{sec:Sx}

We state the lemma again and give a proof.

\begin{lemma}
Let $V$ be a set of nodes and $\mathsf{T} = \{\T^1,\dots,\T^m\}$ a collection of root-directed spanning trees on $V$. 
Define the {\em stability score matrix} $S:=S(\mathsf{T}) \in \R_{\geq 0}^{V \times V}$ by
\begin{equation}\label{eqn:W.T}
s_{ij} := S(\mathsf{T})_{ij} := \#\{\T \in \mathsf{T}: j \to i \in \T\}. 
\end{equation}
Suppose that the maximum root-directed spanning tree $\T_{\max}$ of the graph on $V$ with score matrix $S(\mathsf{T})$ is unique. Then $E(\mathsf{T}) = \T_{\max}$.
\end{lemma}

\begin{proof} 
Identify any root-directed tree $\T$ with the matrix $T=(T_{uv}) \in \{0,1\}^{|V| \times |V|}$. 
Write $\mathbf{1} = (\mathbf{1}_{uv}) $ for the all-one matrix of the same dimension. Let $\T'\in\Psi$ be an arbitrary root-directed spanning tree on $V$. 
Our goal is to show that
$$  \sum_{i=1}^m{\rm nSHD}(\T',\T^i) \geq  \sum_{i=1}^m{\rm nSHD}(\T_{\max},\T^i), $$
which would establish that $\T_{\max} = E(\mathsf{T})$ by eq. \eqref{eqn:E.T.argmin} of the Paper.
 Indeed, 
\begin{align*}
&  \sum_{i=1}^m{\rm nSHD}(\T',\T^i) = \sum_{i=1}^m\sum_{u,v \in V: u \neq v} \mathbf{1}\{T'_{uv} \neq T^i_{uv}\} = \sum_{u,v \in V: u \neq v}\sum_{i=1}^m \mathbf{1}\{T'_{uv} \neq T^i_{uv}\} \\
 =& \sum_{u,v \in V: u \neq v}\left(s_{uv}\mathbf{1}_{u \to v \notin \T'} + (m-s_{uv})\mathbf{1}_{u \to v \in\T'}\right) \\
 =& -2\langle S,T' \rangle +  \langle S,\mathbf{1} \rangle + \langle m\mathbf{1}, T' \rangle \quad \mbox{ where } \langle \cdot,\cdot \rangle \mbox{ denotes the Frobenius inner product} \\
 =& -2\langle S,T' \rangle +  \langle S,\mathbf{1} \rangle + m(d-1) \quad \mbox{ since $\T'$ as root-directed spanning tree has $d-1$ edges } \\
 \geq& -2\langle S,T_{\max} \rangle +  \langle S,\mathbf{1} \rangle + m(d-1) \rangle \quad \mbox{ by definition of } \T_{\max} \\
 =& -2\langle S,T_{\max} \rangle +  \langle S,\mathbf{1} \rangle + \langle m\mathbf{1}, T_{\max} \rangle \rangle \quad \mbox{ since $\T_{\max}$ is a spanning tree on $V$} \\
 =& \sum_{i=1}^m{\rm nSHD}(\T_{\max},\T^i). 
 \end{align*}
\end{proof}

\section{Proof of the Complexity of \texorpdfstring{\QTree}{QTree}}\label{sec:s1}

We work with the solution of the max-linear Bayesian network on a tree $\calt=(V,E)$ defined in eq. \eqref{eqn:main.max.plus} of the Paper; see e.g.
\citep[§3]{baccelli1992synchronization} or \citep[Theorem~2.2]{GK1}.
Let $C^*=(c^\ast_{ij})$ be the matrix of longest paths, also known as the {\em Kleene star of $C=(c_{ij})$}. 
Then 
\begin{equation}\label{eqn:sol}
    X_i = \bigvee_{j:j \rightsquigarrow i \in \mathcal{T}} (c^\ast_{ij}+ Z_j),\quad c^\ast_{ij}, Z_{ij}\in\R,\quad i\in V.
\end{equation}
If there is no path $j\rightsquigarrow i$, then, by definition, $c^\ast_{ij}:=-\infty$.

Lemma \ref{lem:noise.free} concerns the noise-free case, Lemma \ref{lem:qtree.complexity} gives the complexity of Algorithm~1 and Lemma~\ref{lem:2} that of Algorithm~2. 

\begin{lemma}\label{lem:noise.free}
Let $\mathcal{X} = \{x^1,\dots,x^n\}$ be i.i.d observations from the max-linear model given by \eqref{eqn:sol}, 
\emph{not} corrupted with noise. Assume that the $Z_i$ are independent and have continuous distributions.  Define 
\begin{equation}\label{eqn:c.ast.min}
\hat{c}_{ij} = \min_{x \in \mathcal{X}}(x_i-x_j).
\end{equation}
Suppose that for each edge $j \to i$ such that $c_{ij} > -\infty$, there exist at least \emph{two} observations $x \in \mathcal{X}$ where $j$ drives $i$. Then $C^\ast$ can be uniquely recovered from $\hat{C}$ since
\begin{equation}\label{eqn:c.min.twice}
 \hat{c}_{ij} = c^\ast_{ij} \iff \mbox{ min in (\ref{eqn:c.ast.min}) is achieved at least twice}.
 \end{equation}
 In particular, $C^\ast$ can be computed in time $O(|V|^2n)$. 
If all nodes are equally likely to be candidate parents,
 then the matrix $C^\ast$ is recovered exactly for $n = O(|V|(\log(|V|))^2)$.
\end{lemma}

\begin{proof}
Since a root-directed spanning tree has at most one path between any pair of nodes $(j,i)$ we have for an edge $j\to i$ that $c^*_{ij}= c_{ij}$. 
Furthermore, as indicated at the beginning of Section~\ref{sec:intuition} in the Paper,
if for an observation $x$ the value at $j$ drives that at $i$, then $x_i = c^*_{ij} + x_j$.  If $j$ does not drive $i$, then $x_i > c^*_{ij} + x_j$. 
Rearranging motivates the estimator \eqref{eqn:c.ast.min} with properties as stated in \citep[Proposition~1]{GKL}.
Equation \eqref{eqn:c.min.twice} follows from \citep[Lemma~1]{GKL}.

Now we prove the complexity claim. Since there are $O(|V|^2)$ many edges, and for each edge we need $O(n)$ operations to compute the minimum in \eqref{eqn:c.ast.min}, the complexity is $O(|V|^2n)$. The number of observations needed, so that each edge is seen at least twice, is a variant of coupon-collecting \citep{boneh1997coupon,boneh1996general}, where each node must collect two coupons (parents) among its set of parents. Since the nodes are collecting the coupons simultaneously, by the union bound, the number of observations needed is at most $\log(|V|)$ times the number of observations needed for the node with highest degree to collect all of its coupons, which in turn is $O(|V|\log(|V|))$. 
\end{proof}

The following lemmas allow for noisy observations as detailed in (4).

\begin{lemma}\label{lem:qtree.complexity}
The \QTree\ Algorithm~1 runs in time $O(|V|^2n)$.
\end{lemma}
\begin{proof}
For each pair $i,j \in V, i \neq j$, to estimate $w_{ij}$, one needs to compute the $\alpha$-quantile of $\mathcal{X}_j$, the $\smallR$-quantile and the mean of $\mathcal{X}_{ij}(\al)$. Since $\al$ and $\smallR$ are fixed in advance, the empirical quantiles can be computed in time $O(n)$, see \cite{musser1997introspective}. As there are $O(|V|^2)$ pairs, computing $W = (w_{ij})$ takes $O(|V|^2n)$. Chu–Liu/Edmonds' algorithm runs on the complete bidirected graph supported by $W$, and thus takes $O(|V|^2)$, see \cite{gabow1986efficient}. So the complexity of \QTree\ is $O(|V|^2n + |V|^2) = O(|V|^2n)$. 
\end{proof}

The quadratic dependence on $|V|$ and linear dependence on $n$ in Lemma~\ref{lem:qtree.complexity} is optimal, since it takes $O(|V|^2n)$ just to compute pairwise statistics such as the concentration measures in (7) or (8) for every pair of nodes. Similarly, the runtime of Algorithm~2 (\texttt{auto-tuned}  \QTree) also has optimal runtime, which scales linearly with the number of repetitions $m$ and the size of the parameter grid $|\Theta|$. 

\begin{lemma} \label{lem:2}
The \texttt{auto-tuned} \QTree~Algorithm~2 has complexity $O(|V|^2nm|\Theta|)$. 
\end{lemma}
\begin{proof}
For each pair $(\smallR,\al) \in \Theta$, step 3 takes $O(mn)$ and step 4 takes $O(|V|^2nm)$ by Lemma~\ref{lem:qtree.complexity}. Step 6 takes $O(m|V|^2)$, and step 7 takes $O(|V|^2)$ by Chu–Liu/Edmonds' Algorithm. Computing the reachability graph for a root-directed tree on $|V|$ nodes takes $O(|V|)$, so step 8 takes $O(m|V|^2)$, since for each of the $m$ trees in $\mathsf{T}$ we need to compute its structural Hamming distance from the estimated tree $E(\mathsf{T})$. So, for each pair $(\smallR,\al) \in \Theta$, steps 3 to 8 take $O(|V|^2nm)$ time. Thus overall, the algorithmic complexity is $O(|V|^2nm|\Theta|)$.
\end{proof}

\section{Proof of the Consistency Theorem}\label{sec:s3}

In this section, we prove Theorem~\ref{thm:main} of the Paper, which we recall here for ease of reference. 
\begin{quote}
\textbf{Gumbel-Gaussian noise model.} For $i \in V$, the innovations $Z_i$ are i.i.d. Gumbel$(\beta,0)$ (location 0 and scale $\beta$), the independent noise variables $\eps_i$ are i.i.d with symmetric, light-tailed density $f_\eps$ satisfying
\beam\label{eqn:light.tail2}
f_{\eps}(x) \sim  e^{-K x^p}\mbox{ as } x \to \infty,
\eeam
for some $p > 1$ and $K > 0$ and the derivative of $f_{\eps}$ exists in the tail region.
Throughout, for two functions $a,b$, positive in their right tails, we write $a(x)\sim b(x)$ as $x\to\infty$ for $\lim_{x\to\infty} a(x)/b(x) = c,$ where $c>0$ is some arbitrary constant.
\end{quote}


\begin{theorem}[Theorem~\ref{thm:main} of the Paper]
Assume the Gumbel-Gaussian noise model. \\
(a) There exists an $r^\ast > 0$ such that for any pair $0 < \smallR < \bigR <  r^\ast$, the \QTree\ Algorithm~1 with score matrix $W=(w_{ij})$ defined as the lower quantile gap \begin{equation}\label{eqn:wij.tail2}
w_{ij}(\smallR,\bigR) :=
\frac{1}{n_{ij}}\left(Q_{\mathcal{X}_{ij}(\alpha)}(\bigR)-Q_{\mathcal{X}_{ij}(\alpha)}(\smallR) \right)^2
 \end{equation}
  returns a strongly consistent estimator for the tree $\T$ as the sample size $n\to\infty$.\\
(b) There exists an $r^\ast > 0$ such that for any $0 < \smallR < r^\ast$, the \QTree\ Algorithm~1 with score matrix $W=(w_{ij})$ defined as the quantile-to-mean gap 
\begin{equation}\label{eqn:wij.tail.mean2}
w_{ij}(\smallR) := 
\frac{1}{n_{ij}}\left(\E(\mathcal{X}_{ij}(\al))- Q_{\mathcal{X}_{ij}(\al)}(\smallR) \right)^2
 \end{equation}
returns a strongly consistent estimator for the tree $\T$ as the sample size $n\to\infty$.
\end{theorem}

The proof of this theorem comes in a series of steps. 
Moreover, for simplicity, we omit the normalization by $n_{ij}$ and the squaring in \eqref{eqn:wij.tail2} and \eqref{eqn:wij.tail.mean2} as this leaves the proof unchanged.
Also we set in the proof $\al=0$. 

As a preliminary result, Lemma \ref{lem:w.good} identifies a set of `good' deterministic input matrices $W=(w_{ij})$, where if we apply the \QTree\ algorithm to such an input, then it returns the true tree $\T$ \emph{exactly}. The proof then reduces to the problem of proving that as $n \to \infty$, the matrices $W_n$ derived from data converge a.s. to a `good' $W$. Intuitively, $W$ is `good' if for each node $j$, the weight $w_{ij}$ is smallest when $i$ is the child of $j$. For the root we have a special explicit condition. For each fixed $j$, we split the set of node pairs $\{(j,i): j,i \in V, i \neq j\}$ into three scenarios: 
\begin{itemize}
  \item $j \rightsquigarrow i$, that is, $i$ is a descendant $j$ in the true tree,
  \item $i \rightsquigarrow j$, that is, $i$ is an ancestor of $j$ in the true tree, and
  \item $i \not\sim j$, that is, $i$ is neither of the above. 
\end{itemize}

We first consider the case where $W=(w_{ij})$ is the matrix of lower quantile gaps \eqref{eqn:wij.tail2} of the \emph{true} distribution. 
Note that this $W$ is no longer random. The goal is to show that if the true quantiles are known, then one can choose the parameters $(\smallR,\bigR)$ such that $W$ is good.

Next, Proposition~\ref{prop:compare.X} gives an explicit representation for $w_{ij}$ in each of the three scenarios above as the lower quantile gap  of a certain family of distributions $(F^b: b \in \mathbb{R})$, parametrized by a \emph{single parameter $b$}, one value for each edge $j \to i$. Then, we use a calculus of variation argument to detail how $w_{ij}$ changes as $b$ varies. This allows us to show (cf. Corollary \ref{cor:variation} and Lemma \ref{lem:want}) that among the three scenarios above, there exist some choices of quantile levels $(\smallR,\bigR)$ such that for \emph{any} fixed $j$, $w_{ij}$ is smallest when $i$ is the child of $j$ in the true tree. A separate argument is made for the root. Thus, this proves that if the true quantiles are known, then the resulting $W$ is good. 

Finally, we invoke the fact that the empirical quantiles converge a.s. to the true quantiles as $n \to \infty$, and thus the empirical $w_{ij}$ are a.s. close to the true ones. 
A union bound over the $d$ nodes of the graph thus says that, the empirical $W_n$ is a.s. `good' as $n \to \infty$, and thus proves the Consistency Theorem for the lower quantile gap. 

The proof for the quantile-to-mean gap is similar, with Proposition \ref{prop:compare.X.2} playing the role of Proposition \ref{prop:compare.X}. 

\begin{lemma}[A criterion for `good' inputs $W$]\label{lem:w.good}
Let $W = (w_{ij})$ be a score matrix such that each true edge $j \to i \in \mathcal{T}$ satisfies
\begin{equation}\label{eqn:w.ij.correct}
w_{ij} < w_{i'j} \mbox{ for all } i' \in V, i' \neq i,j,
\end{equation}
and in addition, the true root $i^\ast$ satisfies
\begin{equation}\label{eqn:min.root}
\min_{i'} w_{i'i^\ast} > \max_{i,j: j \to i} w_{ij}. 
\end{equation}
Then the \QTree\ Algorithm~1 applied to input $W$ returns the true tree $\T$. 
\end{lemma}

\begin{proof}
\QTree\ applies Chu–Liu/Edmonds' algorithm to find a minimum directed spanning tree from the complete graph with score matrix $W$, and returns that tree. We shall prove that under the conditions \eqref{eqn:w.ij.correct} and \eqref{eqn:min.root} on $W$, Chu–Liu/Edmonds' algorithm would converge after one iteration and returns the true tree $\T$. Indeed, let $\mathcal{G}$ denote the graph that consists of the smallest outgoing edge at each node. By \eqref{eqn:w.ij.correct}, $\mathcal{G} = \T \cup i^\ast \to i'$ for some node $i' \in V$. By Chu–Liu/Edmonds' algorithm, the minimum spanning tree $\T_w$ is a subset of $\mathcal{G}$. In particular, $\T_w$ is a minimum spanning tree of $\mathcal{G}$.  By \eqref{eqn:min.root}, edge $i^\ast \to i'$ is the maximal edge. Since it belongs to the unique cycle in $\mathcal{G}$, deleting this edge would yield the minimum directed spanning tree of $\calg$. Therefore $\T_w = \T$. 
\end{proof}

\subsection{Proof of Theorem~\ref{thm:main} for the lower quantile gap}

\subsubsection{For known quantiles,  \texorpdfstring{$W$}{W} is `good' for appropriate choices of  \texorpdfstring{$(\smallR,\bigR)$}{(r,r)}}

In this subsection we work with the lower quantile gap matrix $W = (w_{ij})$ derived from the \emph{true} quantiles of the distributions of $X_i-X_j$ under the Gumbel-Gaussian noise model, for some quantile levels $(\smallR,\bigR)$. The goal is to show that there exist some appropriate choices of $(\smallR,\bigR)$ such that the resulting $W$ is `good', that is, it satisfies Lemma~\ref{lem:w.good}. 

The first main result is Proposition \ref{prop:compare.X}, which gives an explicit representation for $w_{ij}$ in the three scenarios. We start with the necessary definitions to state it. 

Recall the definition of $C^*$ from \eqref{eqn:sol}.
Since the true graph is a tree, if $j \rightsquigarrow i$, there is a unique directed path from $j$ to $i$. 
Let $\bar{c}_{ij}$ denote the sum of all the edges along this unique path. 
Path uniqueness implies that $\bar{c}_{ij}=c^\ast_{ij}$ and $C^\ast$ is transitive, i.e. $c^\ast_{ij} = c^\ast_{ik} + c^\ast_{kj}$ if $j \rightsquigarrow k \rightsquigarrow i$. Thus, by the Helmholtz decomposition on graphs \citep[eq. (2.6)]{lek-heng-lim-Hodge-Laplacian-on-Graphs}, $c^\ast_{ij}$ is an edge flow. That is, there exists a unique $t^\ast \in \R^V$ with $t^\ast_1 =0$ such that for all $j \to i \in \calg$,
\begin{equation}\label{eqn:C.true.t}
c^\ast_{ij} = t^\ast_i - t^\ast_j.
\end{equation}
For each $i \in V$, define the constant
\begin{equation}\label{eqn:theta.i}
\theta_i:= \sum_{k\rightsquigarrow i} \exp(-t^\ast_k/\beta).
\end{equation}
For $b\in \R \cup \{-\infty\}$, define the random variable
\begin{equation}\label{eqn:xic}
 \xi_{b} := (\eps_i - \eps_j) + ((Z_i - Z_j) \vee b) 
\end{equation}
with the convention that $\xi_{-\infty} := (\eps_i - \eps_j) + (Z_i - Z_j)$. 
Let $F^b$ denote the distribution function of $\xi_b$ and $q_r(F^b)$ the $r$-quantile of $F^b$ for $r\in (0,1)$. These quantities are deterministic and do not depend on $i,j$ since by assumption, $\eps_i,\eps_j$ are i.i.d and $Z_i,Z_j$ are i.i.d.

\begin{proposition}\label{prop:compare.X}
Assume the Gumbel-Gaussian noise model. Fix $0 < \smallR < \bigR < 1$. Let $w_{ij}=w_{ij}(\smallR,\bigR)$ be the lower quantile gap \eqref{eqn:wij.tail2}. Fix $j \in V$. For $i \in V, i \neq j$, we have three cases.
\begin{enumerate}
	\item[{\rm (1)}] 
	If $j \rightsquigarrow i$, then $w_{ij} = q_{\bigR}(F^b) - q_{\smallR}(F^b)$ for $b=\beta (\log\theta_j-  \log(\theta_i-\theta_j))$.
	\item[{\rm (2)}] 
	If $j \not\sim i$, then $w_{ij} = q_{\bigR}(F^b) - q_{\smallR}(F^b)$ for $b=-\infty$.
	\item[{\rm (3)}] 
If $i \rightsquigarrow j$, then $w_{ij} = q_{1-\smallR}(F^b) - q_{1-\bigR}(F^b)$ for $b= \beta (\log\theta_i -  \log(\theta_j-\theta_i))$.
\end{enumerate}
\end{proposition}

\begin{proof}
We first consider the noise-free case. Observe that by \eqref{eqn:xic} $\xi_b$ simplifies to $(Z_i-Z_j) \vee b$. Therefore, it is sufficient to prove that $w_{ij}$ equals the lower quantile gap of $(Z_i-Z_j) \vee b$. For $i \in V$, let $\bar{X}_i := X_i - t^\ast_i$. Then $\bar{X}_i-\bar{X}_j$ is a constant translation of $X_i-X_j$, so the lower quantile gap of the two corresponding distributions are the same. In other words, it is sufficient to prove the Proposition for $\bar{X}$ instead of $X$. Let $\bar{Z}_i := Z_i - t^\ast_i$. Then
\begin{align}
\bar{X}_i = X_i - t^\ast_i &= \bigvee_{j: j \rightsquigarrow i} (c^\ast_{ij} + Z_j) - t^\ast_i & \mbox{ by \eqref{eqn:sol}} \nonumber \\
&= \bigvee_{j: j \rightsquigarrow i} (t^\ast_i - t^\ast_j + Z_j) - t^\ast_i & \mbox{ by \eqref{eqn:C.true.t}} \nonumber \\
&= \bigvee_{j: j\rightsquigarrow i} \bar{Z}_j. \label{eqn:bar.X}
\end{align}
For each ordered pair $(i,j)$, define 
\begin{align}\label{eqn:sisj}
S_i=\bar{Z}_i \vee \bigvee\limits_{{i' \neq i, i' \rightsquigarrow i, i' \not \rightsquigarrow j}} \bar{Z}_{i'}
\hspace*{3cm}
S_j= \bar{Z}_j \vee \bigvee\limits_{{j' \neq j, j' \rightsquigarrow j}} \bar{Z}_{i'}
\end{align}

In Figure~\ref{fig:illustrateS} we illustrate the two index sets of the random variables $S_i$ and $S_j$.
\begin{figure}[h]
  \centering
\begin{tikzpicture}[->,>=stealth',shorten >=1pt,auto,node distance=1.3cm,
                thick,main node/.style={circle,draw,font=\bfseries,minimum size=0.7cm,inner sep=0cm},,path/.style={ black,decorate,
  decoration={zigzag,amplitude=1pt,segment length=3.4mm,pre=lineto,pre length=0pt}},missing edge/.style=loosely dotted]
                
    \node[main node] (1) {$i'_1$};
    \node[main node] (2) [right of=1] {$i'_2$};
    \node[main node] (3) [right of=2] {$i'_3$};
    \node[main node] (4) [right=1.7cm of 3] {$j'_1$};
    \node[main node] (5) [right of=4] {$j'_2$};
    \node[main node] (6) [right of=5] {$j'_3$};
    \node[main node] (7) [below=1.5cm of 2] {$i$};
    \node[main node] (8) [below=1.5cm of 5] {$j$};

  \path
    (1) edge [path] node {} (7)
    (2) edge [path] node {} (7)
    (3) edge [path] node {} (7)
    (4) edge [path] node {} (8)
    (5) edge [path] node {} (8)
    (6) edge [path] node {} (8)
    (8) edge [path] node {} (7)
    ;
\end{tikzpicture}
\caption{Illustration of the index sets of $S_i$ and $S_j$ for an ordered pair $(i,j)$. The index set for $S_i$ includes besides $i$ also $i'_1,i'_2, i'_3$ and all nodes on the paths $j \rightsquigarrow i$ (excluding $j$), $i'_k \rightsquigarrow i$ for $k=1,2,3$, while the index set for $S_j$ includes $j,j'_1,j'_2, j'_3$ and all nodes on the paths $j'_k \rightsquigarrow j$ for $k=1,2,3$.}\label{fig:illustrateS}
\end{figure}

By definition, $S_i$ and $S_j$ are independent. Since the $Z_i$ are Gumbel$(\beta,0)$ distributed, abbreviated Gumbel$(\beta)$,
the $\bar{Z}_i$ are translated independent Gumbel($\beta$) by definition, standard properties of the Gumbel($\beta$) distribution yield that $S_i$ and $S_j$ are also translated independent Gumbel($\beta$). The exact constants of translation depend on the relation between $i$ and $j$, as this dictates the definition of $S_i$ and $S_j$. Now we consider the three cases. In the first case, $j \rightsquigarrow i$. Then, \eqref{eqn:bar.X} implies
$\bar{X}_i = S_i \vee S_j$ and $\bar{X}_j = S_j$.

A short computation yields $S_i \stackrel{d}{=} Z_i + \beta \log(\theta_i-\theta_j)$, $S_j \stackrel{d}{=} Z_j + \beta\log\theta_j$. Therefore, denoting $\eqd$ equality in distribution,
 \begin{align*}
\bar{X}_i - \bar{X_j} &= (S_i \vee S_j) - S_j = (S_i - S_j)\vee 0 \\
&\stackrel{d}{=} (Z_i - Z_j - \beta(\log\theta_j-\log(\theta_i-\theta_j))) \vee 0\\
& = \left((Z_i - Z_j)\vee \beta(\log\theta_j-\log(\theta_i-\theta_j))\right) - \beta(\log\theta_j-\log(\theta_i-\theta_j)) \\
&= \left((Z_i - Z_j)\vee b\right) - \beta(\log\theta_j-\log(\theta_i-\theta_j)),
\end{align*}
where $b = \beta( \log\theta_j-  \log(\theta_i-\theta_j))$. Since $\beta (\log\theta_j-\log(\theta_i-\theta_j))$ is a translation constant, the quantile gap of $\bar{X}_i - \bar{X_j}$ is equal to the quantile gap of $(Z_i - Z_j)\vee b$. This concludes the case $j \rightsquigarrow i$. Computations for the third case, $i\rightsquigarrow j$, is similar, with the role of $i$ and $j$ reversed, $\smallR$ is replaced by $1-\bigR$, and $\bigR$ is replaced by $1-\smallR$. For the second case, $i \not\sim j$, then
$\bar{X}_i {=} S_i, \bar{X}_j {=} S_j$, where $S_j \stackrel{d}{=}\beta \log\theta_j + Z_j$ and $S_i \stackrel{d}{=} \beta \log\theta_i + Z_i$. Then
$$
\bar{X}_i - \bar{X}_j = S_i-S_j \stackrel{d}{=} Z_i-Z_j + \beta (\log\theta_i-\log\theta_j).
$$
Since $\beta (\log\theta_i-\log\theta_j)$ is a translation constant, the quantile gap of $\bar{X}_i - \bar{X_j}$ is equal to the quantile gap of $Z_i-Z_j$, as claimed.
\end{proof}

\subsubsection{How the lower quantile gap  \texorpdfstring{$w_{ij}$}{wij}  varies with  \texorpdfstring{$b$}{b}}\label{AB}

Now, we aim to show through a variational argument that under the Gumbel-Gaussian assumption, among the three scenarios of Proposition~\ref{prop:compare.X}, $w_{ij}$ is smallest when it falls in a subset of case (1), namely, $j \to i$. 
We first give an overview. By Proposition \ref{prop:compare.X}, the lower quantile gaps $w_{ij}$ in cases (1) and (2) are all of the form $q(b,\bigR) - q(b,\smallR)$ for some constant $b = b(i,j)$. In particular, for fixed $j$, $b(i,j)$ is largest when $j \to i$. Lemma \ref{lem:variation} says that one can choose the quantile levels $(\smallR,\bigR)$ such that $q(b,\bigR) - q(b,\smallR)$ is monotone \emph{increasing} as a function of $b$ on a large interval. Corollary \ref{cor:variation} then shows that a good choice can be made so that for each fixed $j$, the quantile gap is smallest for the edge from $j$ to its child $ch(j)$. Case (3) of Proposition \ref{prop:compare.X}, where $i$ is an ancestor of $j$, is handled by Lemma \ref{lem:want}. 
The Gumbel-Gaussian assumption comes in through Lemma \ref{lem:tail}, which is a technical result that gives an explicit form for the density of the noise differences $\eta := \eps_i-\eps_j$. Intuitively, it shows that under the Gumbel-Gaussian noise model, the tail of $\eta$ is lighter than the tail of the signal differences $Z_i-Z_j$. This is a key observation exploited in the proofs. 

\begin{lemma}\label{lem:tail}
Under the Gumbel-Gaussian noise model, for any pair of nodes $i,j \in V, i \neq j$, $\xi:=Z_i-Z_j$ has density
\begin{align}\label{eqn:xi}
    f_{\xi}(x)=\frac{e^{x/\beta}}{\beta(1+e^{x/\beta})^2} \sim  \frac{1}{\beta} e^{-{x/\beta}} \mbox{ as } x \to \infty,
\end{align}
and $\eta := \eps_i-\eps_j$ has density 
\begin{equation}\label{light.tail}
f_\eta(x)  \sim  x^{1-p/2} \, e^{-K x^p} \mbox{ as } x \to \infty.
\end{equation}
\end{lemma}

\begin{proof}
Computing the convolution integral yields 
\begin{align}\label{eqn:zdiff}
  \P(Z_i-Z_j >  x) =\frac1{1+e^{x/\beta}}, \quad x\in\R,  
\end{align}
and taking the derivative gives the first statement. For the second statement, the density $f_\eps$ is a density with Gaussian tail in the sense of \cite{BKR}:
$$f(x)\sim\gamma(x) e^{-\psi(x)}\mbox{ as } x \to \infty,$$
for constant $\gamma$ and $\psi(x)= K x^p$. The asymptotic form of $f_\eta$ follows by Laplace's integration principle as shown in \cite[page 2]{BKR}. 
\end{proof}

Since $f_\eps$ is differentiable in the tail, $f_\eta$ is also differentiable in the tail, and differentiation of \eqref{light.tail} yields the following formula for the derivative:
\begin{equation}\label{eqn:deriv}
f_\eta'(x)  \sim f_\eta(x)\Big(-K p x^{p-1} + (1-p/2) x^{-1}\Big)
\end{equation}

For functions with two arguments, let $\partial_1$ denotes the derivative in the first argument, $\partial_2$ denotes the derivative in the second argument, $\partial^2_{12} := \partial_1\partial_2$ denote the mixed second derivatives and so forth. Define the functions $H: \R \times \R \to [0,1]$, $q: \R \times [0,1] \to \R$ by
$$ H(b,a) = P(\xi_b \leq a), \quad q(b,r) = \mbox{$r$-quantile of } \xi_b. $$

\begin{lemma}\label{lem:variation}
Under the Gumbel-Gaussian noise model, for each finite constant $B$, there exists some $r^\ast = r^\ast(B) \in (0,1)$ such that
$$ \partial^2_{12} q(b,r) < 0 \quad \mbox{ for all } r \in (0,r^\ast), b \leq B. $$
Equivalently, for any pair $(\smallR,\bigR)$ such that $0 < \smallR < \bigR < r^\ast$ and any pair $(b',b)$ such that $b' < b \leq B$,
\begin{equation}\label{eqn:q.from.H}
 q(b,\bigR) - q(b,\smallR) < q(b',\bigR)-q(b',\smallR).
\end{equation}
\end{lemma}
\begin{proof}
By definition,
\begin{equation}\label{eqn:Ht}
 H(b,q(b,r)) = r. 
 \end{equation}
We take derivatives of both sides, first with respect to $r$, then to $b$. Note that functions and derivatives of $H$ are always evaluated at $(b,q(b,r)) $ while those of $q$ are evaluated at $(b,r)$, so we suppress them in the notations. Differentiate both sides sof \eqref{eqn:Ht} with respect to $r$ gives
\begin{equation}\label{eqn:Ht.1}
\partial_2 H \cdot \partial_2 q = 1.
\end{equation}
Now, differentiating both sides of \eqref{eqn:Ht} with respect to $b$, we get
$$ \frac{\partial}{\partial b}H_1(b,q(b,r)) = \partial_1H + \partial_2 H\cdot \partial_1 q = 0, $$
therefore,
\begin{equation}\label{eqn:Ht.2}
\partial_1 q = \frac{-\partial_1 H}{\partial_2 H}.
\end{equation}
Differentiate \eqref{eqn:Ht.1} with respect to $b$ using implicit differentiation and chain rules, we get
\begin{align}
0 = \frac{\partial}{\partial b}(\partial_2 H \cdot \partial_2 q ) 
&= \frac{\partial}{\partial b}(\partial_2 H(b,q(b,r)) \cdot \partial_2q + \partial_2 H \cdot \partial^2_{12}q  \nonumber \\
&= (\partial^2_{12}H + \partial^2_{22}H \cdot \partial_1 q) \cdot \partial_2q + \partial_2H \cdot \partial^2_{12}q \nonumber \\
&= \frac{\partial^2_{12}H - \partial^2_{22}H \cdot \frac{\partial_1H}{\partial_2H}}{\partial_2H} + \partial_2H \cdot \partial^2_{12}q \quad \mbox{ by \eqref{eqn:Ht.1} and \eqref{eqn:Ht.2}.} 
\end{align}
Rearranging the last equation gives
\begin{equation}\label{eqn:h.ct}
\partial^2_{12}q = \frac{\partial^2_{22}H\cdot \partial_1H-\partial^2_{12}H\cdot \partial_2H}{(\partial_2H)^3}.
\end{equation}
 For fixed $b$, by definition of $H$, $\partial_2H$ is the density of $\xi_b$, so $\partial_2H > 0$. So $ \partial^2_{12}q(b,r) < 0$ if and only if 
 \begin{equation}\label{eqn:Htt}
(\partial^2_{22}H\partial_1H-\partial^2_{12}H\partial_2H)(b,q(b,r)) < 0.
\end{equation}
Now we compute each of the terms $\partial_2H,\partial_1H,\partial^2_{12}H$ and $\partial^2_{22}H$ in the left hand side of \eqref{eqn:Htt} explicitly in terms of the density $f_\eta$ of the noise difference $\eta = \eps_i-\eps_j$. Note that $\xi_b = \eta + (\xi \vee b)$ where $\xi := Z_i-Z_j$. 
Then we have for $\eps>0$ (see Figure~\ref{fig:illus1})
\begin{equation*}
H(b+\eps,a) - H(b,a) 
 =   \P(\eta+\xi\vee (b+\eps) \le a)-\P(\eta+\xi\vee b\le a)
\end{equation*}
\begin{equation*}
 =  \left\{
\begin{array}{lll}
     &  0 & \quad\mbox{ if } \xi> b+\eps,\\
     &   \P(\eta+ b+\eps \le a) -\P(\eta+b \le a) & \quad\mbox{ if } \xi\le b,\\
     &  \P(\eta+ b+\eps \le a) -\P(\eta+\xi \le a) & \quad\mbox{ if } b\le \xi\le b+\eps.
\end{array}
\right.
\end{equation*}
Since $\eta$ and $\xi$ are independent, this implies
\begin{equation}\label{eqn:hc.epsilon}
H(b+\eps,a) - H(b,a) = - \P(\xi \leq b) \P(a-b-\eps \leq \eta \leq a-b) + O(\eps^2). 
\end{equation}
Hence,
\begin{equation}\label{eqn:h.c}
\partial_1H(b,a) = \lim_{\eps \downarrow 0} \frac{H(b+\eps,a) - H(b,a)}{\eps} = -\P(\xi \leq b)f_\eta(a-b). 
\end{equation}
A similar calculation gives (see Figure~\ref{fig:illus2})
\begin{align}\label{eqn:H.a}
\partial_2H(b,a) &= \lim_{\eps \downarrow 0} \frac{H(b,a+\eps) - H(b,a)}{\eps} 
= \P(\xi \leq b)f_\eta(a-b) + \int_b^\infty f_\eta(a-x)f_\xi(x)\, dx \nonumber\\
& =: (-\partial_1H + A)(b,a). 
\end{align}

\begin{figure}[h!]
\begin{adjustwidth}{-25pt}{-25pt}
\centering
\begin{minipage}{.48\textwidth}
  \centering
\includegraphics{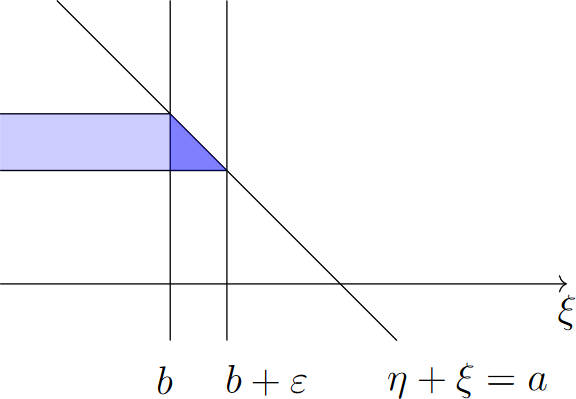}
  \caption{$H(b+\varepsilon,a)-H(b,a)$ is the probability that $(\xi,\eta)$ lies in the shaded regions (light and dark shaded)}\label{fig:illus1}
\end{minipage}
\hspace{0.3cm}
\begin{minipage}{.48\textwidth}
  \centering
\includegraphics{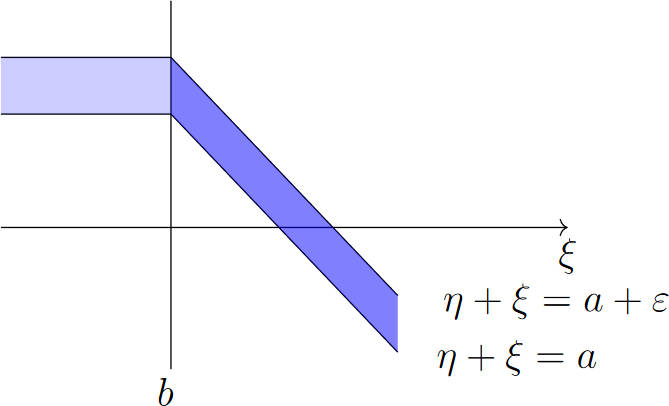}
\caption{$H(b,a+\varepsilon)-H(b,a)$ is the probability that $(\xi,\eta)$ lies in the shaded region (light and dark shaded)}\label{fig:illus2}
\end{minipage}

\end{adjustwidth}
\end{figure}

Thus, 
\begin{align*}
\partial^2_{12}H(b,a) &= -\P(\xi \leq b)f'_\eta(a-b), \\
\partial^2_{22}H(b,a) &= \P(\xi \leq b)f'_\eta(a-b) + \int_b^\infty f'_\eta(a-x)f_\xi(x)\, dx \\
&= -\partial^2_{12}H(b,a) + \int_b^\infty f'_\eta(a-x)f_\xi(x)\, dx. \\
&= (-\partial^2_{12}H + A_a)(b,a).
\end{align*}
Now we have
\begin{align*}
& (\partial^2_{22}H\partial_1H-\partial^2_{12}H\partial_2H)(b,a) 
= (A_a\partial_1H -A\partial^2_{12}H)(b,a) \\
=& \P(\xi \leq b)  \left( - f_\eta(a-b) \int_b^\infty f'_\eta(a-x)f_\xi(x)\, dx  + f'_\eta(a-b)\int_b^\infty f_\eta(a-x)f_\xi(x)\, dx \right).
\end{align*}
Thus, \eqref{eqn:Htt} holds if and only if
\begin{equation}\label{eqn:Htt.f}
-  f_\eta(q(b,r)-b) \int_b^\infty f'_\eta(q(b,r)-x)f_\xi(x)\, dx  + f'_\eta(q(b,r)-b)\int_b^\infty f_\eta(q(b,r)-x)f_\xi(x)\, dx <  0. 
\end{equation}
Now we need to show that for each constant $B$, there exists an $r^\ast(B) > 0$ such that for all $r < r^\ast$ and $b \leq B$, \eqref{eqn:Htt.f} holds. Fix the constant $B$. Since the noise $\eps$ has no upper bound, $\eta$ and hence $\xi_b$ have unbounded support below. For each fixed $b$, $q(b,r) \to -\infty$ as $r \downarrow 0$. \JB{Therefore, there exists some sufficiently small $r^\ast > 0$ such that for all $r < r^\ast$, $q(b,r)-b$ is a large negative number. Fix such an $r^\ast$.} Therefore, for all $x > b$, $q(b,r)-x$ is a large negative number. This allows us to use Lemma \ref{lem:tail} to make the left hand side of \eqref{eqn:Htt.f} explicit. In particular, by \eqref{light.tail}, as $t \to \infty$,
\begin{align}\label{eqn:decrease}
f_\eta(t)= K_1 t^{1-p/2} e^{-K t^p}(1+o(1)),\quad f_\eta'(t)=f_\eta(t)\Big(-K p t^{p-1} + (1-{p}/{2}) t^{-1}\Big)(1+o(1)).
\end{align}
Setting $t = |x-q(b,r)|$ and use the fact that $f_\eta$ is symmetric, $f_\eta(q(b,r)-x) = f_\eta(|x-q(b,r)|)$, we have
\begin{align*}
& - f_\eta(q(b,r)-b) \int_b^\infty f'_\eta(q(b,r)-x)f_\xi(x)\, dx  + f'_\eta(q(b,r)-b)\int_b^\infty f_\eta(q(b,r)-x)f_\xi(x) dx
\\
=& f_\eta(q(b,r)-b) \int_b^\infty [-K p(x-q(b,r))^{p-1} + (1-\frac{p}2)(x-q(b,r))^{-1}] f_\eta(q(b,r)-x)f_\xi(x) dx\\  
& + [K p(b-q(b,r))^{p-1} - (1-\frac{p}2)(b-q(b,r))^{-1}] f_\eta(q(b,r)-b)\int_b^\infty f_\eta(q(b,r)-x)f_\xi(x) dx \\
=& f_\eta(q(b,r)-b) \int_b^\infty  A(x) f_\eta(q(b,r)-x)f_\xi(x)dx,
\end{align*}
where
$$ A(x) = \Big(K p [(b-q(b,r))^{p-1}-(x-q(b,r))^{p-1}] - (1-p/2)[(b-q(b,r))^{-1}-(x-q(b,r))^{-1}] \Big). $$
\JB{If $p \in (1,2]$, since $x > b$, the first term $(b-q(b,r))^{p-1}-(x-q(b,r))^{p-1}$ and the second term $-(1-p/2)[(b-q(b,r))^{-1}-(x-q(b,r))^{-1}]$ are both negative. If $p>2$, since $x > b \gg q(b,r)$ the first term $(b-q(b,r))^{p-1}-(x-q(b,r))^{p-1}$ is a large negative number. Since $q(b,r)$ is a large negative number, $b-q(b,r)$ is a large positive number, so $(b-q(b,r))^{-1}, (x-q(b,r))^{p-1} < 1$. Thus $|(1-p/2)[(b-q(b,r))^{-1}-(x-q(b,r))^{-1}]| \leq |1-p/2|$. For this reason, $A(x) < 0$ for all $p>1$, $x > b$, while $f_\eta, f_\xi >0$ everywhere as they are densities.} Thus the integral is negative, that is, \eqref{eqn:Htt.f} holds for all $r \in (0,r^\ast)$ and $b \leq B$, as needed. 
\end{proof}

Below we denote $ch(j)$ the child of node $j$.

 \begin{corollary}\label{cor:variation} 
 Under the Gumbel-Gaussian noise model, there exists an $r_1^\ast > 0$ such that: for all $0 < \smallR < \bigR < r_1^\ast$, for all $j \in V$ and for all $i' \in V, i' \neq j, ch(j)$  and either $j \rightsquigarrow i'$ or $j \not\sim i'$, then
$$ w_{ch(j)\,j} < w_{i'j}. $$
\end{corollary}

\begin{proof}
It is sufficient to show that the above holds with some constant $r^\ast(j)$ for each fixed $j$, then set $r_1^\ast = \min_j r^\ast(j)$. Fix $j$ and $i'$ as stated. Let $b^\ast := \beta (\log \theta_j - \log(\theta_{ch(j)}-\theta_j))$, and let $r^\ast(j)$ be the constant $r^\ast$ that works for $B = b^\ast$ in Lemma \ref{lem:variation}. By Proposition \ref{prop:compare.X},
$$ w_{ch(j)\,j} = q(b^\ast,\bigR) - q(b^\ast,\smallR). $$
Now we consider two cases. \\
Case 1: $i'$ is a descendant of $j$, that is, $j \rightsquigarrow i'$. Then by Proposition \ref{prop:compare.X},
$$w_{ji'} = q(b,\bigR) - q(b,\smallR) $$
where $b =\beta(\log\theta_j - \log(\theta_{i'}-\theta_j))$. But since $i' \neq ch(j)$, $i'$ must be a descendant of $i$ as well. By definition \eqref{eqn:theta.i}, $i \rightsquigarrow i'$ implies $\theta_{i'} > \theta_i$. Therefore, $b < b^\ast$, so by \eqref{eqn:q.from.H}, $w_{ij} < w_{i'j}$. This concludes case 1. \\
Case 2: $j \not\sim i'$. Then by Proposition \ref{prop:compare.X}, 
$$ w_{i'j} = q(-\infty,\bigR) - q(-\infty,\smallR).$$
Since $-\infty < b^\ast$, so by \eqref{eqn:q.from.H}, $w_{ij} < w_{i'j}$. This concludes case 2.
\end{proof}

\begin{lemma}\label{lem:want}
There exists some $r^\ast_2 > 0$ such that for all $0 < \smallR < \bigR < r^\ast_2$, for all $j \in V$, $i' \rightsquigarrow j$ implies
\begin{equation}\label{eqn:want}
q(b,\bigR) - q(b,\smallR) < q(b',1-\smallR) - q(b',1-\bigR),
\end{equation}
where $b = \beta(\log\theta_j - \log(\theta_{ch(j)}-\theta_j))$ and $b' = \beta(\log\theta_{i'} - \log(\theta_j - \theta_{i'}))$. In particular, if $i' \rightsquigarrow j$, then for all quantile levels $\smallR,\bigR$ such that $0 < \smallR < \bigR < r^\ast_2$,
\begin{align}\label{eqn:lemma6}
     w_{ch(j)\, j} < w_{i'j}.
\end{align}
\end{lemma}

\begin{proof}
It is sufficient to prove that \eqref{eqn:want} holds for each fixed $j$ with some constant $r^\ast_2(j)$, and then set $r^\ast_2 = \min_j r^\ast_2(j)$. Fix $j$. First, we do some manipulations on \eqref{eqn:want} to relate it to the partial derivatives of $H$. Define
\begin{equation}\label{eqn:cal.C}
\mathcal{B} := \{\beta (\log\theta_j - \log(\theta_i-\theta_j)): i,j \in V, j \rightsquigarrow i\}.
\end{equation}
Note that \eqref{eqn:want} is equivalent to 
\begin{equation}\label{eqn:want.2}
\partial_2 q(b,r) < \partial_2 q(b',1-r) \quad \mbox{ for all $r \in (0,r^\ast_2)$  and for all $b,b' \in \mathcal{B}$.}
\end{equation}
By \eqref{eqn:Ht.1}, we have
$$  \partial_2 q(b,r) - \partial_2 q(b',1-r) = \frac{1}{\partial_2H(b,q(b,r))} - \frac{1}{\partial_2H(b',q{(b',1-r))}}.$$
By \eqref{eqn:H.a}, $\partial_2H > 0$ point-wise, thus our goal now is to show that for sufficiently small $r$,
\begin{equation}\label{eqn:want.3}
\partial_2H(b',q(b',{1-r})) - \partial_2H(b,q(b,r)) \JB{<}  0
\end{equation}
for all $b,b' \in \mathcal{B}$, that is, some finite set of constants.\JB{ We shall do this by writing $\partial_2H$ in terms of the tail densities $f_\eta$ and $f_\xi$ using \eqref{eqn:H.a}, then apply Lemma \ref{lem:tail}. 
Indeed, by \eqref{eqn:H.a},
\begin{align*}
\partial_2H(b',a) &= \P(\xi \leq b')f_\eta(a-b') + \int_{b'}^\infty f_\eta(a-x)f_\xi(x)\, dx 
\end{align*}
By Lemma \ref{lem:tail}, $f_\xi$ has heavier tail than $f_\eta$, so for $a \to \infty$, the main contribution from $\int_{b'}^\infty f_\eta(a-x)f_\xi(x)\, dx$ comes from $f_\xi (a)$. That is, for large $a$, there exists some constant $b_1 > 0$ such that
\begin{equation}\label{eqn:H.big}
\partial_2H(b',a) > b_1f_\xi(a).
\end{equation}
Now we consider $\partial_2H(b,-a)$. From \eqref{eqn:H.a},
\begin{align*}
\partial_2H(b,-a) = \P(\xi \leq b)f_\eta(-a-b) + \int_b^\infty f_\eta(-a-x)f_\xi(x)\, dx.
\end{align*}
Again, for large $a$
$$ f_\eta(-a-x) < f_\eta(-a-b) \mbox{ for all } x > b. $$
Therefore, we can bound the second term above as
$$ \int_b^\infty f_\eta(-a-x)f_\xi(x)\, dx < f_\eta(-a-b)\int_b^\infty f_\xi(x)\, dx = f_\eta(-a-b)\P(\xi > b). $$
Adding in the first term, we get that for large $a$,
$$ \partial_2H(b,-a) < f_\eta(-a-b) $$
Combining this with \eqref{eqn:H.big} and noting that $\partial_2H(b,a)$ is just the density $f_{\xi_{b}}(a)$ of $\xi_{b}$, we get
\begin{align}\label{eq:taildensities}
    f_{\xi_{b}}(-a)=O(f_{\xi_{b'}}(a))
\end{align}
for all $b,b' \in \mathcal{B}$ and $a$ large. Now $\partial_2H(b,q(b,{r}))$ is just the slope of the cdf of $f_{\xi_{b}}$ at its $r$-quantile. Therefore, for $r$ small, by \eqref{eq:taildensities},  $\partial_2H(b',q(b',{1-r})) < \partial_2H(b,q(b,r))$ which proves \eqref{eqn:want.3} and thus completes the proof of \eqref{eqn:want}. The last statement follows from Proposition \ref{prop:compare.X}, case (3). }
\end{proof}

\begin{corollary}\label{cor:recap}
If the true quantiles are known, then there exist some choices of $(\smallR,\bigR)$ such that the lower quantile gap matrix $W$ satisfies the conditions of Lemma~\ref{lem:w.good}, that is, \eqref{eqn:w.ij.correct} and \eqref{eqn:min.root}.
\end{corollary}
\begin{proof}
Set $r^\ast = \min(r^\ast_1,r^\ast_2)$ where $r^\ast_1$ comes from Corollary \ref{cor:variation}, and $r^\ast_2$ comes from Lemma~\ref{lem:want}. Let $(\smallR,\bigR)$ be any pair such that $0 < \smallR < \bigR < r^\ast$, and let $W$ be the corresponding lower quantile gap matrix with the true quantiles. 
Then \eqref{eqn:min.root} holds because of \eqref{eqn:lemma6} and the fact that for the root $r$ of the root-directed spanning tree, $i' \rightsquigarrow r$ holds for every $i' \neq r$. Corollary \ref{cor:variation} and Lemma~\ref{lem:want} together guarantee that \eqref{eqn:w.ij.correct} is satisfied for $W$. 
\end{proof}

\subsection*{Proof of Theorem~\ref{thm:main} for the lower quantile gap} 

Fix $(\smallR,\bigR)$ such that Corollary \ref{cor:recap} holds, and let $W$ be the corresponding lower quantile gap matrix derived from the true quantiles. Let $W_n$ be the lower quantile gap matrix derived from an empirical distribution with sample size $n$.
Note that the set of `good' matrices, that is, those that satisfy Lemma \ref{lem:w.good}, is an open polyhedral cone in the space of matrices $\R^{V \times V}$, since the conditions of `goodness' is a set of linear inequalities. By Corollary \ref{cor:recap}, $W$ is a point inside this cone. 
Recall that empirical quantiles converge a.s. as $n\to\infty$ to the true ones for continuous limit distributions, hence, also the empirically-derived lower quantile gap converges a.s..
By a union bound over the $d^2-d$ possible edge pairs $(i,j)$, for any metric $D$ (e.g. induced by a matrix norm), we thus have $D(W_n,W)\to 0$ a.s. 
The Consistency Theorem then follows from Lemma~\ref{lem:w.good}. \qed

\subsection{Proof of Theorem~\ref{thm:main} for the quantile-to-mean gap}

Our proof follows the same structure as the previous proof, but the calculations in all steps are a bit simpler, since there is only one quantile parameter to deal with. 
First, expectation is linear, so we work with empirical means $\bar{X}_i$ for $i\in V$ and mention in passing that they converge a.s. to the true mean as $n\to\infty$.
The analogue of Proposition~\ref{prop:compare.X} is the following.

\begin{proposition}\label{prop:compare.X.2}
Fix $\smallR \in [0,1)$, and let $w_{ij}$ be the quantile-to-mean gap \eqref{eqn:wij.tail.mean2}. Assume the Gumbel-Gaussian noise model. Then 
\begin{enumerate}
	\item[{\rm (1)}] 
	If $j \rightsquigarrow i$, then $w_{ij} =- q_{\smallR}(\xi^b)$ where $b=\beta (\log\theta_j- \log(\theta_i-\theta_j))$.
	\item[{\rm (2)}] 
	If $j \not\sim i$, then $w_{ij} =- q_{\smallR}(\xi^b)$ where $b=-\infty$.
	\item[{\rm (3)}] 
	If $i \rightsquigarrow j$, then $w_{ij} =  q_{1-\smallR}(\xi^b)$ where
	$ b = \beta (\log\theta_i-  \log(\theta_j-\theta_i))$.
\end{enumerate}
\end{proposition}

Instead of a lengthy proof of the analog of Proposition~\ref{prop:compare.X} by duplicating arguments, we provide some informal reasoning.
We check that our quantile-to-mean gaps $w_{ij}$ satisfy the inequalities of Corollary \ref{cor:variation} and Lemma \ref{lem:want} by first checking the noise-free case, where $\eps_i \equiv \eps_j \equiv 0$. We consider the three cases of Proposition~\ref{prop:compare.X.2}. 
\begin{enumerate}
	\item If $j \rightsquigarrow i$. Then $\xi^b$ has a left-most atom at $b= \beta (\log\theta_j-\log(\theta_i-\theta_j))$, so for sufficiently small $r$, $w_{ij}=-b$. This is minimal when $i$ is a direct descendant of $j$. 
	So Corollary~\ref{cor:variation} for the case $j \rightsquigarrow i$ holds in the noise-free case. 
	\item If $j \not\sim i$. Then $\xi^b$ has no left-most atom, so as $\smallR \downarrow 0$, $q_{\smallR}(\xi^b) \to -\infty$, so $w_{ij} \to \infty$. So Corollary \ref{cor:variation} also holds in the noise-free case for the remaining case, $j \not\sim i$.  
	\item If $i \rightsquigarrow j$. Then $\xi^b$ has a left-most atom, but no right-most atom. Again, as $\smallR \downarrow 0$, $q_{1-\smallR}(\xi^b) \to \infty$, so $w_{ij} \to \infty$. Thus, Lemma \ref{lem:want} holds in the noise-free case.
\end{enumerate}

Now we consider the effect of noise. We send $\smallR \downarrow 0$. As long as $\eta := \eps_i - \eps_j$ has lighter tail than $Z_i - Z_j$, as guaranteed by Lemma \ref{lem:tail}, we have the following:
\begin{itemize}
	\item In case (1), \JB{$q_{\smallR}(\xi^b)$} is dominated by the lower tail of $\eta$. 
	\item In case (2), \JB{$q_{\smallR}(\xi^b)$} is dominated by the lower tail of $Z_i - Z_j$ and, in particular, is going to $-\infty$ at a faster rate than case (1). 
	\item In case (3), \JB{$q_{1-\smallR}(\xi^b)$} is dominated by the upper tail of $Z_i - Z_j$, and in particular, is going to $\infty$ at a faster rate than case (1). 
\end{itemize}
This domination calculation is the same calculation done in the proof of Lemma \ref{lem:want}. The above says that for fixed $j$, for small enough $\smallR$, the minimum of $\{w_{ij}: i \neq j, i \in V\}$ lies in case (1). Within case (1), we want to make sure that, if $w_{ij}$ is smallest, then $i$ is the child of $j$. 
Indeed, write\JB{
$$\xi^b = \eps_i - \eps_j + \xi'_{ij}$$
where $\xi'_{ij} = (Z_i - Z_j)\vee (\beta(\log\theta_j-\log(\theta_i-\theta_j)))$.
For fixed $j$, $(\xi'_{ij}: j \rightsquigarrow i)$ is a particular family of distribution indexed by $i$. By a decoupling argument, it is sufficient to show that $q_{\smallR}(\xi'_{ij}) $ is smallest when $i$ is the child of $j$.} But this reduces to the noise-free case, which we already proved above. This finishes the proof of Theorem~\ref{thm:main} for the quantile-to-mean gap. \qed

\newpage

\section{Supplemental Figures for Section~4.2}\label{sec:s4}

Below we present the figures analogous to Figure~\ref{fig:parameter_sel_danube} of the Paper for the different data sets of the Colorado river network.
For certain values $\alpha$, the boxplots in 
Figures~\ref{fig:parameter_sel_colup}
and~\ref{fig:parameter_sel_colmid}
degenerate, indicating that the 25\% and 75\% quantiles of the resampled data match, i.e. 50\% of the data settle upon the same value of the metric. The very large boxes for high $\alpha$ indicate the uncertainty in the small data sets for computations of the metrics.

\begin{figure}[htbp]
    \centering
 \includegraphics[width=\textwidth]{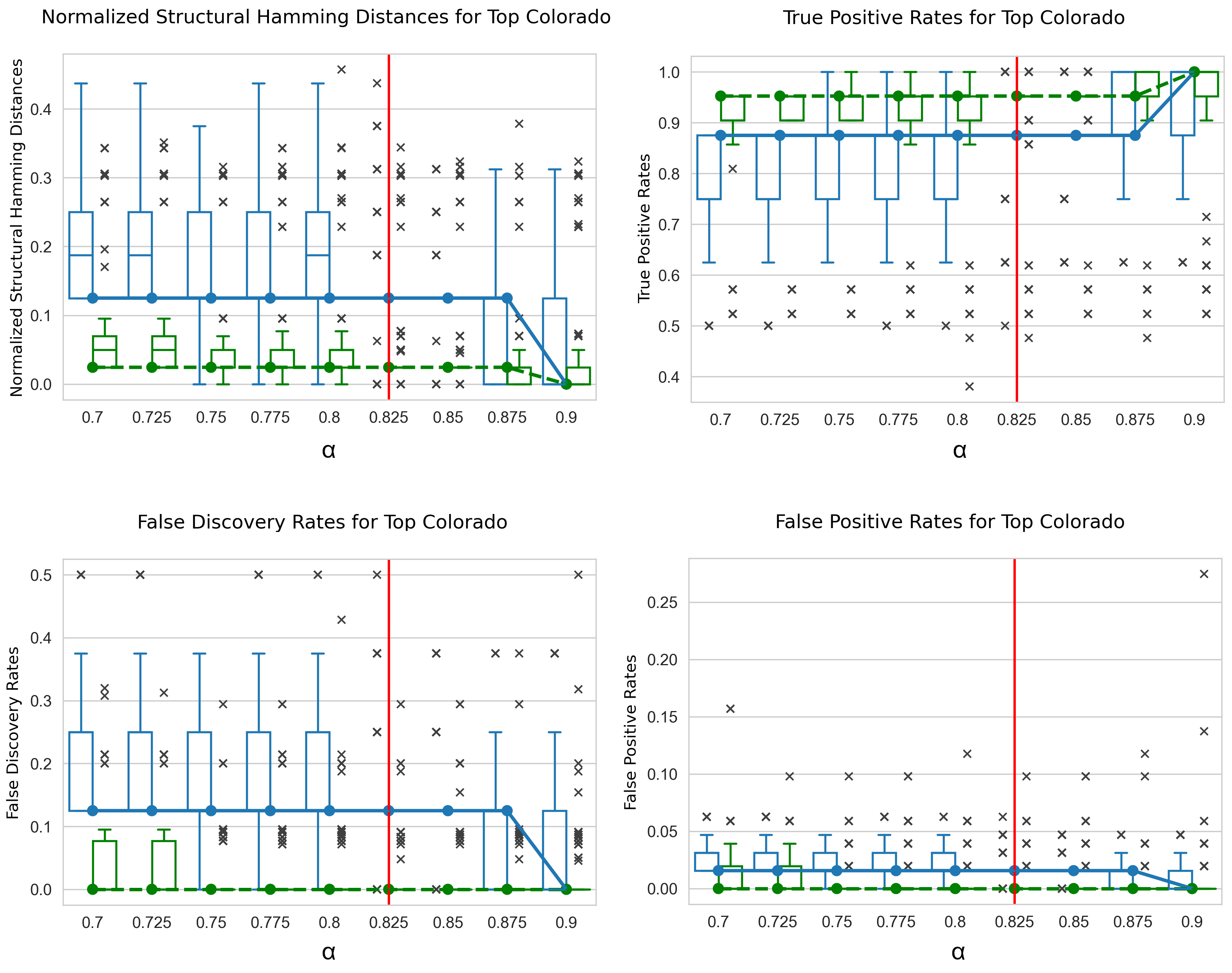}
    \caption{Metrics nSHD, TPR, FDR and FPR for the Top sector of the Colorado network and varying parameters $\al$. For detailed explanations see Figure~\ref{fig:parameter_sel_danube} of the Paper. } 
    \label{fig:parameter_sel_colup}
\end{figure}

\begin{figure}
    \centering
    \includegraphics[width=\textwidth]{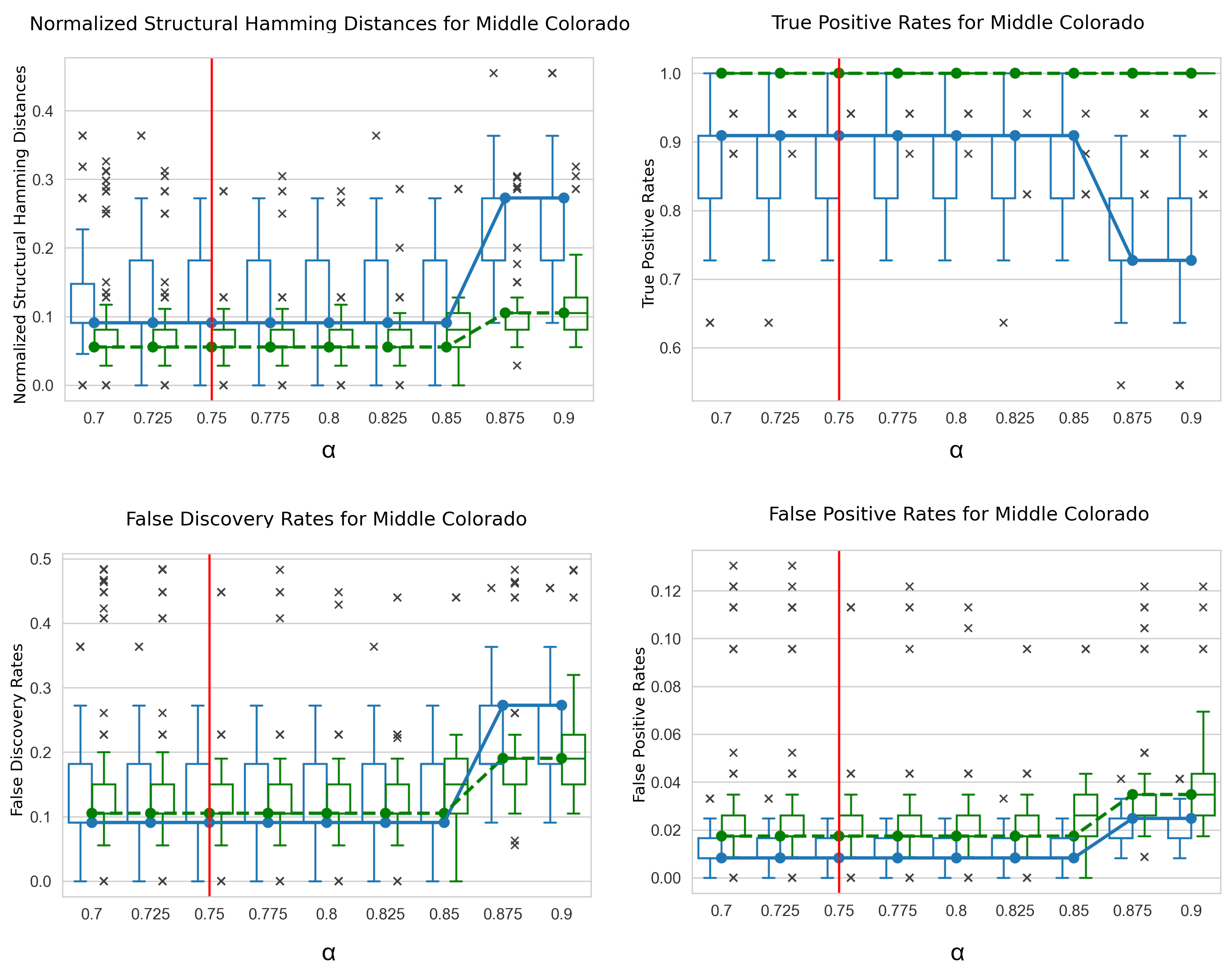}
    \caption{Metrics nSHD, TPR, FDR and FPR for the Middle sector of the Colorado network and varying parameters $\al$. For detailed explanations see Figure~\ref{fig:parameter_sel_danube} of the Paper. 
    } 
    \label{fig:parameter_sel_colmid}
\end{figure}

\begin{figure}
    \centering
    \includegraphics[width=\textwidth]{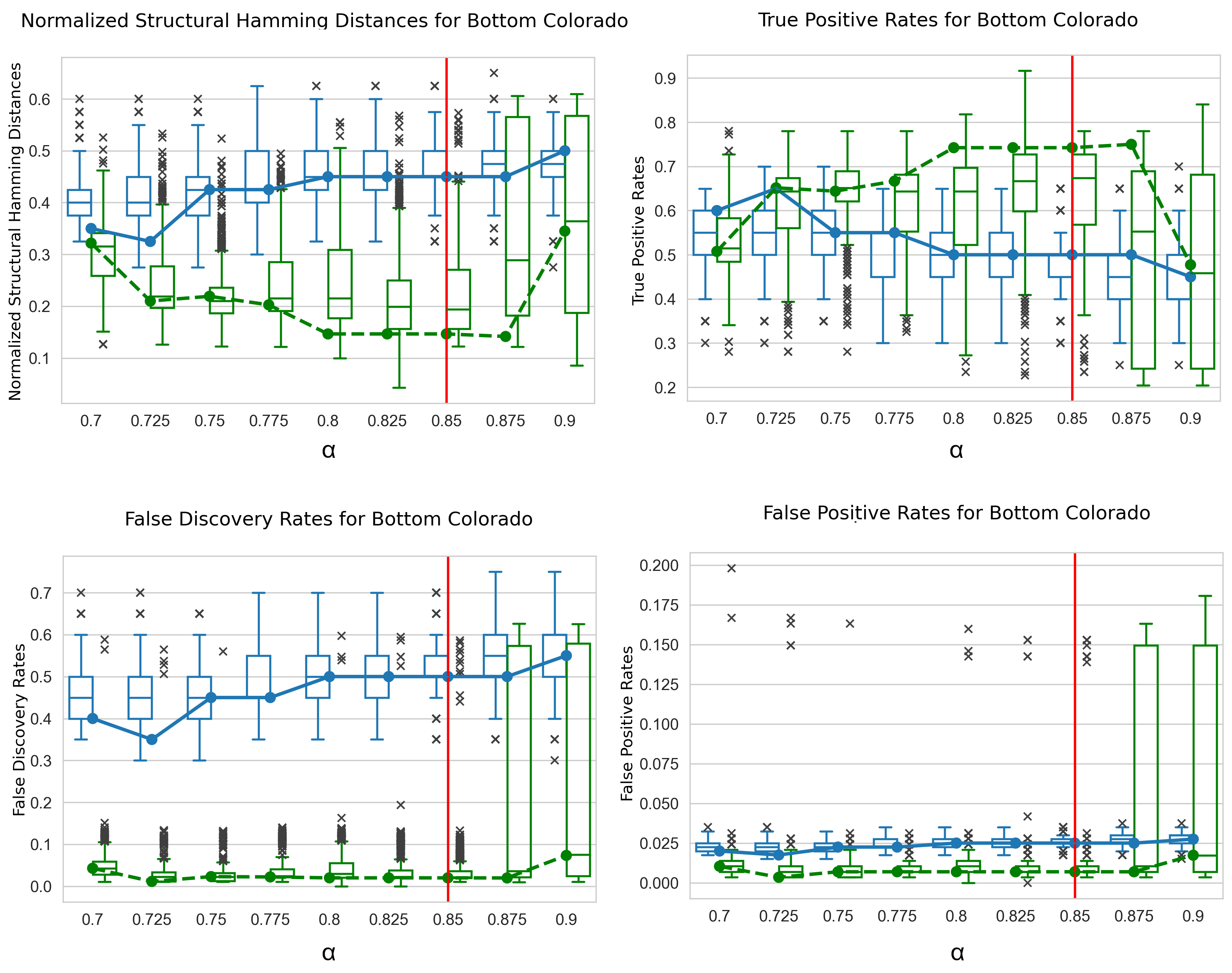}
    \caption{Metrics nSHD, TPR, FDR and FPR for the Bottom sector of the Colorado network and varying parameters $\al$. For detailed explanations see Figure~\ref{fig:parameter_sel_danube} of the Paper. } 
    \label{fig:parameter_sel_collow}
\end{figure}

\begin{figure}
    \centering
    \includegraphics[width=\textwidth]{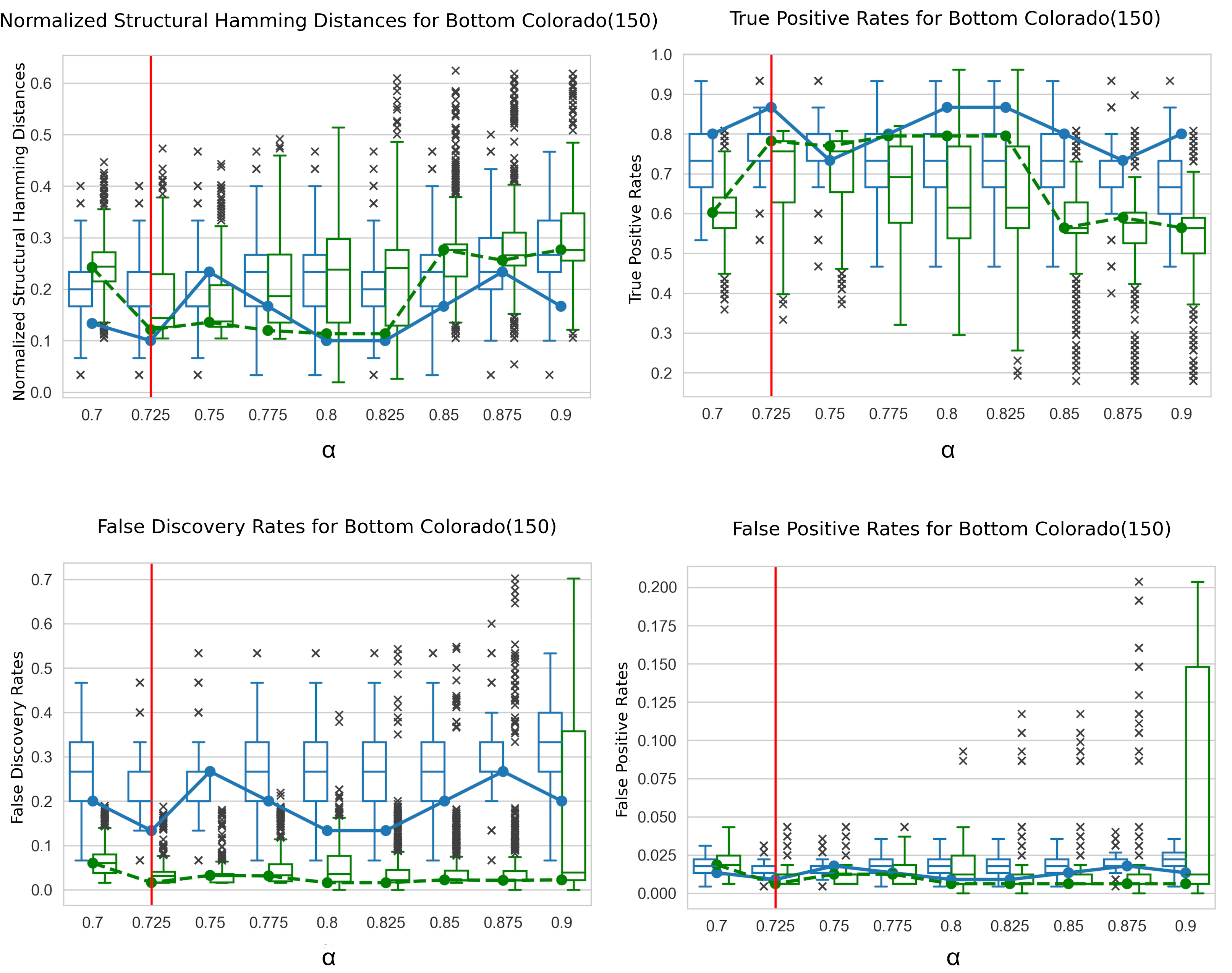}
    \caption{Metrics nSHD, TPR, FDR and FPR for Bottom150 of the Colorado network and varying parameters $\al$. For detailed explanations see Figure~\ref{fig:parameter_sel_danube} of the Paper.} 
    \label{fig:parameter_sel_lowcol150}
\end{figure}

\pagebreak

\FloatBarrier

\section{Supplemental Figures for Section~5}\label{sec:simSupp}
   
Below we visualize the performance measures nSHD and TPR as defined in equation~\eqref{eq:per_metrics} of the Paper from simulations of settings (2) and (3) of Section~5 of the Paper.

       \begin{figure}[h!]
 \centering
 \includegraphics[width=12cm]{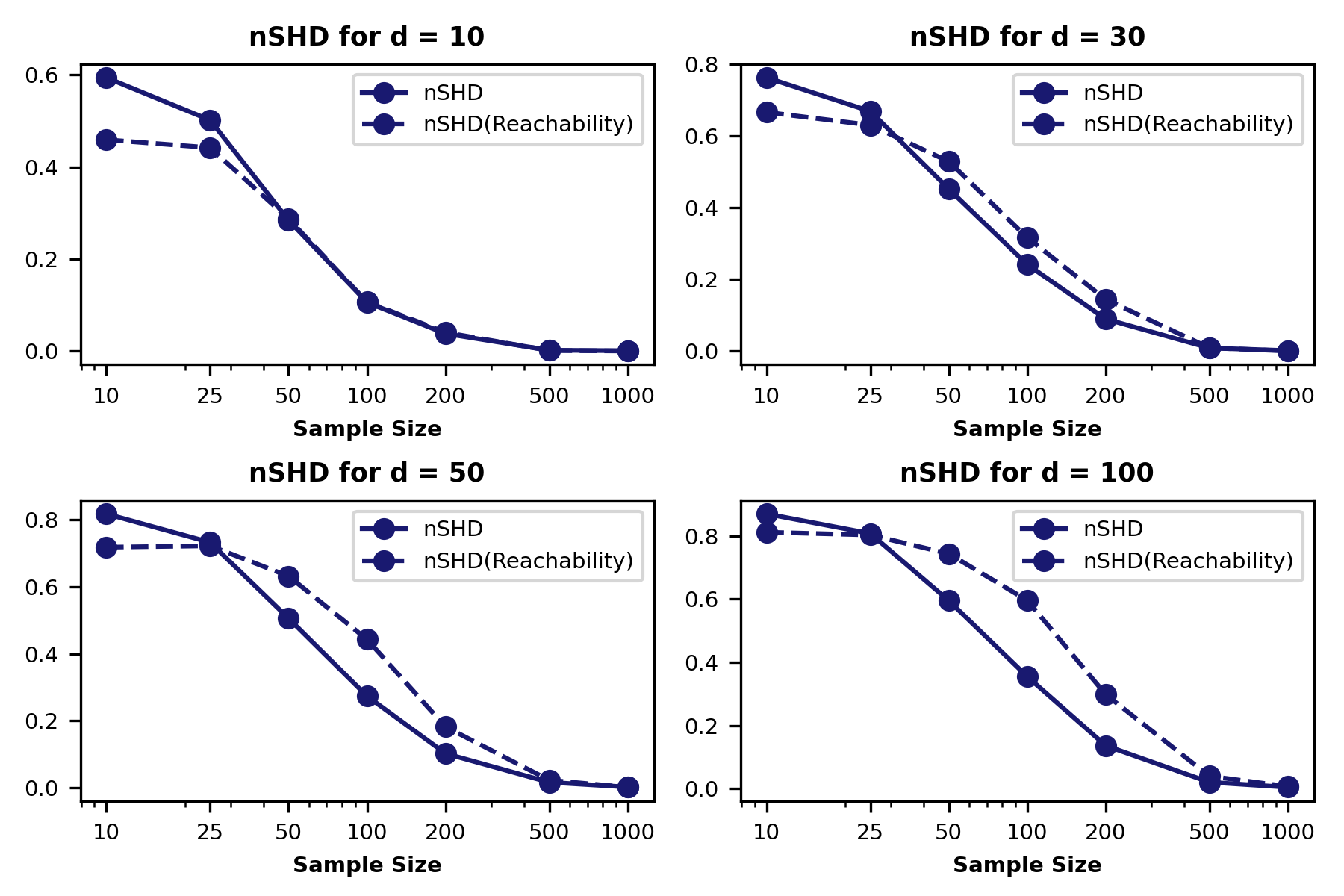}
        \caption[]
        {Mean nSHD for the weak dependence setting (2) and different graph sizes.\\
        For detailed explanations, see Figure~\ref{fig:SHD} of the Paper.}\label{fig:SHD-smalldep}
    \end{figure}
    
          \begin{figure}[b!]
 \centering
 \includegraphics[width=12cm]{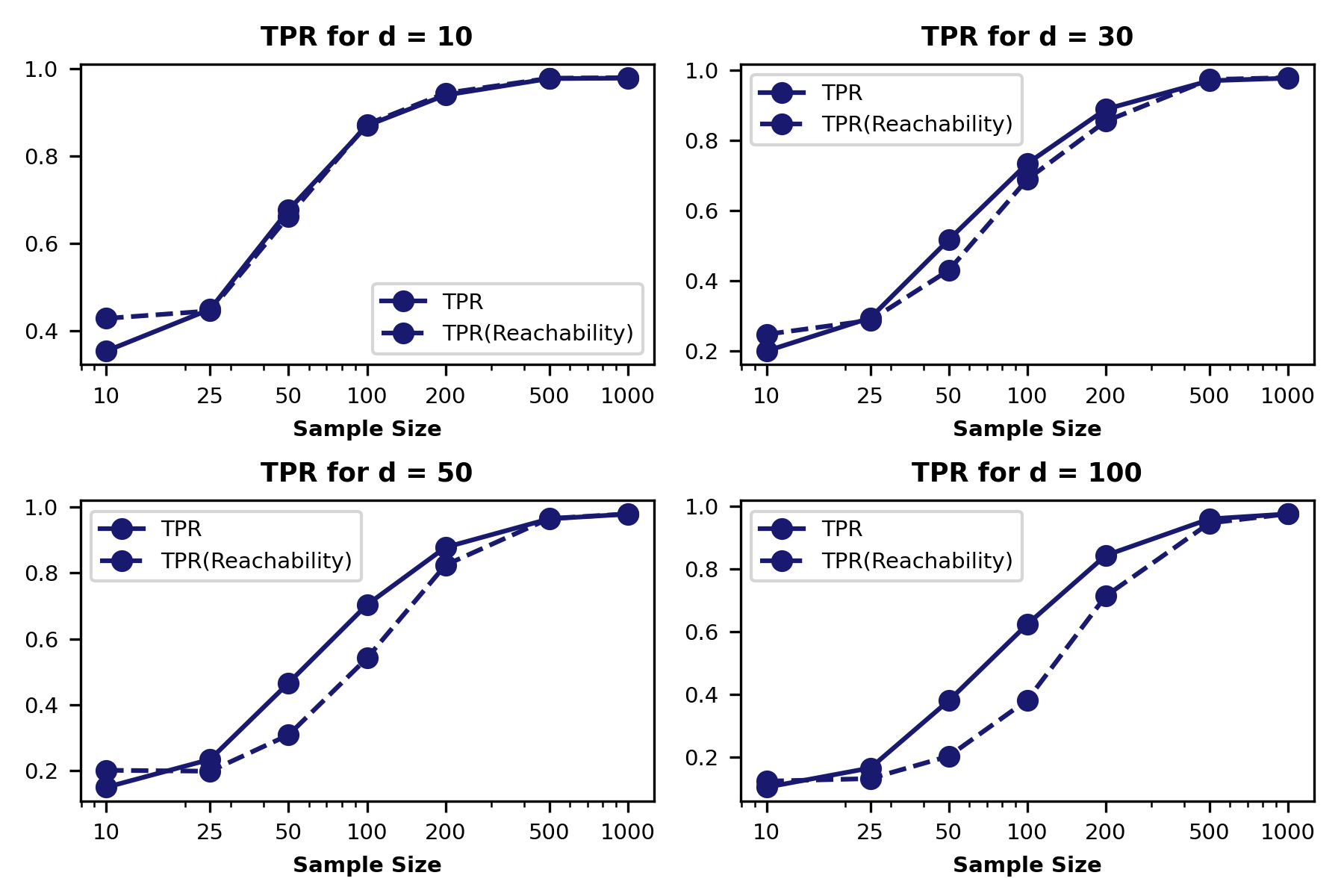}
        \caption[]
        {Mean TPR for the weak dependence setting (2) and different graph sizes.\\
        For detailed explanations, see Figure~\ref{fig:SHD} of the Paper.} \label{fig:TPR-smalldep}
    \end{figure}
    
    \begin{figure}[h!]
 \centering
 \includegraphics[width=12cm]{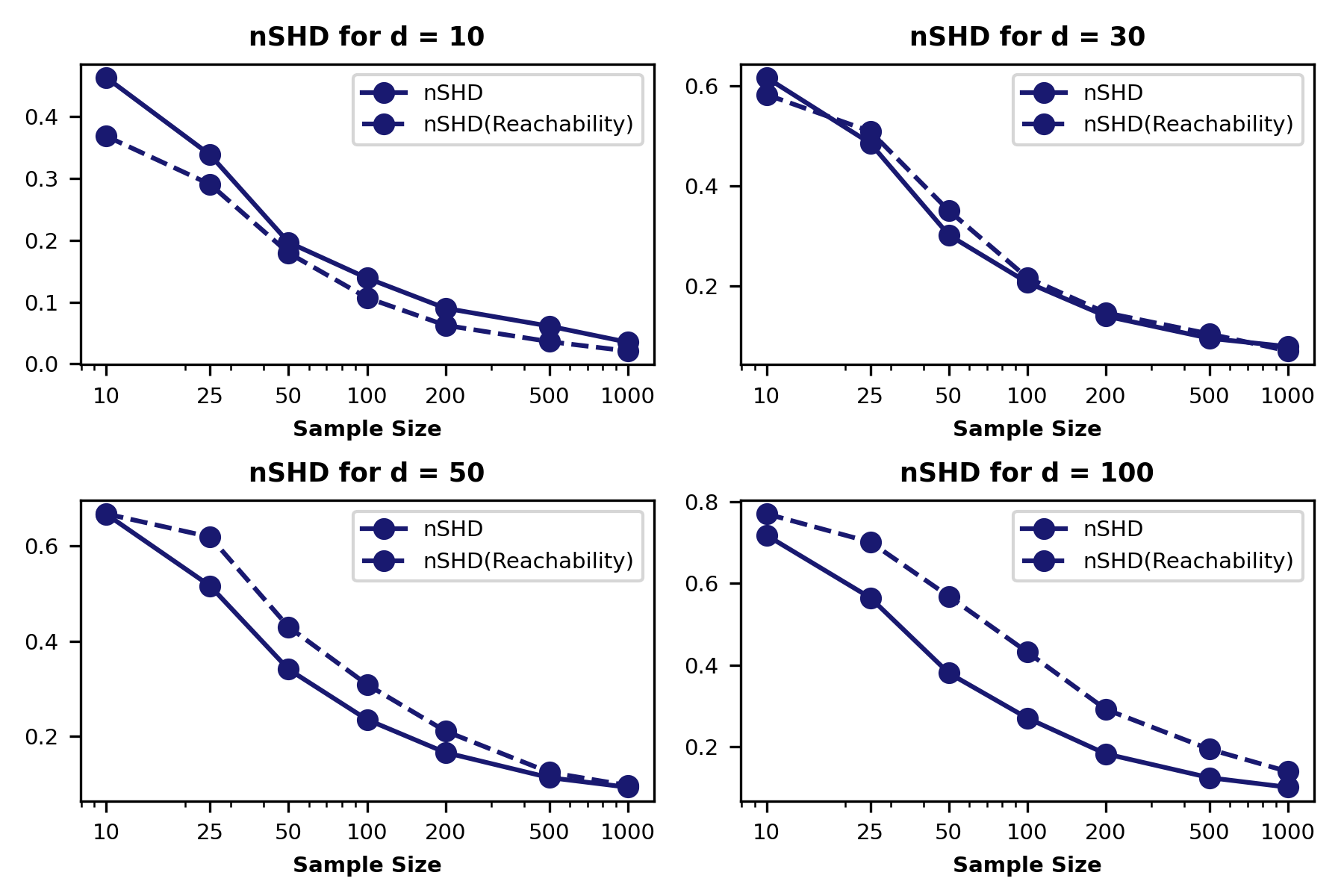}
        \caption[]
        {\small Mean nSHD for the mixed distribution setting (3) and different graph sizes.\\
        For detailed explanations, see Figure~\ref{fig:SHD} of the Paper.}\label{fig:SHD-mixed}
    \end{figure}
    
          \begin{figure}[b!]
 \centering
 \includegraphics[width=12cm]{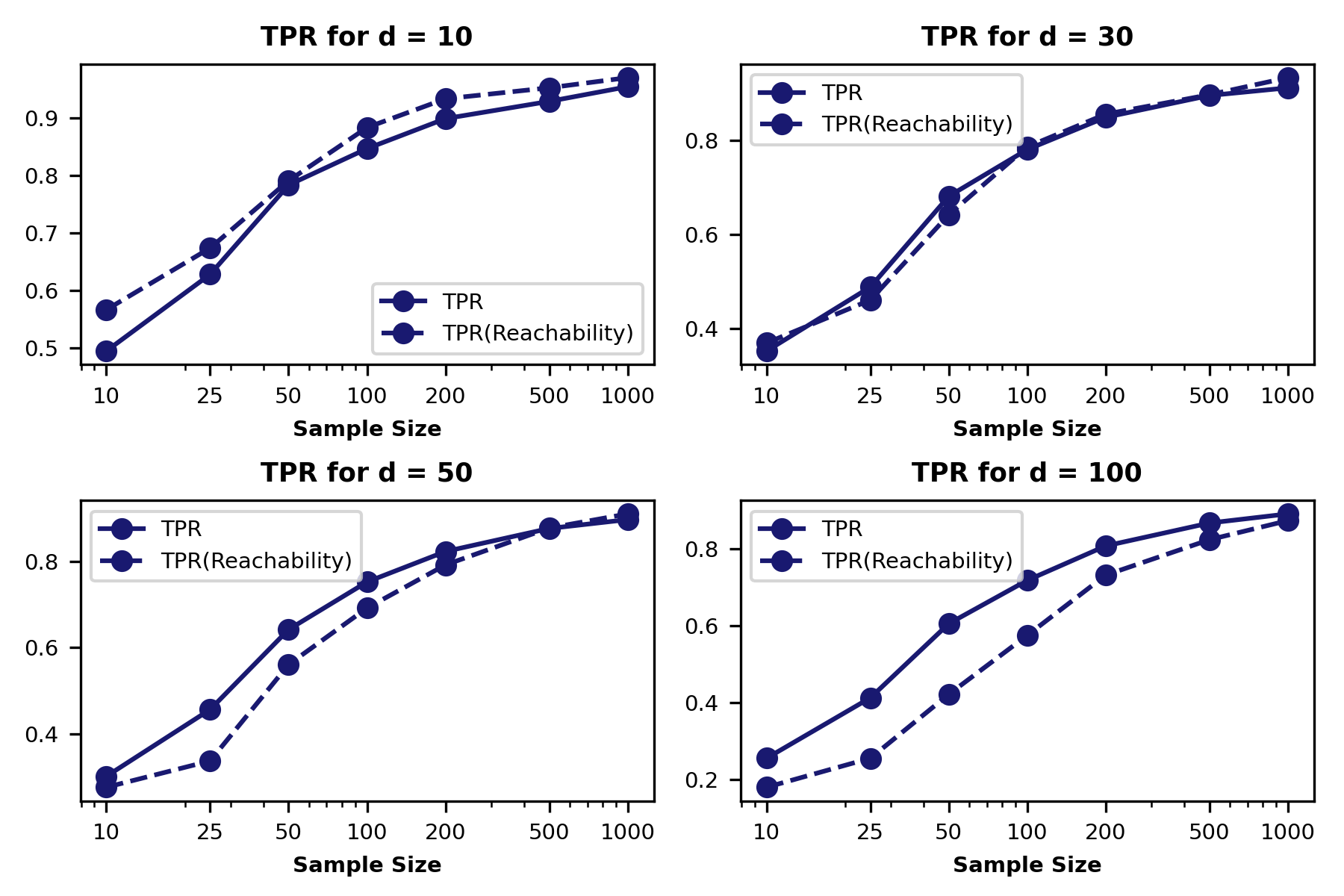}
        \caption[]
        {Mean TPR for the mixed distribution setting (3) and different graph sizes.\\
        For detailed explanations, see Figure~\ref{fig:SHD} of the Paper.} \label{fig:TPR-mixed}
    \end{figure}

\FloatBarrier




\end{document}